\documentclass{article}

\usepackage{aistats2021}
\usepackage[utf8]{inputenc} % allow utf-8 input
\usepackage{hyperref}       % hyperlinks
\usepackage{url}            % simple URL typesetting
\usepackage{booktabs}       % professional-quality tables
\usepackage{amsfonts}       % blackboard math symbols
\usepackage{nicefrac}       % compact symbols for 1/2, etc.
\usepackage{microtype}      % microtypography
\usepackage{graphicx}
\usepackage[space]{grffile}
\usepackage{latexsym}
\usepackage{textcomp}
\usepackage{longtable}
\usepackage{tabulary}
\usepackage{booktabs,array,multirow}
\usepackage{url}
\usepackage{xr} 
\usepackage{amsmath,amsfonts,amssymb}
\usepackage[ruled,shortend]{algorithm2e}
\usepackage{mathtools}
\usepackage{bm}
\usepackage[mathscr]{euscript}
\usepackage{enumitem}
\usepackage{xcolor}
\usepackage{booktabs}
\usepackage{fix-cm}
\usepackage{colortbl}
\usepackage{tikz}
\usepackage{pbox}
\usepackage{xcolor}
\usepackage{etoolbox}
\usepackage{bbm}
\usepackage{amsthm,pifont}
\usepackage{chngcntr}
\usepackage{apptools}
\usepackage[toc,page,header]{appendix}
\usepackage{minitoc}

\usepackage[round]{natbib}
\DeclareMathOperator{\pr}{\mathbb P}
\DeclareMathOperator{\E}{\mathbb E}
\DeclareMathOperator{\Var}{Var}

\newcommand{\Real}{\mathbb R}

\newcommand{\NatInt}{\mathbb N}

\newcommand{\CalO}{\mathcal O}

\newcommand{\CalH}{\mathcal H}
\newcommand{\CalX}{\mathcal X}
\newcommand{\CalW}{\mathcal W}
\newcommand{\CalB}{\mathcal B}
\newcommand{\CalMB}{\mathcal{MB}}
\newcommand{\BX}{\bold X}
\newcommand{\Bx}{\bold x}
\newcommand{\By}{\bold y}

\newcommand{\Bk}{\bold k}
\newcommand{\Bs}{\bold s}

\newcommand{\Balpha}{\boldsymbol{\alpha}}

\newcommand{\argmin}{\mathop{\mathrm{argmin}}}
\newtheorem{theorem}{Theorem}
\AtAppendix{\counterwithin{theorem}{section}}
\newtheorem{defn}{Definition}
\AtAppendix{\counterwithin{defn}{section}}
\newtheorem{rem}{Remark}
\newtheorem{lem}{Lemma}
\AtAppendix{\counterwithin{lem}{section}}

\newtheorem{prop}{Proposition}
\AtAppendix{\counterwithin{prop}{section}}

\newtheorem{assumption}{Assumption}
\newtheorem{condition}{Condition}
\AtAppendix{\counterwithin{assumption}{section}}

\begin{document}
\doparttoc % Tell to minitoc to generate a toc for the parts
\faketableofcontents % Run a fake tableofcontents command for the partocs

% Start the document part
% Insert the document TOC


\onecolumn

% \aistatstitle{High-Dimensional Non-Parametric Density Estimation in Mixed Smooth Sobolev Spaces under Integral Probability Metric }
% \aistatstitle{High-Dimensional Non-Parametric Density Estimation in Mixed Smooth Sobolev Spaces under Integral Probability Metric and Wasserstein Distance}
\aistatstitle{High-Dimensional Non-Parametric Density Estimation in Mixed Smooth Sobolev Spaces}

{\large{\textbf{ $\text{Liang Ding}^{*1}$, \    $\text{Lu Zou}^{\dagger1}$, \  $\text{Wenjia Wang}^{\dagger2}$, \  $\text{Shahin Shahrampour}^{\sharp}$  \ and \  $\text{Rui Tuo}^{*2}$ }
}}
{\Large{\aistatsaddress{ \\[2ex]$\text{Texas A\&M University}^{*}$\\[2ex]   $\text{The Hong Kong University of Science and Technology}^{\dagger}$ \\[2ex] $\text{Northeastern University}^{\sharp}$ } }
}
\large
\begin{abstract}
Density estimation plays a key role in many tasks in machine learning, statistical inference, and visualization. The main bottleneck in high-dimensional density estimation is the prohibitive computational cost and the slow convergence rate. In this paper, we propose  novel estimators for high-dimensional non-parametric density estimation called the adaptive hyperbolic cross density estimators, which enjoys nice convergence properties in the mixed smooth Sobolev spaces. As modifications of the usual Sobolev spaces, the mixed smooth Sobolev spaces are more suitable for describing high-dimensional density functions in some applications. We prove that, unlike other existing approaches, the proposed estimator does not suffer the curse of dimensionality under Integral Probability Metric, including H\"older Integral Probability Metric, where Total Variation Metric and Wasserstein Distance are special cases. 
% for an $n$-points dataset. 
Applications of the proposed estimators to generative adversarial networks (GANs) and goodness of fit test for high-dimensional data are discussed to illustrate the proposed estimator's good performance in high-dimensional problems. 
% Specifically, we construct an improved convergence rate for GANs and an asymptotically normal statistic for goodness of fit test. 
Numerical experiments are conducted and illustrate the efficiency of our proposed method.
\end{abstract}

\section{Introduction}
In unsupervised learning, one particular task is to learn the underlying probability distribution based on observed data. Non-parametric density estimation provides a flexible way to study the unobserved distribution, and has been widely adopted across various domains, especially in the statistical inference, visualization, and machine learning. 

In non-parametric density estimation, the prevalent methods include kernel density and orthogonal/wavelet sequence estimators. A kernel density estimator \citep{friedman2001elements} is a linear combination of scaled kernel functions, where the scale is determined by the bandwidth. A key challenge for kernel density estimation is to determine the optimal bandwidth to balance the bias and variance which costs additional computation burden to the estimation procedure. 
Some popular techniques to determine the bandwidth are cross-validation, rules-of-thumb, and \textit{plug-in} methods \citep{Loader1999}. However, the tuning bandwidth procedure and the slow converge rate make the kernel density estimator works poorly in high dimensional setting.  

Besides the kernel density estimators, the orthogonal/wavelet series estimators also received appreciable attention \citep{Hall1981,wahba1981,Hall1987,Donoho96}. Under the high dimension setting, the convergence rates of these methods are usually slow, because the number of bases needs to be increase exponentially with the dimension of the data points to guarantee a small bias, which results in a large 
% value of the  part
variance. One way to deal with the high dimensional density estimation is the semi-parametric mixture model with a fixed number of parametric components \citep{hathaway1985,wang2015}, while \cite{genovese2000,Ghosal2000} indicate that the mixture models generally are less superior to the kernel density estimators.

In this work, we propose a new class of  estimators applied to the high dimensional setting named as \textit{adaptive hyperbolic cross estimator} (AHCE). Here we note that the hyperbolic cross estimators has been applied to the one-dimensional case \citep{Hall1987}, while in this work, we modify it such that it can be applied to the high dimensional setting, which makes the theoretical analysis and empirical results significantly different. In particular, we consider the density lies in a mixed smooth Sobolev spaces $\CalH^r_{mix}$, which is tensor product of 1-D Sobolev space $\CalH^r$ ($r$ can be a fractional number). These mixed functions spaces contains a wide range of functions such as the semi-parametric density functions of the form $f(x_1,\cdots,x_D)=g(x_{R})b(x_1,\cdots, x_D)$ \citep{han2007}, where $R$ is a subset of $[D]\coloneqq\{1,\cdots,D\}$ and $x_{R}=(x_j, j\in R)$ is a subset of the variables, $b$ is parametrized density function and $g$ is in some specific family of functions. The mixed smooth Sobolev spaces are considered in high-dimensional approximation and numerical methods of PDE \citep{bungartz1999note}, data mining \citep{garcke2001data}, and deep neural networks \citep{dung2021deep}.
% The mixed smooth Sobolev spaces are less restrictive than the low effective dimensionality, where the later one imposes a low-dimensional structure on the input space. 
Motivated by a general formulation in GANs \citep{Arjovsky17,liang2018generative,liu2017approximation} (see Section \ref{subsec:improvegan} for more details), we investigate the convergence rate of AHCE under the following \textit{Integral Probability Metric} (IPM),
\begin{equation}
        \label{eq:GAN_formulation}
         d_{F_d}(p,q)\coloneqq \sup_{f\in F_d} \E_{X\sim p} f(X)-\E_{Y\sim q}f(Y),
\end{equation}
where $p$ and $q$ are two probability measure, and $F_d$ is some pre-determined function space called the discriminator class. In the following content, we assume throughout that the discriminator class $F_d$ is a mixed smooth Sobolev space $\CalH^\beta_{mix}$ or a H\"older class $\CalW^\beta_\infty$ with some $\beta\geq 0$ and the true density function $p$ lies in $\CalH^\alpha_{mix}$ with some $\alpha>0$. Notice that when $\beta=0$, IPM becomes the $L^2$ metric. 

\begin{table*}[h]
\centering
\caption{A comparison among the rates of convergence of other density estimators. The results of this work are shown in shaded rows.  Notes: In Row 4, $p^*=\frac{p-1}{p}$; Our work includes the case  $L^2=\CalH^0_{mix}$.}
\tiny{
    \begin{tabular}{ | c| c | c | c |}
    \hline
    Reference & Discriminator Class & True Density Function & Rate of Convergence \\ [2ex] \hline 
    \cellcolor[gray]{0.9}This Work & \cellcolor[gray]{0.9}Mixed Smooth Sobolev: $\CalH^\beta_{mix}$ & \cellcolor[gray]{0.9}Mixed Smooth Sobolev: $\CalH^\alpha_{mix}$ &\cellcolor[gray]{0.9} $n^{-\frac{1}{2}}+n^{-\frac{\alpha+\beta}{2\alpha+1}}[\log n]^{D\frac{\alpha+\beta+1}{2\alpha +1}}$\\ [2ex]
    %\cellcolor[gray]{0.9} & \cellcolor[gray]{0.9}$L^2: \CalH^\beta_{mix}$  with $\beta=0& \cellcolor[gray]{0.9}$ &\cellcolor[gray]{0.9} \\ 
    \hline
    \cellcolor[gray]{0.9}This Work & \cellcolor[gray]{0.9}H\"older: $\mathcal{W}^\beta_\infty$ & \cellcolor[gray]{0.9}Mixed Smooth Sobolev: $\CalH^\alpha_{mix}$ &\cellcolor[gray]{0.9} $n^{-\frac{\alpha+\beta/D
    }{2\alpha+1}}[\log n]^{D\frac{\alpha+\beta/D+1}{2\alpha +1}}$\\ [2ex]
    \hline
    \citet{liang2018generative} & Sobolev: $\mathcal{W}^\beta_2$ & Sobolev: $\mathcal{W}^\alpha_2$ & $\min\{n^{-\frac{\alpha+\beta}{2(\alpha+\beta)+D}},n^{-\frac{\beta}{D}}+n^{-\frac{1}{2}}\}$ \\ [2ex]
    \hline
    \citet{Uppal19} & Besov: $\mathcal{B}^\beta_{p_{\beta},q_{\beta}}$ & Besov: $\mathcal{B}^\alpha_{p_\alpha,q_\alpha}$ & $n^{-\sigma}\sqrt{\log n}: \sigma=\min\{{\frac{1}{2}},$ \\
    & & & ${\frac{\alpha+\beta}{2\alpha+D}},{\frac{\alpha+\beta+D-D/p_\alpha-D/p^*_\beta}{2\alpha+D(1-2/p_\alpha)}}\}$\\[1.5ex]
    \hline
    \citet{Donoho96,Tsybakov08}& $L^p$ &Sobolev: $\mathcal{W}^\alpha_2$ & $n^{-\frac{\alpha}{2\alpha+D}}$\\[2ex]
    % \citet{Tsybakov08}& & & \\
    \hline
    \citet{weed2017sharp,Lei_2020};& Lipschitz Functions& Borel Measurable Distributions& $n^{-\frac{1}{2}}+n^{-\frac{1}{D}}$\\
    \citet{Singh18},& & & \\
    % \citet{Lei_2020}& & & \\[2ex]
    \hline
    \end{tabular}}
% \scalebox{0.9}{.}
\label{Tab:1}
\end{table*}

Under the mixed order Sobolev spaces setting, the AHCE renders promising convergence results: we partially overcome the {\it curse of dimensionality} which is present in the existing works. A comparison of the existing results is summarized in Table \ref{Tab:1}. In other works, the dimension $D$ appears in the denominator of the exponent of $n$, which implies that the convergence rate of the estimator deteriorates rapidly as the dimension $D$ grows. In contrast, our proposed approach suffers much less for it requires much less bases to achieve the same level of bias compared to other methods. The dimension $D$ only appears in the $\log$ term, and if $\beta$ satisfies some mild conditions, the convergence rate of our method is $n^{-1/2}$, which is {\it dimension-free} and is the optimal rate in learning problems. In general case, we show that the convergence rate achieves the almost optimal convergence rate, up to a $\log(n)$ multiplier. It is worth to note that even under the mixed order Sobolev spaces assumption, other estimators, for example, the empirical estimator, can still suffer the curse of dimensionality; see Theorem \ref{thm:empEst}. In summary, our contributions are:
\begin{itemize}[leftmargin=*]
    \item We propose a novel estimator adaptive hyperbolic cross estimator (AHCE) applied to high dimensional problems, which covers a large class of estimators. Under the mixed smooth Sobolev spaces, we show that the convergence rate of the proposed method under IPM loss is (nearly) dimension-free.
    \item We prove a lower bound of any estimators under IPM loss, and show that the convergence rate of AHCE is almost optimal, up to a log term.
    \item We apply AHCE in an adversarial framework, which are widely used in the analysis of GANs. We show that if a GAN is adaptive to the AHCE, then it enjoys a faster convergence rate which is almost dimension-free, where the dimension only appears in the $\log$ term while GANs with empirical density estimator have no such advantage regardless of the smoothness assumption of the true underlying density.
    \item We investigate the application of AHCE to goodness of fit test, and show that it can achieve high power, where the convergence rate again is almost dimension-free. In addition, we establish the asymptotic normality of the test statistics under mild condition imposed on the underlying density function. 
\end{itemize}

The rest of this work is arranged as follows. Section \ref{sec:pre} introduces useful notation and definitions of this paper. Section \ref{sec:AHCE} provides the explicit form of our estimator and analyze its convergence property under IPM. In Section \ref{sec:appli}, we show that the proposed estimators can by applied to other high-dimensional problems including GANs and goodness of fit test. Section \ref{sec:num} presents our numerical results. Conclusions are made in Section \ref{sec:conclu}, and technical proofs and details of the numerical studies are in the Appendix.

\section{Preliminaries}\label{sec:pre}

\textbf{Notation}\quad For non-negative sequences $\{a_n\}_{n\in\NatInt}$, $\{b_n\}_{n\in\NatInt}$,  $a_n\lesssim b_n$ denotes $\lim\sup_{n\to\infty}\frac{a_n}{b_n}<\infty$, and $a_n\asymp b_n$ denotes $a_n\lesssim b_n\lesssim a_n$. The input domain of the underlying function is normalized as $\Omega=[0,1]^D$. Let $L^p(\Omega)$ denote the set of functions $f$ with $\|f\|_p\coloneqq(\int_\Omega|f|^p)^{\frac{1}{p}}<\infty$ and $\|f\|_\infty\coloneqq\max_{\Bx\in\Omega}|f(\Bx)|$. Let $l^p$ denote the set of sequences $a\coloneqq\{a_n\}_{n\in\NatInt}$ with $\|a\|_{l^p}\coloneqq(\sum_{n\in\NatInt}|a_n|^p)^{\frac{1}{p}}<\infty$ and $\|a\|_{l^\infty}\coloneqq\max_{n\in\NatInt}|a_n|$ . For any $H\in\NatInt$, let $[H]\coloneqq\{1,\cdots,H\}$.  We use the notation $x\vee y\coloneqq\max\{x,y\}$ and $x\wedge y\coloneqq\min\{x,y\}$. Let $f\circ g\coloneqq f(g(\cdot)))$ denote function composition. For a general function space $F$ equipped with norm $\|\cdot\|_F$, define the $F$ ellipsoid as
$F(L)\coloneqq\{f: \|f\|_F\leq L\}.$

\textbf{H\"older spaces}\quad H\"older spaces of fractional order $r$ $\CalW^r_\infty$ are function spaces equipped with the norm
\begin{equation}
    \|f\|_{\CalW^r_\infty}\coloneqq \max_{ |\Bs|<r}\|D^{\Bs}f\|_\infty+ \max_{ |\Bs|=\lfloor r\rfloor}\max_{\Bx\neq \By\in\Omega}\frac{|D^{\Bs}f(\Bx)-D^{\Bs}f(\By)|}{\|\Bx-\By\|^{r-\lfloor r\rfloor} }\label{eq:Holder_norm}
\end{equation} 
where $|\bold{s}|=\sum_{i=1}^D|s_i|$ for multi-index $\bold{s}\in\NatInt^D$, and $D^{\Bs}f$ denotes the partial derivative $\frac{\partial^{|\Bs|}f}{\partial x_1^{s_1}\cdots\partial x_D^{s_D}}$.

\textbf{Mixed Smooth Sobolev spaces}\quad Mixed Smooth Sobolev spaces of fractional order $\CalH_{mix}^r$ can be defined via $L^2$ orthogonal basis as follows \citep{HyperbolicCross}.
% integer order can be defined as follows.
% \begin{defn}%[Mixed Smooth Sobolev spaces with integer orders]
% For integer $r\in\NatInt$, the \textit{mixed smooth Sobolev space } $\CalH_{mix}^r$ is a Hilbert space defined as
% $$\CalH_{mix}^r\coloneqq\{f: \|D^{r}_{x_1}D^{r}_{x_2}\cdots D^{r}_{x_D}f\|_{2}<\infty\},$$
% where $D^{r}_{x_i}$ denotes the $r$-weak partial derivative with respect to the $i^{\rm th}$ variable $x_i$.
% \end{defn}
% Also, the assumption that a general $D$-dimensional density $p$ is in a mixed smooth Sobolev space is reasonable (e.g., Gaussian and exponential family). 
% According to Chapter 3 of \cite{HyperbolicCross}, we can use $L^2$ orthogonal basis to characterize the mixed smooth Sobolev space $\CalH_{mix}^r$ to generalize it to fractional order:
\begin{defn}
\label{dfn:mixSobolevNorm}
For a constant $r\geq0$, the mixed smooth Sobolev space is equiped with the norm
% norm for any $f\in \CalH_{mix}^r$ is defined as 
$$
    \|f\|_{\CalH_{mix}^r}^2\coloneqq\sum_{\bold{k}\in\NatInt^D}\big(\big\|\sum_{\bold{s}\in\rho(\bold{k})}\hat{f}_\bold{s}\phi^F_{\Bs}\big\|_22^{r|\bold{s}|}\big)^2\asymp \sum_{\bold{k}\in\NatInt^D}\sum_{\bold{s}\in\rho(\bold{k})}\check{f}_\bold{s}^22^{2(r-\frac{1}{2})|\bold{s}|},
$$
where
\begin{align*}
    &\rho(\bold{k})=\{\bold{s}\in\mathbb{Z}^d: 2^{k_i-1}\leq |s_i|<2^{k_i},i\in[D]\},\\
    &\hat{f}_{\bold{s}}=\int_\Omega f(\Bx)\phi^F_{\Bs}(\Bx)d\Bx,\quad \check{f}_{\bold{s}}=\int_\Omega f(\Bx)\phi^W_{\Bs}(\Bx)d\Bx
\end{align*}
and $\{\phi^F_{\Bs}\}$ is the Fourier basis and $\{\phi^W_{\Bs}\}$ wavelet basis indexed by $\Bs$. 
% The \textit{mixed smooth Sobolev ellipsoid} is defined as $\CalH_{mix}^r(L)\coloneqq\{f:\|f\|_{\CalH_{mix}^r}\leq L\}$.
\end{defn}
The mixed smooth Sobolev space is tensor product of 1-D Sobolev space and, hence, it is equilvalent to the reproducing kernel Hilbert space (RKHS) generated by separable Mat\'ern kernel. Therefore, it is often considered as a reasonable model reducing the complexity in high-dimensional space \citep{kuhn2015approximation,dung2021deep}.
% , because it is closely related to the \textit{effective dimension} \citep{gilbert2021equivalence}. 
% It has been applied in many fields including quantum chemistry \citep{yserentant2010regularity}, finance \citep{wang2003effective,kuhn2015approximation}, etc.
As a simple example, the density $p$ of a $D$-dimensional random vector $X$ with mutually independent entries is of the form $p(\Bx)=\prod_{i=1}^Dp_i(x_i)$, so $\|p\|_{\CalH_{mix}^r}=\prod_{i=1}^D\|D^r_{x_i}p_i\|_2$, and $p$ is in a mixed smooth Sobolev space. 
Note that when $r=0$, the above definition coincides with the $L^2$ norm. The above definition of mixed smooth Sobolev space via $L^2$ basis enables $\CalH^r_{mix}$ to cover parametrized functions which can be represented by finitely many $L^2$ orthogonal basis functions.

\section{Adaptive Hyperbolic Cross Estimators}\label{sec:AHCE}
Suppose the true  density function $p$ lies in a mixed smooth Sobolev space $\CalH^\alpha_{mix}$ with some $\alpha>0$.
% , and the discriminator class $F_d=\CalH^\beta_{mix}$. 
% We first present our proposed estimator \textit{adaptive hyperbolic cross estimator} (AHCE). 
Write $\phi^F_{\Bs}$ or $\phi^W$ as $\phi_\Bs$ for simplicity. We start with a regularized empirical density adaptive to $n$ i.i.d. samples $\BX=\{X_i\}_{i=1}^n$ generated from $p$ defined as
\begin{equation}
    \label{eq: densityEstimator}
    \tilde{p}_{n}(\Bx)=\sum_{|\bold{k}|=0}^\infty\sum_{\bold{s}\in\rho(\bold{k})}b_{\bold{s},n}\tilde{p}_{\bold{s},n}\phi_{\bold{s}}(\Bx),
\end{equation}
where 
$\tilde{p}_{\bold{s},n}=\frac{1}{n}\sum_{i=1}^n{\phi_{\bold{s}}}(X_i).$ When $\phi_\Bs=\phi^F_\Bs$, \eqref{eq: densityEstimator} is called orthogonal series density estimators \citep{Hall86} and when $\phi_\Bs=\phi^W_\Bs$, \eqref{eq: densityEstimator} is wavelet estimator \citep{Donoho96}. The set  $\{\Bs: |\bold{k}|\leq l,\bold{s}\in\rho(\bold{k})\}$ is called hyperbolic cross \citep{HyperbolicCross} with level $l$.  It is essential to choose the sequence $b_{\bold{s},n}$, called ``smoothing policy'', to insures a pointwise convergence of random series. The ``smoothing policy'' should satisfy $b_{\bold{s}} \rightarrow 0$ as $|\bold{s}|\rightarrow \infty$, and for each fixed $\bold{s}$, $b_{\bold{s},n}\rightarrow 1$ as $n\rightarrow\infty$. We propose the following smoothing policies:

% \begin{enumerate}[leftmargin = *]
    % \item  
    1. $b_{\bold{s},n}=0$ for all $|\bold{s}|>l$, where $l$ satisfies $2^l \asymp n^{\frac{1}{2\alpha+1}}\big[\log n\big]^{\frac{D(\alpha+\nu+1)}{2\alpha+1}}$, $\nu>0$ and  $\phi_\Bs=\phi^F_\Bs$;
    
    % \item  
    2. $b_{\bold{s},n}=0$ for all $|\bold{s}|>l$, where $l$ satisfies $2^l \asymp n^{\frac{1}{2\alpha}}\big[\log n\big]^{\frac{D(\alpha+\nu+1/2)}{2\alpha}}$, $\nu>0$ and  $\phi_\Bs=\phi^W_\Bs$;
    
    % \item 
    3. $b_{\bold{s},n}=[1+c2^{2|s|\alpha}/n]^{-1}$, where $c=\Var[\phi_{\bold{s}}(X)]$ and  $\phi_\Bs=\phi^F_\Bs$.
    
% \end{enumerate}
Policies 1, 2, and 3 are modification of truncation smoothing (\cite{Hall1987}, wavelet threshold (without nonlinear correction) \citep{Donoho96} and Wahba smoothing \citep{wahba1981data}, respectively, such that they can be applied to the high dimensional problems.

\begin{rem}
AHCE is equivalent to kernel density estimator for some predetermined $\{b_{\Bs}\}$. We can first define the kernel
$k(\Bx,\By)\coloneqq\sum_{|\Bk|=0}^\infty\sum_{\Bs\in\rho(\Bk)}b_{\Bs}\phi_{\Bs}(\Bx)\bar{\phi_{\Bs}}(\By),$
then we can use the  features of kernel function to determine $\{b_{\Bs}\}$ for kernel reconstruction. For example, product-Mat\'ern type kernel and Gaussian kernel can be approximated by Fourier feature \citep{rahimi2008random}, or the \textit{Fej\'er} kernel in harmonic analysis.
% , defined by
% \begin{align*}
%     k(\Bx,\By;\bold{L})&=\prod_{j=1}^D\frac{1}{L_j}\sum_{t=0}^{L_j-1}\sum_{\omega=-t}^te^{2\pi\omega i|x_j-y_j|}\\
%     &=\prod_{j=1}^D\frac{1}{L_j}\frac{1-\cos(2\pi L_j|x_j-y_j|)}{1-\cos(2\pi|x_j-y_j|)}.
% \end{align*}
\end{rem}

\section{Minimax Rate under IPM}\label{subsec:minimaxr}
In this section,  we construct the minimax convergence rate of AHCE under IPM
\begin{equation}
    \label{eq:metric}
    \sup_{p\in \CalH^\alpha_{mix}}\E[d_{F_d}(\tilde{p}_n,p)],
\end{equation}
where $F_d$ can be another mixed smooth Sobolev space $\mathcal{H}^\beta_{mix}(L_\beta)$ or a H\"older space $\mathcal{W}^\beta_\infty(L_\beta)$. We first provide an upper bound of the IPM loss of a general AHCE.% as in the following theorem.
%\begin{theorem}[Upper Bound]
%\label{thm:densityIPM}
%For a density function $p\in\CalH^\alpha_{mix}(L_\alpha)$ with $\alpha>0$, suppose that the discriminator class is the mixed smooth Sobolev ellipsoid $F_d=\CalH^\beta_{mix}(L_\beta)$ with $\beta\geq 0$. Then, the regularized empirical density estimator $\tilde{p}_n$ in \eqref{eq: densityEstimator} with level $l \asymp \log_2 \big(n^{\frac{1}{2\alpha+1}}\big[\log n\big]^{\frac{(D-1)(\alpha+\beta+1)}{2\alpha+1}}\big)$ satisfies
%{$$\E d_{F_d}(p,\tilde{p}_n)\lesssim L_\alpha L_\beta n^{-\frac{\alpha+\beta}{2\alpha+1}}\big[\log n\big]^{(D-1)\frac{\alpha+\beta+1}{2\alpha+1}}+n^{-\frac{1}{2}}.$$}
%\end{theorem}

\begin{theorem}%[Upper Bound]
\label{thm:densityIPM}
Let $\tilde{p}_n$ be the AHCE defined in \eqref{eq: densityEstimator} with smoothing policy selected from policy 1-3.
%such that the level parameter  $l$ satisfies  $l \asymp \log_2 \big(n^{\frac{1}{2\alpha+1}}\big[\log n\big]^{\frac{D(\alpha+\nu+1)}{2\alpha+1}}\big)$. 
Then,
$$\sup_{p\in\CalH_{mix}^\alpha(L_\alpha)}\E d_{F_d}(p,\tilde{p}_n)\lesssim  n^{-\frac{\alpha+\nu}{2\alpha+1}}\big[\log n\big]^{D\frac{\alpha+\nu+1}{2\alpha+1}}+n^{-\frac{1}{2}}$$
for any $\alpha>0$, where $\nu = \beta$ if $F_d=\CalH^\beta_{mix}(L_\beta)$ and $\nu = \frac{\beta}{D}$ if $F_d=\mathcal{W}^\beta_\infty(L_\beta)$.
% \begin{equation*}
%     \nu=\begin{cases}
%     \beta\quad &\text{if}\ F_d=\CalH^\beta_{mix}(L_\beta)\\
%     \frac{\beta}{D} &\text{if}\ F_d=\mathcal{W}^\beta_\infty(L_\beta)
%     \end{cases}.
% \end{equation*}
\end{theorem}
In Theorem \ref{thm:densityIPM}, the dimension $D$ only appears on the exponent of the logarithm term. Therefore, the convergence rate is very close to the dimension-free polynomial term $n^{-\frac{\alpha+\nu}{2\alpha+1}}$. In fact, if we replace AHCE by other density estimators, such as the density estimator proposed in \cite{Donoho96} or in \cite{liang2018generative}, the fast convergence rate is not guaranteed. This is because the constructions of these estimators use too many uninformative basis functions in approximating the true underlying density and, hence, fail to control the variance in an optimal level. 

In order to further show the efficiency of AHCE, we provide a lower bound for \textit{any} density estimator %adaptive to samples $\{X_i\}_{i=1}^n$ 
(not restricted to AHCE), as stated in the following theorem.
%\begin{theorem}[Lower Bound]
%\label{thm:minmax_lower}
%For any $\alpha>0$, let $\CalH^\alpha_{mix}(L_\alpha)$ be a mixed smooth Sobolev space.
% that covers the true underlying density $p$. 
%For any $\beta\geq 0$, let $F_d=\CalH^\beta_{mix}(L_\beta)$ be the discriminator class. Given any estimator $p_n$ adaptive to the $n$ i.i.d samples $\{X_i\}_{i=1}^n\sim p$, we have
%\begin{align*}
%    &\quad\inf_{p_n}\sup_{p\in\CalH^\alpha_{mix}(L_\alpha)}d_{F_d}(p_n,p)\\
%    &\gtrsim  L_\alpha L_\beta n^{-\frac{\alpha+\beta}{2\alpha+1}}\big[\log n\big]^{(D-1)\frac{\alpha+\beta}{2\alpha+1}}+n^{-\frac{1}{2}}.
%\end{align*}
%\end{theorem}

\begin{theorem}%[Lower Bound]
\label{thm:minmax_lower}
% that covers the true underlying density $p$. 
Given any estimator $p_n$ adaptive to the $n$ i.i.d samples $\{X_i\}_{i=1}^n\sim p$, for any $\alpha>\frac{1}{2}$ and $F_d=\CalH^\beta_{mix}(L_\beta)$,
\begin{align*}
\inf_{p_n}\sup_{p\in\CalH^\alpha_{mix}(L_\alpha)}\E d_{F_d}(p_n,p)\gtrsim
 n^{-\frac{\alpha+\beta}{2\alpha+1}}\big[\log n\big]^{(D-1)\frac{\alpha+\beta}{2\alpha+1}};
\end{align*}
for any $\alpha>0$ and $F_d=\mathcal{W}^\beta_\infty(L_\beta)$,
\begin{align*}
  \inf_{p_n}\sup_{p\in\CalH^\alpha_{mix}(L_\alpha)}\E d_{F_d}(p_n,p)\gtrsim
     n^{-\frac{\alpha+\beta/D}{2\alpha+1}}. 
\end{align*}
\end{theorem} 
It can be seen from Theorems \ref{thm:densityIPM} and \ref{thm:minmax_lower} that the convergence rate of AHCE is very close to the optimal convergence rate among all estimators adaptive to the training samples under IPM --- they only differ by a log term.  When the discriminator class is smooth enough, the log term  difference is cancelled by the extra smoothness, and AHCE is optimal with a rate $n^{-\frac{1}{2}}$.
% so we can easily conclude that AHCE is optimal with a rate $n^{-\frac{1}{2}}$. 

\begin{rem}
In Theorems \ref{thm:densityIPM} and \ref{thm:minmax_lower}, we consider general H\"older IPM losses ($F_d=\CalW^\beta_\infty$) , which include some special cases; for example, Total Variation metric ($\beta=0$) and Wasserstein distance ($\beta=1$).
\end{rem}

\section{Applications to High-Dimensional Problems}\label{sec:appli}
% We have shown the approximation power of AHCE in the previous section. 
In high-dimensional scenarios, we can see that AHCE gives highly accurate estimation to the true density function, with nearly dimension-independent rate of convergence under the $L^2$ loss and with fully dimension-free rate of convergence under IPM with smooth-enough discriminator in an asymptotic sense. Therefore, the general results in Section \ref{subsec:minimaxr} can serve as a promising tool to analyze high-dimensional and large-scale data sets. In this section, we first apply AHCE to generative models to show that it can improve the performance of these models; then we apply AHCE to goodness of fit test to show that it can achieve high power under small size of samples.

\subsection{Improved Rate for GANs}\label{subsec:improvegan}
GANs \citep{Goodfellow14} are powerful unsupervised methods in learning and sampling from high-dimensional data distributions. The goal of GANs is to search over a function space parametrized by deep neural networks (DNNs) that can serve as a proxy to generate samples similar to observed independent data points. We call such a space \textit{generator} class under the GAN framework. Given an input from an easy-to-sample distribution, we want the distribution of the generator output to be close to the data distribution. To this end, a \textit{discriminator} is trained to measure the difference between data samples and generated samples, and the generator is trained to confuse the discriminator. Specifically, GANs aim at solving the following  problem to get a random sample generator $g^*$:
\begin{equation}
    \label{eq:GAN}
    g^*=\underset{g\in Q_g}{\arg\inf}\sup_{f\in F_d} \E_{X\sim \mu}\big[f\circ g(X)\big] -\E_{Y\sim p}\big[f(Y)\big]
\end{equation}
where $Q_g$ is the generator class, $\mu$ is some easy-to-sample density and $p$ is the true underlying density. Examples include the GAN under Total Variation metric \citep{Goodfellow14}, where the discriminator class is all bounded functions ($F_d=\CalW^0_\infty$); the Wasserstein GAN \citep{Arjovsky17}, where the discriminator class is all  Lipschitz functions ($F_d=\CalW^1_\infty$); and the MMD GAN \citep{li2017mmd}, where $F_d=\CalH^\beta_{mix}$, which is equivalent to the RKHS generated by a separable Mat\'ern kernel.

% When the discriminator class is all bounded functions ($F_d=\CalW^0_\infty$), \eqref{eq:GAN} coincides with the GAN under Total Variation metric \citep{Goodfellow14}; when the discriminator class is all  Lipschitz functions ($F_d=\CalW^1_\infty$), \eqref{eq:GAN} coincides with the Wasserstein GAN \citep{Arjovsky17}; when $F_d=\CalH^\beta_{mix}$, which is equivalent to the reproducing kernel Hilbert space (RKHS) generated by a separable Mat\'ern kernel, \eqref{eq:GAN} coincides with the MMD GAN \citep{li2017mmd}. 
% It is worthy to note that we consider a general case where $F_d=\CalW^\beta_\infty$ with $\beta >0$; thus the GAN under Total Variation metric and the Wasserstein GAN are two special cases under our analysis.
% formulation of GANs in \cite{Goodfellow14} (Total Variation metric); when the discriminator class is all  Lipschitz functions: $F_d=\CalW^1_\infty$, \eqref{eq:GAN} coincides with the formulation of Wasserstein GANs in \cite{Arjovsky17}; when $F_d=\CalH^\beta_{mix}$, which is the reproducing kernel Hilbert space (RKHS) generated by separable Mat\'ern kernel, \eqref{eq:GAN} coincides with the formulation of MMD GANs in \cite{li2017mmd}.

In practice, both the generator class $Q_g$ and the discriminator class $F_d$ are estimated by deep neural networks. In our paper, we consider the following \textit{ReLU} nets.
\begin{defn}
Let $\eta(x)=\max(x,0)$ and for a vector $\Bx$, $\eta(\Bx)$ is applied element-wisely. A ReLU network $\Phi(H,W,S,B)$ is the space of functions in the form
$$f(\Bx)=[A_{H}\eta(\cdot)+b_{H}]\circ\cdots\circ[A_{1}\Bx+b_{1}],$$
where $A_h\in\Real^{W\times W}, b_h\in\Real^W, h\in[H],$
$\sum_{h=1}^H\|A_h\|_{l^0}+\|b_h\|_{l^0}\leq S$, and $\max_{h\in[H]}\|A_h\|_{l^\infty}\vee\|b_h\|_{l^\infty}\leq B.$
\end{defn}
Because the true underlying density $p$ is inaccessible, it is replaced by a regularized empirical density $\hat{p}$  adaptive to $n$ i.i.d samples $\{X_i\}_{i=1}^n\subset\Real^D$ generated from $p$ as
\begin{align}\label{eq:empde}
    \hat{p}(\Bx)=\frac{1}{n}\sum_{i=1}^n\delta(\Bx-X_i),
\end{align} 
where $\delta$ is the Dirac delta function. Then, \eqref{eq:GAN} becomes
\begin{equation}
\label{eq:GAN_empirical}
    \hat{g}_n=\argmin_{g\in \Phi_g}\max_{f\in\Phi_d}\E_{X\sim\mu}\big[f\circ g(X)\big]-\E_{Y\sim \hat{p}}\big[f(Y)\big].
\end{equation}
%A popular \textit{non-parametric} formulation of the GAN is \citep{dziugaite2015training,Arjovsky17,li2017mmd,li2015generative,liang2018generative,liu2017approximation,Uppal19}
%\begin{equation}\label{eq:ganform}
%        \min_{q\in Q_g}d_{F_d}(q,p),
%\end{equation}
%where $p$ is the true underlying density of data, and $Q_g$ is the generator class. In this paper, we consider the case $Q_g$ and $F_d$ are parametrized by the ReLU network defined as follows

% A deep neural network (DNN) can be viewed as a parametrized function space defined as
% We first define the deep neural networks (DNNs). Let $f_{\Btheta}$ be a function parametrized by $\Btheta\in\boldsymbol{\Theta}$. 

 %that estimates
%the probability that a sample came from $p$ rather than the generator. 

Obviously, the performance of $\hat{g}_n$ is highly correlated to the IPM distance between the true density $p$ and its estimator $\hat{p}$. 
Because the purely empirical density estimator does not utilize any smoothness assumption about $p$, it can have a much slower convergence rate, as shown in the following theorem.
% we have the following inferior lower bound:
\begin{theorem}
\label{thm:empEst}
Let  $\hat{p}$ be an empirical density estimator adaptive to the n i.i.d samples $\{X_i\}_{i=1}^n\sim p$, we have  almost surely
\begin{equation*}
    \sup_{p\in\CalH^\alpha_{mix}(L_\alpha)}d_{F_d}(\hat{p},p)\gtrsim 
    \begin{cases}
    n^{-\frac{\beta}{D}}\quad &\text{if}\ F_d=\CalW^\beta_\infty\\
    n^{-\frac{1}{D}}\quad &\text{if}\ F_d=\Phi(H,W,S,B)
    \end{cases}
\end{equation*}
where $\Phi(H,W,S,B)$ satisfies
$
    H\geq 3+2\lceil log_2\big(\frac{3^D}{ c}\big)+5\rceil\lceil\log_2 D\rceil,
    W\geq 40Dn,\ S\geq HW^2, B\geq 4\sqrt[D]{2n}.
$

\end{theorem}
The slow convergence rate shown in Theorem \ref{thm:empEst} is because the empirical density estimator does not utilize any smoothness information of the underlying function. Thus, the GANs with empirical density estimator encounter the curse of dimensionality.
In contrast, if we replace the purely empirical density estimator by AHCE, we can achieve a much faster convergence rate, which is almost dimension-free. 
\begin{theorem}%[Wasserstein GAN]
\label{thm:GAN-improved}
Suppose Assumption \ref{A1} holds. Suppose the empirical density estimator $\hat{p}$ in \eqref{eq:GAN_empirical} is replaced by the AHCE $\tilde{p}_n$ , the discriminator class  $\Phi_d=\Phi(H_d,W_d,S_d,B_d)$ satisfies
\[H_d\asymp \log n+1, W_d+S_d\asymp (n^{\frac{\alpha D+\beta}{2\alpha\beta+\beta}}+n^{\frac{D}{2\beta}})H_d,\]
and the generator class is a ReLU net $\Phi_g=[g_1,...,g_D]$ with $g_k\in \Phi(H_g,W_g,S_g,B_g)$, $k=1,...,D$ satisfies
\begin{align*}
    &H_g\gtrsim \log n+1, W_g+S_g\gtrsim (n^{\frac{\alpha D+\beta}{2\alpha^2+2\alpha+1/2}}+n^{\frac{D}{2\alpha+1}})H_g.
\end{align*}
Then we have the following convergence rate for $\hat{g}_n$ as
\begin{align*}
    \E d_{\CalW^\beta_\infty(L_\beta)}(\hat{g}_n{\#\mu},p)\lesssim n^{-\frac{\alpha+\beta/D}{2\alpha+1}}\big[\log n\big]^{\kappa}+n^{-\frac{1}{2}},%\label{eq:estimationErr}
\end{align*}
where $\hat{g}_n{\#\mu}$ denotes the pushforward measure of $\mu$ by GAN $\hat{g}_n$ and $\kappa=D\frac{\alpha+\beta/D+1}{2\alpha+1}$.
\end{theorem}

\begin{rem}
Our analysis is different from \cite{liang2018generative} and \cite{Uppal19} in that we consider the neural network as a Transformation from the easy-to-generate distribution to the target distribution instead of treating $\Phi_g$ as solely a density estimator. As a result, we need to apply techniques in optimal transport to prove the existence of the transformation.
\end{rem}

We can see from Theorem \ref{thm:empEst} that if the discriminator class is a large-enough ReLU net then the error rate of the empirical density is doomed to be slow for high-dimensional data. As a result, any random sample generator adaptive to the empirical density cannot have a small error rate as AHCE. On the other hand, it can be easily checked that there must be some ReLU nets that can satisfy the requirements of discriminator class in both Theorems \ref{thm:empEst} and \ref{thm:GAN-improved}. If  the discriminator of a GAN is selected among this class of ReLU nets, the advantage of AHCE is obvious --  the resulting GAN generates samples much closer to the real distribution than its counterpart, which uses empirical density.

% When the samples $\{X_i\}_{i=1}^n$ are images or other types of data instead of purely numerical vectors, AHCE becomes impractical. However, our analysis is still meaning because, from a different perspective, Theorem \ref{thm:GAN-improved}  shows that the \emph{smoothness} of discriminator $\Phi_d$ should be adaptive to the number of samples. This is because for any $f\in\Phi_d$, the expectation of $f(X)$ induced by $\tilde{p}_n$ can be written as

Another interpretation of Theorem \ref{thm:GAN-improved} is that the \emph{smoothness} of discriminator $\Phi_d$ should be \emph{adaptive} to the sample size. To be more specific, for any $f\in\Phi_d$, the expectation of $f(X)$ induced by $\tilde{p}_n$ can be written as
\begin{equation}
        \E_{X\sim\tilde{p}_n}[f(X)]=\frac{1}{n}\sum_{i=1}^n\sum_{\Bk\in\NatInt^D,\Bs\in\rho(\Bk)}b_{\Bs,n}\hat{f}_{\Bs}\phi_{\Bs}(X_i)\coloneqq \frac{1}{n}\sum_{i=1}^n\tilde{f}_n(X_i)=\E_{X\sim\hat{p}}[\tilde{f}_n(X)]\label{eq:smoothing}
\end{equation}
where $\{b_{\Bs,n}\}$ is our smoothing policy 
% satisfying the requirement of AHCE 
and $\hat{f}_{\Bs}=\int f\phi_{\Bs}dx$. The $\tilde{f}_n$ can be viewed as a smoothing function of $f$. This indicates that if the discriminator class $F_d$ is adaptively chosen such that the distance between any $f$ in $F_d$ and its smoothing $\tilde{f}_n$ is close, then the fast convergence rate of the GAN \eqref{eq:GAN_empirical} with the empirical density estimator is also guaranteed.
% , which is verified in the following theorem.

% Therefore, This leads to the following theorem
\begin{theorem}%[Smoothing of $\Phi_d$]
\label{thm:smoothing}
Suppose the true density $p$, the predetermined density $\mu$ and the ReLu net $\Phi_g$ satisfy the same conditions in Theorem \ref{thm:GAN-improved}. %Let $F_d$ be the RKHS generated by the following kernels
%\begin{align*}
%    &k_{\beta}(\Bx,\By)=\\
%    &\begin{cases}
%    \prod_{j=1}^D\frac{2^{1-\nu}}{\Gamma(\nu)}(\frac{|x_j-y_j|}{\rho})^{\nu}K_\nu(\frac{|x_j-y_j|}{\rho})\quad &\text{if} \ \beta\in(\alpha,\infty)\\
%    \exp\{\frac{|x_j-y_j|^2}{\rho}\} &\text{if} \ \beta=\infty
%    \end{cases}.
%\end{align*}
%where $\Gamma$ is the gamma function, $K_\nu$ is the modified Bessel function of the second kind and $\nu=\beta-1/2$. 
Suppose the empirical discriminator $\Phi_d$ (neural net or other approximator) satisfies Asusmption \ref{A2}. Then the GAN $\hat{g}_n$ in \eqref{eq:GAN_empirical} has the following convergence rate
\[\E d_{\CalW^\beta_\infty}(\hat{g}_n\#\mu,p)\lesssim n^{-\frac{\alpha+\beta/D}{2\alpha+1}}\big[\log n\big]^{D\frac{\alpha+\beta/D+1}{2\alpha+1}}+n^{-\frac{1}{2}}.\]
\end{theorem}

\begin{rem}
Although Theorem \ref{thm:smoothing} implies that by carefully choosing the discriminator class, the GANs with empirical density estimator can still achieve a faster convergence rate, the conditions are much more complicated than those in Theorem \ref{thm:GAN-improved}. Moreover, Theorem \ref{thm:smoothing} suggests that discriminators with smooth activation functions can improve performance, which is empirically verified in Section \ref{sub:mnist}.
% and {\color{red}left in appendix. Although the conditions are complicate, they suggest that discriminators with smooth activations can improve performance.} % Therefore, we believe that AHCE is still a good 
\end{rem}

%Theorems \ref{thm:empirical-density-GAN} and \ref{thm:GAN-improved} indicate that if the discriminator class is close to $L^2$ space, then the GAN associated with AHCE achieves a faster rate than the one associated with the purely empirical density, which ignores the smoothness of the distribution.

\subsection{Goodness of Fit Test}
The problem of testing the goodness of fit of a model is an enduring and ever-growing research area,  with various tests polices proposed in the statistics community. Traditional methods include the maximum likelihood ratio test, Kolmogorov-Smirnov test, and the $\chi^2$ test. An alternative approach is to use the smoothed $L^2$ distance between the empirical characteristic function of the sample and the characteristic function of the target density.  A recent popular method \citep{jordan2016,Liu20} employs the Stein operator in a reproducing kernel Hilbert space while this method is suboptimal \citep{ming2017}. To resolve this suboptimal in test power, \cite{ming2019} propose a Maximum Mean Discrepancy tests using Gaussian kernel with an appropriately chosen scaling parameter.

In this subsection, we propose a non-parametric statistical test for the goodness of fit problem based on AHCE. Consider the goodness of fit test problems with general alternatives, i.e. $H_{0}: p=p_{0}$ v.s. $H_{1}: p\neq p_{0}$, where $p$ is the distribution of the observed samples, and $p_0$ is the target distribution. Our test statistic is based on the empirical estimate of the $L^2$ distance between the distribution of the observed samples $\tilde{p}_{n}$ and the target distribution $p_0$, taking the form of a U-statistic in terms of an adaptive kernel. For the simplicity in the theoretical analysis, we use wavelet expansion and choose the truncation smoothing policy with level $2^l\asymp n^{\frac{1}{2\alpha}}[\log n]^{\frac{\alpha(\alpha+0.5)}{2\alpha}}$ in this subsection.
%To gain the quantile of the null distribution, we will present a test procedure utilizing a wild bootstrap method for i.i.d. samples.

We first show the consistency of the AHCE in probability and an upper bound for the mean integrated squared error (MISE) of the AHCE in the following propositions.

\begin{prop}
\label{pro:consistent}
%Let $p\in\CalH^{\alpha}_{mix}$ be a probability density function. Let $\mathcal{X}=\{X_{1}, \cdots, X_{n}\}\subset [0,1]^D$ be samples drawn from $p$, and $\tilde{p}_n$ be the estimated density function given by \eqref{eq: densityEstimator} with the truncation smoothing policy with level $2^l\asymp n^{\frac{1}{2\alpha}}[\log n]^{\frac{\alpha(\alpha+0.5)}{2\alpha}}$, then the AHCE $\tilde{p}_n$ is consistent and satisfies $\lim_{n\rightarrow\infty}P(\|\tilde{p}_{n}-p\|^2_{2}=0)=1.$
%$$P(\lim_{n\rightarrow\infty}\|\tilde{p}_{n}-p\|^2_{2}=0)=1.$$
Let $p\in\CalH^{\alpha}_{mix}$ be a probability density function with $\alpha>0$. Let $\mathcal{X}=\{X_{1}, \cdots, X_{n}\}\subset [0,1]^D$ be samples drawn from $p$, and $\tilde{p}_n$ be the estimated density function given by \eqref{eq: densityEstimator}, then the AHCE $\tilde{p}_n$ is consistent and satisfies $\lim_{n\rightarrow\infty}P(\|\tilde{p}_{n}-p\|^2_{2}=0)=1.$
\end{prop}
Proposition \ref{pro:consistent} guarantees that our test statistic can distinguish two distributions, with a testing power (i.e., the probability of correctly rejecting null hypothesis $H_0: p=p_{0}$) tending to one asymptotically.
\begin{prop}
\label{pro:mise_error_bound}
Suppose the conditions of Proposition  \ref{pro:consistent} hold.
% Assuming the underlying density function $p\in\CalH_{mix}^{\alpha}(L_\alpha)$, and the regularized empirical density estimator $\tilde{p}_n$ in equation \ref{eq: densityEstimator} with the truncation level $l\asymp \log_{2}[n^{\frac{1}{2\alpha+1}}[\log n]^{\frac{(D-1)(\alpha+1)}{2\alpha+1}}]$, 
The MISE satisfies the following upper bound 
$$
    \E ||p-\tilde{p}_n||_{2}\lesssim L_\alpha n^{-\frac{\alpha}{2\alpha+1}}\left(\log n\right)^{\frac{D(\alpha+1)}{2\alpha+1}}.
$$
\end{prop}
Proposition \ref{pro:mise_error_bound} implies that the testing power of the proposed method does not suffer much from the ``curse of dimensionality''.
Following Proposition \ref{pro:mise_error_bound}, we can construct a goodness of fit test statistic based on the $L^2$ norm distance between $\tilde{p}$ and $p_0$.
% %Having the orthogonality of the series $\Phi_{\bold{s}}$, we can construct a goodness of fit test statistic by utilizing the $L^2$ norm distance between $\tilde{p}$ and $p$.
% \begin{align*}
% \|\tilde{p}_{n}-p\|_{2}^2=&\frac{1}{n^2}\sum_{i,j}\sum_{|\bold{k}|=0}\sum_{\bold{s}\in\rho(\bold{k})}\phi_{\bold{s}}(X_i)\phi_{\bold{s}}(X_j)b_{\bold{s}, n}^2\\
% &-\frac{2}{n}\sum_{i}\sum_{|\bold{k}|=0}\sum_{\bold{s}\in\rho(\bold{k})}\phi_{\bold{s}}(X_i)b_{\bold{s},n}\alpha_{\bold{s}}\\
% &+\sum_{|\bold{k}|=0}\sum_{\bold{s}\in\rho(\bold{k})}\alpha_{\bold{s}}^2,
% \end{align*}
% where $\alpha_{\bold{s}}=\mathbb{E}\phi_{\bold{s}}(X)$, and $X\sim p$.
To this end, we propose a degenerate U-statistics $T_{n}=\frac{1}{n(n-1)}\sum\limits_{1\leq i\neq j\leq n} {\tilde{H}}_n(X_i, X_j)$, where 
\begin{align*}
  \tilde{H}_n(\Bx,\By)=&H_n(\Bx,\By)-\mathbb{E}_{X\sim p_{0}}H_n(\Bx, X)-\mathbb{E}_{X\sim p_{0}}H_n(X, \By)+\mathbb{E}_{X, Y\sim p_{0}}H_n(X, Y), 
\end{align*}
and $
H_n(\Bx,\By)=\sum\limits_{|\bold{k}|=0}\limits^{l}\sum\limits_{\bold{s}\in\rho(\bold{k})}\phi_{\bold{s}}(\Bx)\phi_{\bold{s}}(\By).$
%\begin{align*}
%    \tilde{H}_n(x,y)=&H_n(x,y)-\mathbb{E}_{X\sim p_{0}}H_n(x, X)\\
%&-\mathbb{E}_{X\sim p_{0}}H_n(X, y)+\mathbb{E}_{X, Y\sim p_{0}}H_n(X, Y), 
%\end{align*}
%and 
%$H_n(x,y)=\sum\limits_{|\bold{k}|=0}\limits^{l}\sum\limits_{\bold{s}\in\rho(\bold{k})}\phi_{\bold{s}}(x)\phi_{\bold{s}}(y).$
Note that $T_n$ is an estimate of $\|\tilde{p}_{n}-p_0\|_{2}^2$ (with appropriate normalization). Therefore, under the null hypothesis, it converges to zero with an almost dimension-free convergence rate according to Proposition \ref{pro:mise_error_bound}. Furthermore, we can construct an asymptotic distribution of $T_n$ under general condition, shown in the following proposition.
\begin{theorem}\label{thm:wchi}Let $\{Z_i\}_{i=1}^\infty$ be a sequence of independent standard Gaussian random variables.  Assume $p_0\in\CalH_{mix}^{\alpha}(L)$ with $\alpha>0$ and $\mathbb{E}\tilde{H}_n^2(X,Y)<\infty$, $\forall n\geq 1$. Under the null hypothesis, we have $nT_{n}\xrightarrow[]{d} \sum_{i=1}^\infty\lambda_i(Z_i^2-1)$, where $\lambda_i$ are the eigenvalues of the kernel operator such that there exist functions $\psi_{i}$, $\int_{[0,1]^{D}}\tilde{H}(\Bx,\By)\psi_{i}(x)p_{0}(\Bx)d\Bx=\lambda_{i}\psi_{i}(\By)$, and $\tilde{H}(\Bx,\By)=\lim_{n\rightarrow\infty}\tilde{H}_n(\Bx,\By)$. 
\end{theorem}
%\begin{rem} if we apply the truncation policy with level $l$, the kernel function $H(x,y)=\sum_{|\bold{k}|\leq l}\sum_{\bold{s}\in\rho(\bold{k})}[\phi_{\bold{s}}(x)\phi_{\bold{s}}(y)+\alpha_{\bold{s}}^2-\phi_{\bold{s}}(x)\alpha_{\bold{s}}-\phi_{\bold{s}}(y)\alpha_{\bold{s}}]$.
%\end{rem}
In practice, it is usually hard to estimate the eigenvalues of the kernel operator. We suggest applying a wild bootstrap technique to estimate the quantiles of the null distribution. The key step of the wild bootstrap is to simulate a simple Markov chain taking values in $\{-1,1\}$. First, we take the initial state $B_{1,n}=1$, then we update the Markov chain by $B_{j,n}=\mathbbm{1}(U_{j}>\alpha)B_{j-1,n}-\mathbbm{1}(U_{j}<\alpha)B_{j-1,n}$, where $U_j\sim \text{Unif}(0,1)$ are i.i.d random variables and we take $\alpha=0.5$. This yields the bootstrap statistics as $A_{n}=\frac{1}{n(n-1)}\sum\limits_{i\neq j} B_{i,n}B_{j,n}H_{n}(X_i, X_j).$ The implement of goodness of fit test procedure is as in Algorithm \ref{Alg1}. 
\begin{algorithm}
\SetKwInOut{Input}{input}
\SetKwBlock{Begin}{begin}{end}
\Input{Samples $\mathcal{X}=\{X_{1}, \cdots, X_{n}\}\subset [0,1]^D$, a null hypothesis $H_0: p=p_{0}$.}
\Begin{
Step 1. Calculate the test statistics $T_n$ using $\mathcal{X}$\;
Step 2. Apply the wild bootstrap samples $\{A_n\}_{i=1}^{B}$ and estimate the empirical threshold  at level $1-\alpha$\;
Step 3. If the test statistic $T_n$ exceeds the threshold, reject the hypothesis; otherwise, do not reject the hypothesis.
}
\caption{Goodness of fitting test based on AHCE}\label{Alg1}
\end{algorithm}
\vspace{-3mm}

Limit theory for degenerate U-statistics when the kernel $H$ is fixed has been well studied.  In that case, the limit distribution is a weighted chi-square as the result in Theorem \ref{thm:wchi}, and cannot be derived using classical Martingale methods.
However, in certain cases, in which the kernel $H$ depends on the sample size $n$, a normal distribution can be obtained asymptotically. We prove asymptotic normality of our degenerate U-statistics $T_n$ under mild conditions on the underlying density in the following theorem. 
\begin{theorem}\label{thm:asymptotic_normalty}  If the true density function is lower bounded and in $\CalH^\alpha_{mix}$ with $\alpha>1$, then  under the null hypothesis, we have $\frac{T_{n}}{\sigma(T_{n})}\xrightarrow[]{d} \mathcal{N}(0,1).$ Moreover, if we use a U-statistics $\hat{\sigma}_{n}$ to estimate $\sigma(T_n)$, we also have 
$\frac{n T_{n}}{\sqrt{2}\hat{\sigma}_{n}}^2\xrightarrow[]{d} \mathcal{N}(0,1),$ where 
\begin{align*}
 \hat{\sigma}_{n}^2&=\frac{1}{n(n-1)}\sum\limits_{1\leq i\neq j\leq n}H_{n}^{2}(X_{i}, X_{j})\\
 &-\frac{2(n-3)!}{n!}\sum\limits_{\substack{1\leq i, j_1, j_2\leq n \\ |\{i, j_1, j_2\}|=3}}H_{n}(X_{i}, X_{j_{1}})H_{n}(X_{i}, X_{j_{2}})\\
 &+\frac{(n-4)!}{n!}\sum\limits_{\substack{1\leq i_1, i_2, j_1, j_2\leq n \\ |\{i_1, i_2, j_1, j_2\}|=4}}H_{n}(X_{i_{1}}, X_{j_{1}})H_{n}(X_{i_{2}}, X_{j_{2}}).
\end{align*}
\end{theorem}

\begin{rem}
Theorem \ref{thm:asymptotic_normalty} requires that $\alpha>1$, which imposes a stronger condition on the underlying density function than that in Propositions \ref{pro:consistent} and \ref{pro:mise_error_bound}. It is necessary for the asymptotic normality results, because otherwise the statistic $T_n$ may have a heavy tail which is caused by the roughness of the density function.
\end{rem}

\section{Numerical Experiments}\label{sec:num}
In this section, we first compare the $L^2$ convergence rates of AHCE to empirical density estimation and two benchmark density estimators under mixed smooth Sobolev IPM loss; then we illustrate how AHCE and smooth density estimators can help improve the performance of GANs; lastly, we apply AHCE in the goodness of fit test and show that the statistics associated with AHCE are more sensitive than statistics based on the commonly used Gaussian kernel. More details of the numerical experiments can be found in the Appendix.

\subsection{Synthetic Data}\label{subsec:syn} 

\textbf{$L^2$ Convergence Rate}\quad %{Comparison}\quad
% In this subsection, 
We compare AHCE to two benchmark density estimators: kernel density estimator with Gaussian kernel (KDE) and B-spline distribution estimator with density (BSDE). We consider two distributions: a 5-dimensional Beta-distribution $\text{Beta}(a,b)$ with $a=[2, 2 ,2 ,5 ,5]$ and $b=[5 ,5 ,2, 2 ,2]$, and a 10-dimensional $\text{Beta}(a,b)$ with $a=[2,\cdots,2]$ and $b=[5,\cdots,5]$. The estimated root mean square errors (RMSE) 
% for three density estimators 
are shown in Panels 1 and 2 of Figure \ref{fig:DensityCompare}. It can be seen that AHCE has the best performance in all experiments with the smallest RMSE. 

% \textbf{Experiment Setup}: 

\begin{figure}[t!]
\centering
\includegraphics[width=0.235\textwidth]{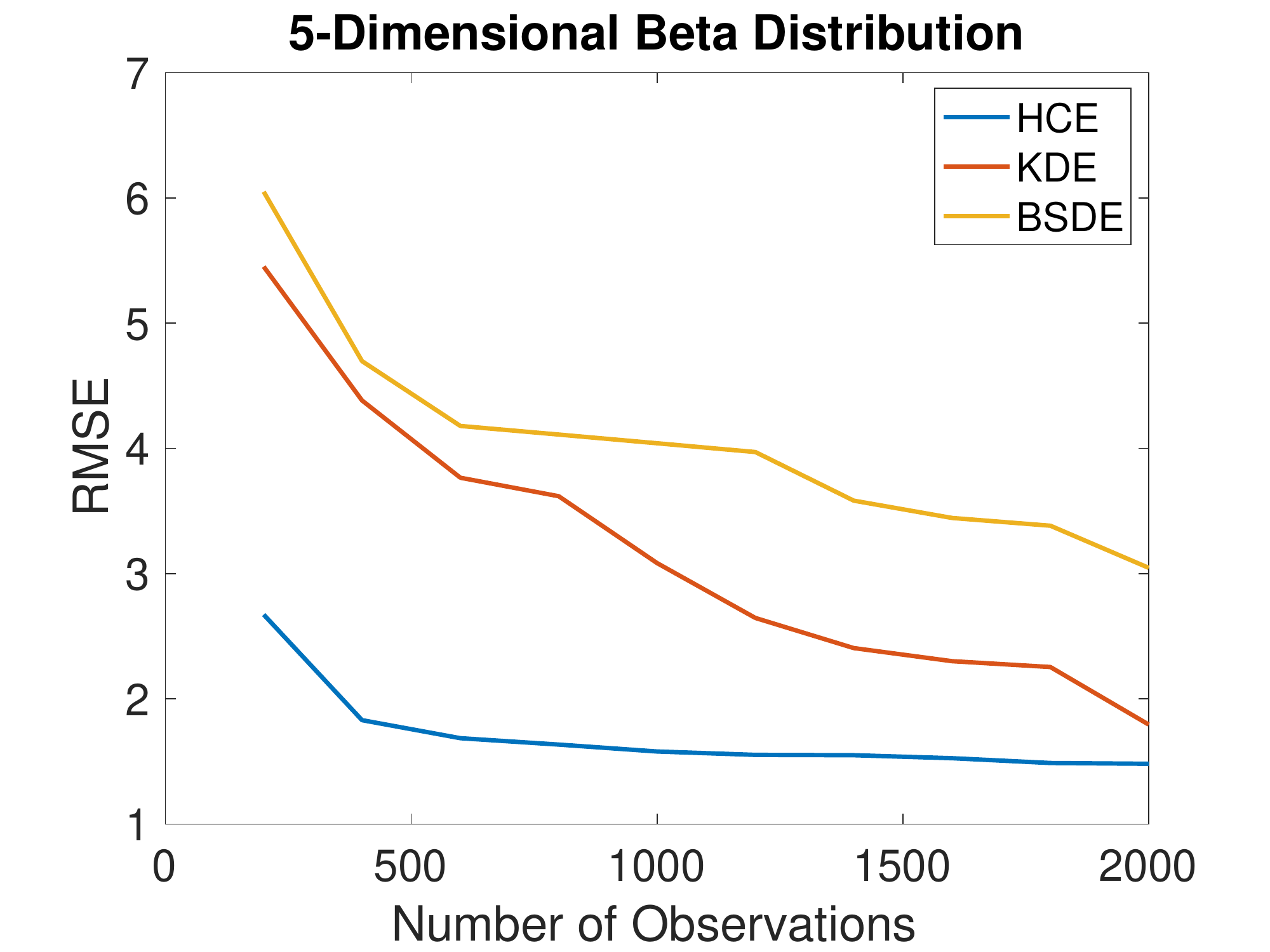}  
\includegraphics[width=0.235\textwidth]{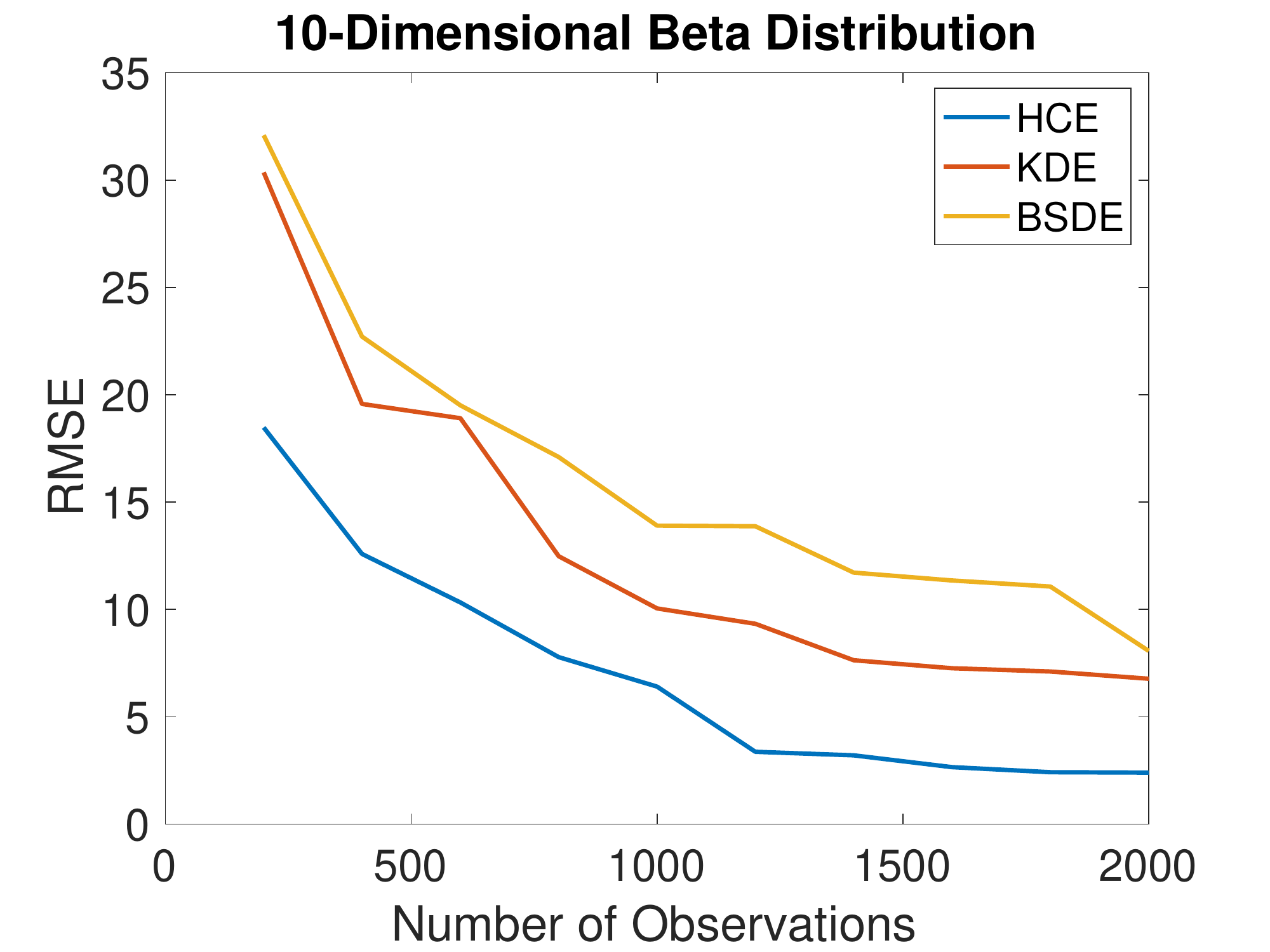} 
\includegraphics[width=0.235\textwidth]{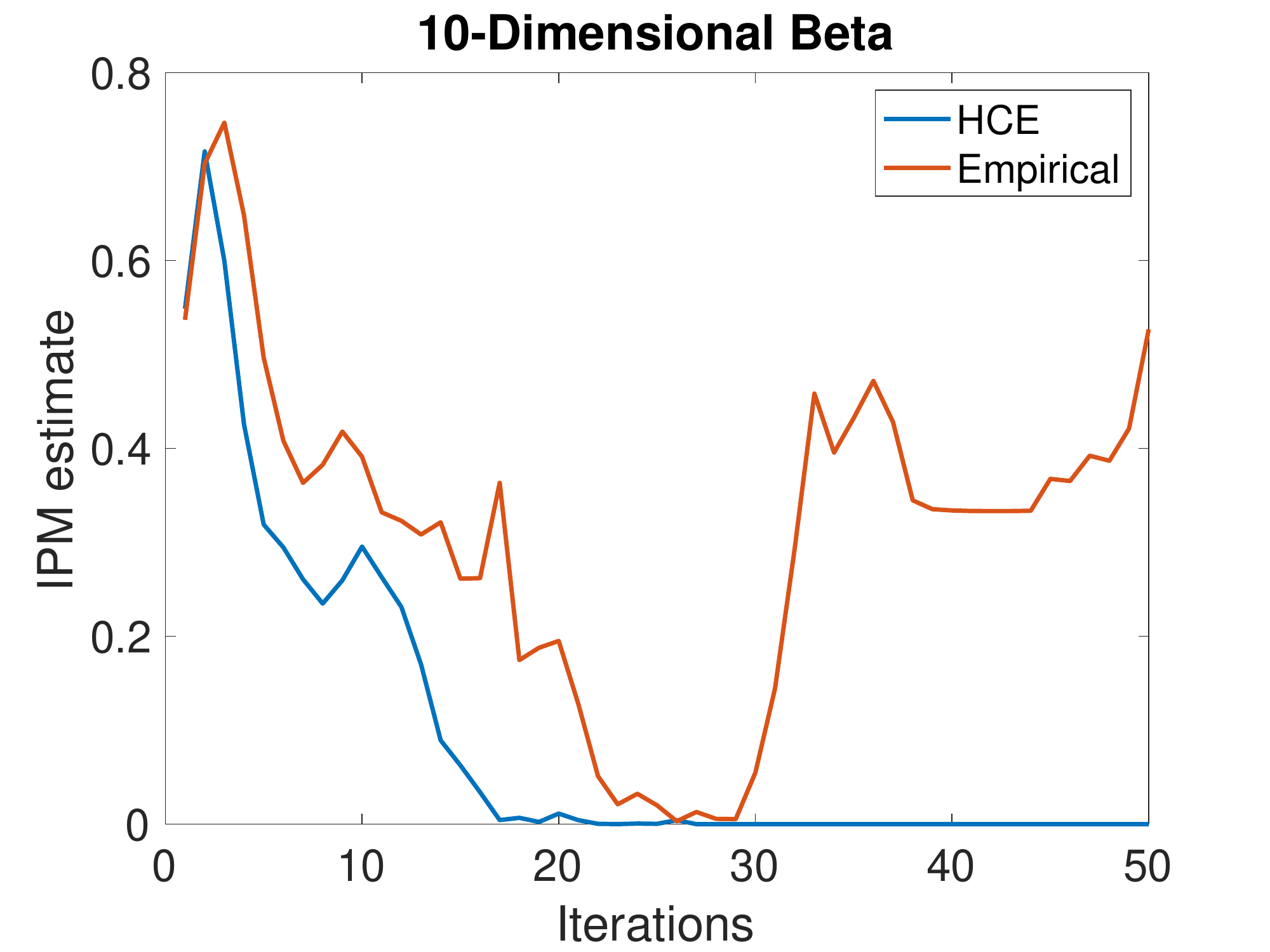} 
\includegraphics[width=0.235\textwidth]{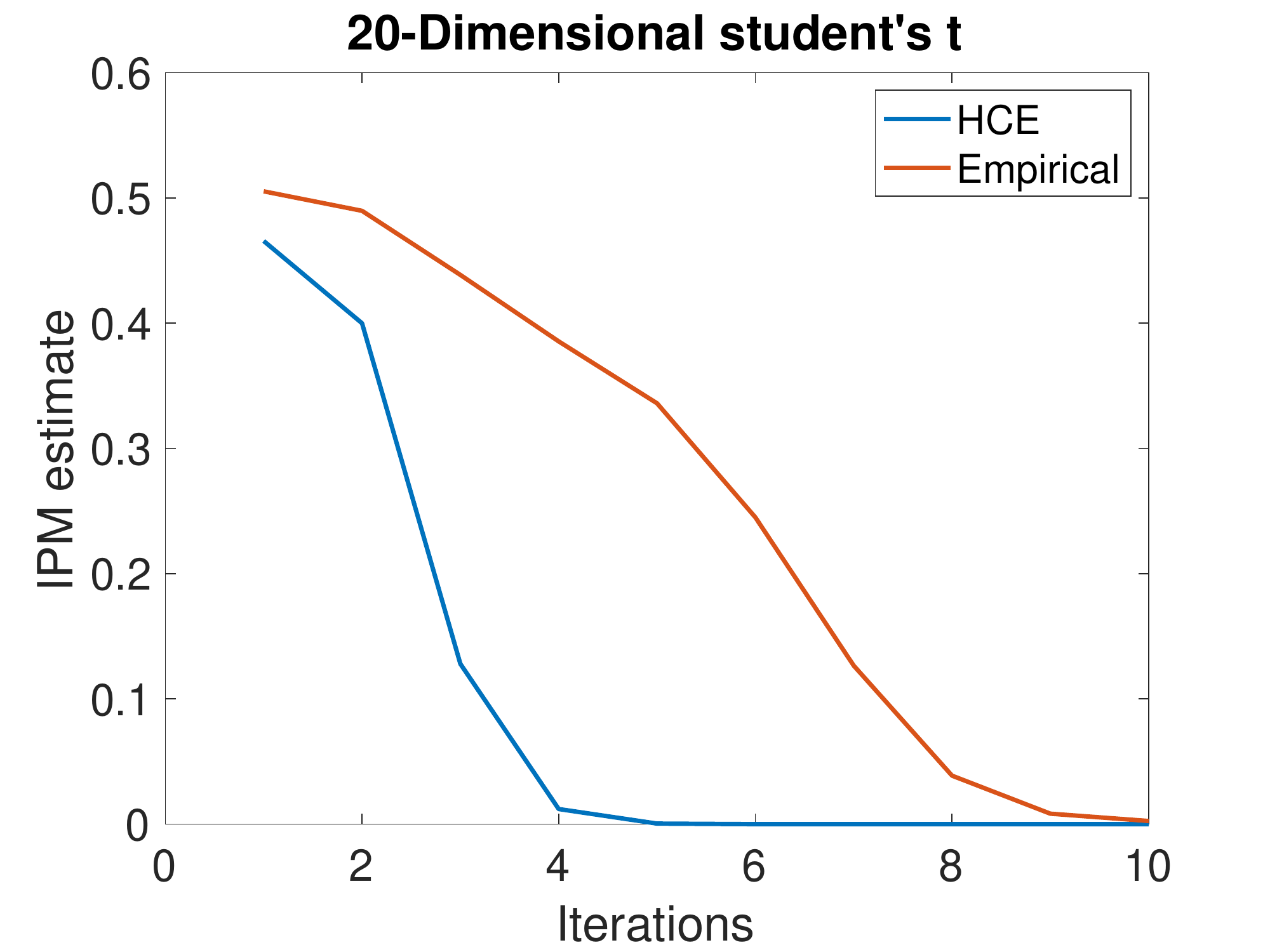}
\caption{Numerical results in Section \ref{subsec:syn}. \textbf{Panels 1-2}: RMSE of the estimators on 5-Dimensional (Panel 1) and 10-dimensional (Panel 2) Beta distribution. \textbf{Panels 3-4}: IPM of GANs applied to learn the 10-Dimensional Beta distribution (Panel 3) and the 20-dimensional student's $t$-distribution (Panel 4).}\label{fig:DensityCompare}
\vspace{-4mm}
\end{figure}

% \begin{figure}[t!]
% \centering
% \includegraphics[width=0.235\textwidth]{AISTATS2021PaperPack/IPM_10D.pdf} 
% \includegraphics[width=0.235\textwidth]{AISTATS2021PaperPack/IPM_20D.pdf}
% \caption{IPM of GANs applied to learn the 10-Dimensional Beta distribution (top) and the 20-dimensional student's $t$-distribution (bottom).}\label{fig:IPMCompare}
% \end{figure}

% Figure \ref{fig:DensityCompare} shows the performance of AHCE compared with KDE and BSDE. 

\textbf{Improved Performance for GAN}\quad 
% According to \cite{suzuki2018adaptivity}, ReLU neural networks \citep{Glorot10} can approximate any function in $\CalH^\alpha_{mix}$ with arbitrarily small distance. Let $\{X_i\}_{i}^{n}$ be the set of observations. 
We train two GAN estimators $\hat{g}'$ and $\hat{g}$, with the empirical estimator and the 
% expectation 
GAN adaptive to AHCE $\tilde{p}_n$, respectively (see \eqref{eq:num11} and \eqref{eq:num12}). We consider two distributions: 10-dimensional random vector with each element i.i.d. Beta$(2,5)$, and a 20-dimensional vector with each element i.i.d. $t$-distribution with degree of freedom two.
The loss function is proportional to the IPM \citep{Arjovsky17}, we can use it as an estimate of IPM. We run the training process until the GANs
associated with empirical estimator and AHCE converge. For each run, we record the IPM estimate value at the end of each iteration. We repeat the training process 100 times and plot the mean of the IPM estimates as in Panels 3 and 4 of Figure \ref{fig:DensityCompare}.

In both cases, the GAN associated with AHCE has a faster and stable convergence, which corroborates Theorem \ref{thm:GAN-improved}. The intuition is that the AHCE can be viewed as a smoothed empirical density which leads to a more stable gradient estimate for updating the parameters of GAN. This stable property of AHCE is more obvious in learning the 10-D Beta(2,5) density, which is asymmetric and less smooth.  
We also run experiments to compare the performance of the GANs associated with AHCE and empirical estimator on learning high-dimensional Gaussian densities. However, both of them turn out to have good performance in learning the Gaussian densities so we omit the comparison on Gaussian distributed data. Nevertheless, via the experiments on Beta, student's $t$- and Gaussian distributions, we can have a conclusion that the smoother the true underlying density and its estimator, the more stable the training process. 

\subsection{MNIST Data Set}\label{sub:mnist}
In this experiment, we conduct a numerical experiment to illustrate Theorem \ref{thm:smoothing}, which states that a GAN with a smoother discriminator tend to achieve a better performance. We compare the performances of two GANs: one with a ReLU empirical discriminator and the other one with a Sigmoid empirical discriminator. Both GANs have exactly the same structure and the only difference is the activation functions adopted in their empirical discriminators. We randomly select 20000 samples from the MNIST \citep{lecun-mnisthandwrittendigit-2010} data set and train both GANs. We then record down the outputs of both GANs at the $20000^{\rm th}$, $40000^{\rm th}$ and $60000^{\rm th}$ iteration, respectively. The results are shown in Figure \ref{fig:MNISTCompare}.

\begin{figure}[t!]
\centering
\includegraphics[width=0.47\textwidth]{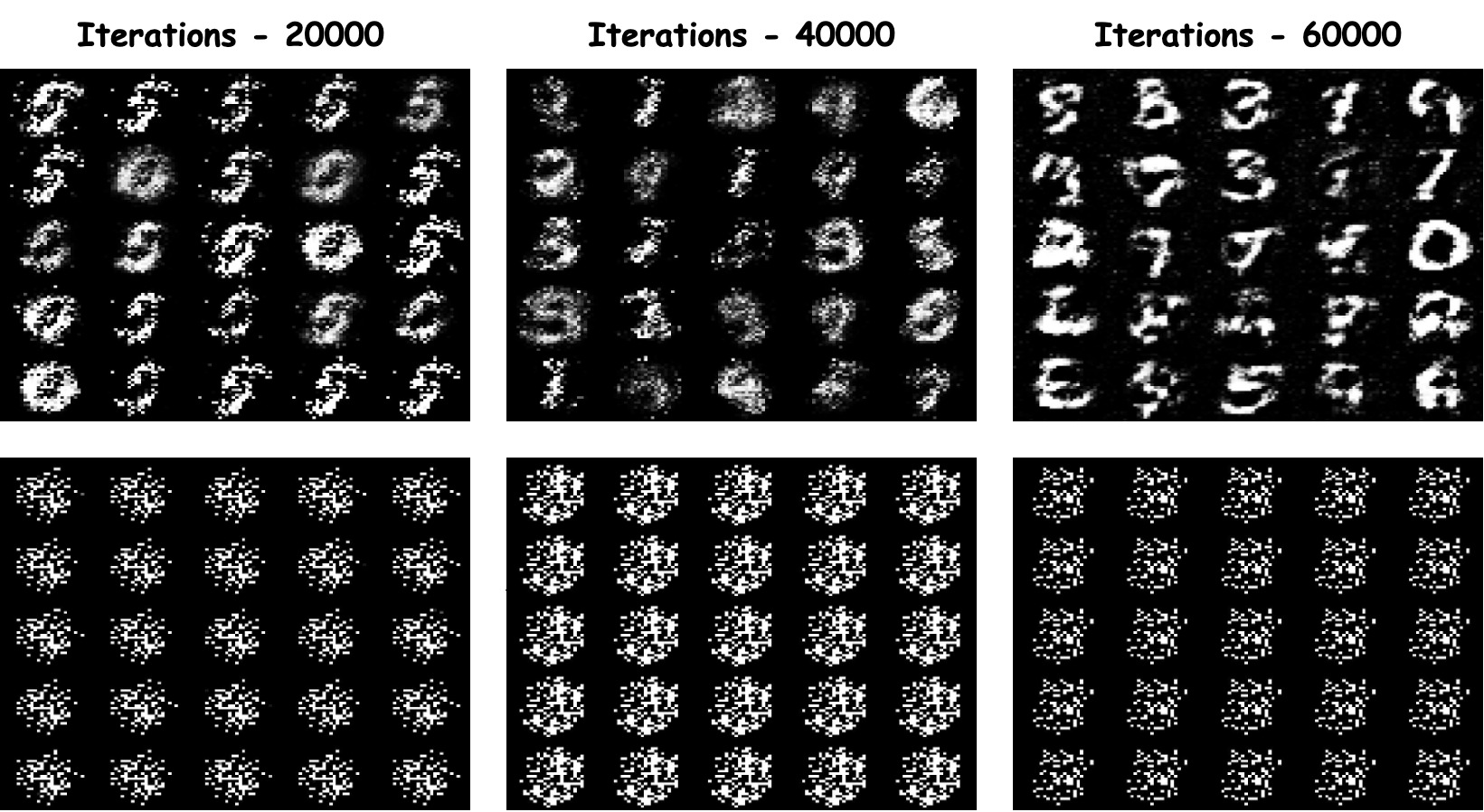}
\caption{ GAN with Sigmoid discriminator (top) and GAN  with ReLU discriminator (botton)}\label{fig:MNISTCompare}
\end{figure}

Obviously, the GAN with Sigmoid discriminator successfully captures  features of the handwritten digits while the GAN with ReLU discriminator fails to learn any features. This result coincides with Theorem \ref{thm:smoothing}. The Sigmoid net and the ReLU net have similar approximation capacity because they have the same structure. However, if we apply the smoothing operation \eqref{eq:smoothing} on the Sigmoid net, the smoothing and the original function yielded by the Sigmoid net should be very close, while the smoothing operation acting on the ReLU net leads to a smoothing function far away from its original output. This is because any function yielded by the Sigmoid net is smooth and a smoothing operation acting on a smooth function does not dramatically change the function. Intuitively, if the underlying target function is smooth enough, a discriminator adaptive to the sample size and the smoothness can learn more features.

\subsection{Goodness of Fit Test}
We conducted two numerical experiments on the goodness of fit test using Algorithm \ref{Alg1}. We consider the goodness of fit test with % In the first experiment, we consider the goodness of fit test with 
the underlying distribution 
\begin{align}\label{eq:num2}
    X_0\sim \text{Beta}(a, b),\ \ \  X\sim X_0+0.2U.
\end{align}
In the first example, $X_0$ has dimension two, $a=[2,2]$, $b=[2,5]$, and $U\sim {\rm Unif}([0,1]^2)$. 
% is a two-dimensional uniformly distributed random variable on $[0,1]^2$. 
In the second experiment, 
% we consider the goodness of fit test with the form \eqref{eq:num2}, but now 
$X_0$ has dimension six, $a=[2,2,2,2,2,2]$, $b=[2,5,2,5,2,5]$ and $U\sim \text{Unif}([0,1]^6)$.
The results are summarized in Tables \ref{table_test1} and \ref{table_test2}.

Tables \ref{table_test1} and \ref{table_test2} show the powers with different sample sizes $n$ and levels $l$. We observe that the increasing value of level $l$ leads to loss of power in the goodness of fit test problem. The intuition behind this phenomenon is that by choosing a higher level $l$, we may incorporate many redundant bases. Thereby, the level $l$ should be chosen carefully to insure the test power when the sample size is relatively small.
\begin{table}[h!]
\centering
\begin{tabular}{ |p{0.93cm}||p{0.93cm}|p{0.93cm}|p{0.93cm}|p{0.93cm}|p{0.93cm}| }
%  \hline
%  \multicolumn{6}{|c|}{Experiment 1} \\
 \hline
 level      &n=50 &n=100& n=150 & n=200 & n=500\\
 \hline
l=4   & 0.69  & 0.92& 1.00& 1.00& 1.00\\
\hline
l=10 & 0.53& 0.85& 0.94& 0.99& 1.00\\
 \hline
l=16 & 0.55& 0.83& 0.96& 0.98& 1.00\\
\hline
\end{tabular} 
\caption{Test powers under different sample sizes and levels for the first experiment.} \label{table_test1}
\end{table}
\begin{table}[h!]
\centering
\begin{tabular}{ |p{0.93cm}||p{0.93cm}|p{0.93cm}|p{0.93cm}|p{0.93cm}|p{0.93cm}| }
%  \hline
%  \multicolumn{6}{|c|}{Experiment 2} \\
 \hline
 level      &n=50 &n=100& n=150 & n=200 & n=500\\
 \hline
l=10   & 0.62  & 0.86& 0.90 & 1.00& 1.00\\
\hline
l=15 & 0.53& 0.72& 0.85& 0.91& 1.00\\
 \hline
l=20 & 0.51& 0.70& 0.82& 0.89& 1.00\\
\hline
\end{tabular}
\caption{Test powers under different sample sizes and levels for the second experiment.}
\label{table_test2}
\end{table}

\section{Conclusions}\label{sec:conclu}
In this work, we propose a class of non-parametric density estimators AHCE for estimating a probability density in the mixed smooth Sobolev space, where the dimension can be high. We prove that the AHCE is almost asymptotically optimal among all estimators in the mixed smooth Sobolev space. More importantly, the convergence rate \textit{weakly} depends on the dimension, where the dimension only appears in the $\log$ term. This implies that, the convergence rate is almost dimension-free, which partially break the ``curse of dimensionality''. We apply the AHCE to two high-dimensional problems: GANs and the goodness of fit test. Using our density estimator, we can have a faster rate of convergence to the true density function, compared to other estimators which do not make assumptions on the mixed smoothness property of the underlying density function. The numerical experiments show that the convergence of GANs associated with AHCE for density estimation is much more stable than the GANs associated with the original purely empirical density estimator, and our proposed testing statistics based on AHCE are accurate for goodness of fit test problems. 
\bibliography{ref}

%\appendixpagenumbering
\newpage
\numberwithin{equation}{section}
\appendix
\aistatstitle{Supplementary}

\addcontentsline{toc}{section}{Appendix} % Add the appendix text to the document TOC
\part{ } % Start the appendix part
\parttoc % Insert the appendix TOC

\section{Assumptions}

\begin{assumption}\label{A1}
\begin{enumerate}[label=(\roman*)]
    \item the true underlying density $p$ is lower bounded $\inf_{\Bx\in\Omega}p(\Bx)>0$ and in $\CalH_{mix}^\alpha(L_\alpha)$ with $\alpha>\frac{1}{2}$;
    \item the predetermined density $\mu$ is compactly supported on a convex domain and infinitely differentiable.
\end{enumerate}
\end{assumption}
\begin{assumption}\label{A2}
Suppose the empirical discriminator $\Phi_d$ (neural net or other approximator) satisfies
\begin{enumerate}[label=(\roman*)]
    \item $\sup_{f^*\in \CalW^\beta_\infty}\inf_{{f}\in\Phi_d}\|f-{f}^*\|_\infty\lesssim \varepsilon$
    \item $\forall f\in\Phi_d$, $\|\tilde{f}_n-f\|_2\lesssim \varepsilon $
    \item $\sup_{f\in\Phi_d}\|f\|_2<\infty$ and  the metric entropy of $\Phi_d$ under $L^\infty$ norm $\mathcal{E}_\infty(\varepsilon,\Phi_d)$ \citep{geer2000empirical} satisfies:
    \[ \inf_{\delta>0}\{K\int_{\delta/4}^1\sqrt{\mathcal{E}_\infty(K u/2,\Phi_d)}du+\sqrt{n}\delta K\}<\infty\]
\end{enumerate}
where $\varepsilon=n^{-\frac{\alpha+\beta/D}{2\alpha+1}}\big[\log n\big]^{D\frac{\alpha+\beta/D+1}{2\alpha+1}}+n^{-\frac{1}{2}}$, $\tilde{f}_n$ is the smoothing of $f$ defined in \eqref{eq:smoothing}, and $K=\sup_{f\in\Phi_d}\|f\|_\infty$.
\end{assumption}

\section{Details of Numerical Experiments}\label{apdix:expriment}
All the experiments are implemented in \textsc{Matlab} (version 2021a) on a laptop computer with macOS, 3.3 GHz Intel Core i5 CPU, and 8 GB of RAM (2133Mhz). 
\subsection{$L^2$ Convergence Rate}
Settings of two benchmark density estimators:
\begin{enumerate}
    \item \textbf{KDE}: kernel density estimator with Gaussian kernel and wavelength tuned in the order of $\CalO(n^{4/D})$ \citep{ming2017}, where $n$ is the number of observations and $D$ is the dimension;
    \item \textbf{BSDE}: B-spline distribution estimator with density treated as the derivative of the estimated distribution \citep{Lan10} with knot number chosen as $\CalO(n^{1/2})$, where $n$ is the number of observations.
\end{enumerate}

In the first experiment, we generated $n=200k, k=1,...,10$ observations of a 5-dimensional random vector follows a Beta-distribution $\text{Beta}(a,b)$ with $a=[2, 2 ,2 ,5 ,5]$ and $b=[5 ,5 ,2, 2 ,2]$. For a specific $n$, we ran three estimators on the same data set to estimate the true density function. We randomly sampled 1000 points $\{\Bx_i\}_{i=1}^{1000}$ in $[0,1]^5$ and used the following root mean squared error as an estimation for the $L^2$ error
$$RMSE=\left(\frac{1}{n}\sum_{i=1}^n[\hat{p}(\Bx_i)-p(\Bx_i)]^2\right)^{\frac{1}{2}}.$$

In the second experiment, we generated $n=200k, k=1,...,10$ observations of a 10-dimensional random vector follows $\text{Beta}(a,b)$ with $a=[2,\cdots,2]$ and $b=[5,\cdots,5]$.  For a specific $n$, we ran three estimators on the same data set to approximate the true density function and randomly sampled 10,000 points $\{\Bx_i\}_{i=1}^{10,000}$ in $[0,1]^{10}$ and used the RMSE as an estimation for the $L^2$ error.

In both experiments, the smoothing policy is set to be $b_{\Bs,n}=\CalO\bigl((1+2^{2|\Bs|}/n)^{-1}\bigr)$.

We first train a GAN estimator $\hat{g}'$ by solving the optimization problem
\begin{equation}\label{eq:num11}
\hat{g}'=\argmin_{q\in Q_g}\sup_{f\in F_d}\E_{X\sim q}[f(X)]-\frac{1}{n}\sum_{i=1}^nf(\BX_i)
\end{equation}
with generator class $Q_g$ and discriminator class $F_d$ encoded by ReLU neural networks. We then replace the empirical estimator in \eqref{eq:num11} by the expectation adaptive to AHCE $\tilde{p}_n$ and train the GAN estimator $\hat{g}$ by
\begin{equation}\label{eq:num12}
\hat{g}=\argmin_{q\in Q_g}\sup_{f\in F_d}\E_{X\sim q}[f(X)]-\E_{Y\sim \tilde{p}_n}[f(Y)].
\end{equation}

\subsection{Improved Performance for GAN}

In the first experiment, we let $\{X_i\}_{i=1}^{n}$ with $n=500$ be a set of 10-dimensional random vectors sampled from the Beta$(2,5)$ distribution. We sampled 5000 observations $\{Y_i\}_{i=1}^{5000}$ from the AHCE $\tilde{p}_n$ and use the mean of $\{f(Y_i)\}_{i=1}^{5000}$ as an estimate of $\E_{Y\sim \tilde{p}_n}[f(Y)]$. In the second experiment, we let $\{X_i\}_{i=1}^{n}$ with $n=1000$ be a set of 20-dimensional random vectors sampled from a $t$-distribution with degree of freedom two. We sampled 10,000 observations $\{Y_i\}_{i=1}^{10,000}$ from the AHCE $\tilde{p}_n$ and use the mean of $\{f(Y_i)\}_{i=1}^{10,000}$ to estimate $\E_{Y\sim \tilde{p}_n}[f(Y)]$. 

\subsection{MNIST Data Set}
We  randomly drew 20000 samples from the total 60000 training samples in the MNIST data set. Two GANs were trained based on these 20000 samples. One GAN is with smooth discriminator and the other one is with piecewise linear discriminator. Settings of two deep neural net empirical discriminators:\\ 

 \quad 1. \textbf{Sigmoid Net}: a neural net with 3 hidden layers where each layer is with 1024, 512 and 256 neurons, respectively. Dropout is applied on each hidden layer. All activations are set as  the Sigmoid function.
 
 \quad 2. \textbf{ReLU Net}: same structure as the Sigmoid net except that all activations are set as the ReLU function.\\[2ex]
The generator of both GANs is set as a neural net with 3 hidden layers where each layer is with 256, 512 and 1024 neurons, respectively. The activations for each hidden layer are set as the leaky ReLU function with scale equals 0.2.

\subsection{Goodness of Fit Test}

In both experiments, we set the observed sample size to $n=50, 100, 150, 200, 500$.

In both experiments, we fixed the significance level at 0.05, and ran $B=1000$ wild bootstrap to estimate the corresponding threshold. Each experiment was repeated for 100 times and we recorded the mean of the observed power for different choices of the level $l$. The wild bootstrap procedure is presented as Algorithm \ref{Alg1}.

\section{Proof of Theorem \ref{thm:densityIPM}}\label{sec:proof_thm1}
We split the proof into two cases -- $F_d=\CalH^\beta_{mix}$ and $F_d=\CalW^\beta_\infty$. We only prove the convergence rate of truncation smoothing policy with Fourier basis function $\phi_\Bs=\phi^F_\Bs$ because the proof of other smooth policies and wavelet basis can be completed in a same manner.

We first prove the following Lemmas, which will be used frequently for estimation in the later proofs.
\\[2ex]
\begin{lem}
\label{lem:sparsegrid_sum}
For any $x\in\Real$ and integer $l\geq 2$:
\begin{align*}
    \sum_{i=0}^{l-1}x^{ i}{i+D-1\choose D-1}&=\sum_{i=0}^{D-1}{D-1\choose i}(\frac{x}{1-x})^{D-1-i}\frac{1}{1-x} -\frac{x^l}{1-x}\sum_{j=0}^{D-1}{l+D-1\choose j}(\frac{x}{1-x})^{D-1-j}.
\end{align*}
\end{lem}
% \\[4ex]
\begin{proof}
The result is from direct calculations. Let $f^{(i)}$ denote the $i^{\rm th}$ time derivative of any function $f$ of $x$, then
\begin{align*}
    \sum_{i=0}^{l-1}x^{ i}{i+D-1\choose D-1}&=\frac{1}{(D-1)!}\sum_{i=0}^{l-1}(x^{i+D-1})^{(D-1)}\\
    &=\frac{1}{(D-1)!}(x^{D-1}\frac{1-x^l}{1-x})^{(D-1)}\\
    &=\frac{1}{(D-1)!}\sum_{i=0}^{D-1}{D-1\choose i}(x^{D-1}-x^{l+D-1})^{(i)}(\frac{1}{1-x})^{(D-1-i)}\\
    &=\sum_{i=0}^{D-1}{D-1\choose i}\frac{(D-1)!}{(D-1-i)!} x^{D-1-i}\frac{(D-1-i)!}{(D-1)!}(\frac{1}{1-x})^{D-1-i+1}\\
    &\quad - \sum_{j=0}^{D-1}{D-1\choose j}\frac{(l+D-1)!}{(l+D-1-j)!} x^{l+D-1-j}\frac{(D-1-j)!}{(D-1)!}(\frac{1}{1-x})^{D-1-j+1}\\
    &=\sum_{i=0}^{D-1}{D-1\choose i}(\frac{x}{1-x})^{D-1-i}\frac{1}{1-x}-\frac{x^l}{1-x}\sum_{j=0}^{D-1}{l+D-1\choose j}(\frac{x}{1-x})^{D-1-j}.
\end{align*}
\end{proof}
% \\[6ex]
\begin{lem}
\label{lem:bungartz-Lemma-3-7}
For any $s>0$, $l\in\NatInt$ and index $\Bk\in\NatInt^D$:
\begin{align*}
   \sum_{|{\Bk}|>l+D-1}2^{-s|\Bk|}\lesssim 2^{-sl}l^{D-1}.
\end{align*}
\\[8ex]
\end{lem}
\begin{proof}
From a direct calculation, we have
\begin{align*}
    \sum_{|\Bk|>l+D-1}2^{-s|\Bk|}={} &\sum_{i=l+D}^\infty 2^{-s i}\sum_{|\Bk|=i}1
    ={}  \sum_{i=l+D}^\infty 2^{-s i}{\binom{i-1}{D-1}}
    ={} 2^{-sl}\cdot 2^{-sD}\sum_{i=0}^{\infty}2^{-si}{\binom{l +i + D- 1}{ D-1}}.
\end{align*}
 For any $x\in\Real$, calculations similar to Lemma \ref{lem:sparsegrid_sum} show
\begin{align*}
\sum_{i=0}^{\infty}x^i{\binom{\tau +i + d- 1}{d-1}}&=\frac{x^{-l}}{(D-1)!}\bigg(\sum_{i=0}^\infty x^{l+i+D-1}\bigg)^{(D-1)}\\
&=\frac{x^{-l}}{(D-1)!}\bigg(x^{l+D-1}\frac{1}{1-x}\bigg)^{(D-1)}\\
&=\sum_{i=0}^{d-1}{\binom{l+D-1}{ i}}\biggl(\frac{x}{1-x}\biggr)^{D-1-i}\frac{1}{1-x}.
\end{align*}
Let $x=2^{-s}$ and make the substitution, we can get
\begin{align*}
    \sum_{|\Bk|>l+D-1}2^{-s|\Bk|}={} &
    2^{-sl}2^{-sD}\sum_{i=0}^{\infty}2^{-si}{\binom{l +i + D- 1}{ D-1}}\\
    {}\leq &  2^{-s\tau}\max_{j=0,\ldots,D-1}\Biggl\{\frac{2^{-sD}}{1-2^{-s}}\left(\frac{2^{-s}}{1-2^{-s}}\right)^{D-1-j}\Biggr\}\sum_{j=0}^{D-1}{\binom{l+D-1}{ j}}\\
    \lesssim & 2^{-sl}l^{D-1}
\end{align*}
where the last equality is from the fact that
$\sum_{i=0}^{D-1}{\binom{l+D-1}{i}}\lesssim Dl^{D-1}$.
\end{proof}

\subsection{Mixed Smooth Sobolev Discriminator Class:  $F_d=\CalH^\beta_{mix}$}\label{sec:UpperBound_Mix}
According to definition \ref{dfn:mixSobolevNorm}, for any $f\in\CalH^\beta_{mix}(L_{\beta})$, we can write its  expansion as
$$f=\sum_{\bold{k}\in\NatInt^D}\sum_{\bold{s}\in\rho(\bold{k})}\hat{f}_{\bold{s}}\phi_{\Bs},$$
where the convergence is in $L^2$. This representation
can also be applied to any density function $p\in\CalH^\alpha_{mix}(L_{\alpha})$.

Based on the definition of $\tilde{p}_n$, the IPM loss has the following decomposition
\begin{align*}
    \E d_{F_d}(p,\tilde{p}_n)&=\E \sup_{f\in\CalH^\beta_{mix}(L_{\beta})}\langle f, p-\tilde{p}_n\rangle_{L^2}\\
    &=\E \sup_{f\in\CalH^\beta_{mix}(L_{\beta})} \sum_{\bold{k}\in\NatInt^D}\sum_{\bold{s}\in\rho(\bold{k})}\hat{f}_{\bold{s}}\big[\hat{p}_{\bold{s}}-\tilde{p}_{\bold{s}}\big]\\
    &= \E \sup_{f\in\CalH^\beta_{mix}(L_{\beta})} \big\{\sum_{|\bold{k}|\leq l}\sum_{\bold{s}\in\rho(\bold{k})} \hat{f}_{\bold{s}}\big[\hat{p}_{\bold{s}}-\tilde{p}_{\bold{s}}\big]+\sum_{|\bold{k}|> l}\sum_{\bold{s}\in\rho(\bold{k})} \hat{f}_{\bold{s}}\hat{p}_{\bold{s}}\big\}\\
    &\leq \underbrace{\E \sup_{f\in\CalH^\beta_{mix}(L_{\beta})} \sum_{|\bold{k}|\leq l}\sum_{\bold{s}\in\rho(\bold{k})} \hat{f}_{\bold{s}}\big[\hat{p}_{\bold{s}}-\tilde{p}_{\bold{s}}\big]}_{\text{Variance}}+\underbrace{\sup_{f\in\CalH^\beta_{mix}(L_{\beta})}\sum_{|\bold{k}|> l}\sum_{\bold{s}\in\rho(\bold{k})} \hat{f}_{\bold{s}}\hat{p}_{\bold{s}}}_{\text{Bias}}
\end{align*}
where the first term in the last line is the variance and the second term is the bias.

For the variance, we have:
\begin{align*}
    &\quad \E \sup_{f\in\CalH^\beta_{mix}(L_{\beta})} \sum_{|\bold{k}|\leq l}\sum_{\bold{s}\in\rho(\bold{k})} \hat{f}_{\bold{s}}\big[\hat{p}_{\bold{s}}-\tilde{p}_{\bold{s}}\big]\\
    &\leq \E \sup_{f\in\CalH^\beta_{mix}(L_{\beta})}\bigg[\sum_{|\bold{k}|\leq l}2^{2|\Bk|\beta}\sum_{\bold{s}\in\rho(\bold{k})}\hat{f}_{\bold{s}}^2\bigg]^{\frac{1}{2}}\bigg[\sum_{|\bold{k}|\leq l}2^{-2|\Bk|\beta}\sum_{\bold{s}\in\rho(\bold{k})}[\hat{p}_{\bold{s}}-\tilde{p}_{\bold{s}}\big]^2\bigg]^{\frac{1}{2}}\\
    &\leq \sup_{f\in\CalH^\beta_{mix}(L_{\beta})}\|f\|_{\CalH^\beta_{mix}}\bigg[\sum_{|\bold{k}|\leq l}2^{-2|\Bk|\beta}\sum_{\bold{s}\in\rho(\bold{k})}\E[(\hat{p}_{\bold{s}}-\tilde{p}_{\bold{s}})^2\big]\bigg]^{\frac{1}{2}}\\
    &\leq L_{\beta}C\sqrt{\frac{1}{n}\sum_{i=D}^l2^{(1-2\beta)i}{i-1\choose D-1}}
    %&\asymp L_{\beta}C\sqrt{\frac{2^{(1-2\beta)l}l^{D-1}}{n}}
\end{align*}
where $C=\big(\E|\psi_{\bold{s}}(X_1)|^2\big)^{\frac{1}{2}}\asymp \frac{1}{\pi} $. The second line of the above equations is from H\"older's inequality, the third line is from Jensen's inequality.

When $\beta=1/2$, we have
\begin{equation*}
    \sum_{i=D}^l2^{(1-2\beta)i}{i-1\choose D-1}= \sum_{i=D-1}^{l-1}{i \choose D-1}={l\choose D}\asymp \frac{l^D}{D!}
\end{equation*}
For $\beta\neq 1/2$, we apply Lemma \ref{lem:sparsegrid_sum} with the identity $x=2^{1-2\beta}$ to have the following estimate
\begin{align*}
 \sum_{i=D}^l2^{(1-2\beta)i}{i-1\choose D-1}&=\sum_{i=0}^{l-D+1-1}2^{(1-2\beta)(i+D)}{i+D-1\choose D-1}\\
 &=2^{(1-2\beta)D}\bigg[\underbrace{\sum_{i=0}^{D-1}{D-1\choose i}(\frac{x}{1-x})^{D-1-i}\frac{1}{1-x}}_{A} -\underbrace{\frac{x^{l-D+1}}{1-x}\sum_{j=0}^{D-1}{l\choose j}(\frac{x}{1-x})^{D-1-j}}_B\bigg].
\end{align*}
Notice that term A is independent of $l$. For term B, if $\beta<1$, then the summand for $j=D-1$ is the largest one; if $\beta>1$, then $1/(1-x)$ is bounded by 1. In both cases, we can use Stirling's approximation to get
\begin{align*}
    -\frac{x^{l-D+1}}{1-x}\sum_{j=0}^{D-1}{l\choose j}(\frac{x}{1-x})^{D-1-j}\asymp2^{(1-2\beta)(l-D+1)}\frac{l^{D-1}}{(D-1)!}\lesssim2^{(1-2\beta)l}l^D.
\end{align*}
Therefore, we have the following estimate for the variance:
\begin{align*}
    \E \sup_{f\in\CalH^\beta_{mix}(L_{\beta})} \sum_{|\bold{k}|\leq l}\sum_{\bold{s}\in\rho(\bold{k})} \hat{f}_{\bold{s}}\big[\hat{p}_{\bold{s}}-\tilde{p}_{\bold{s}}\big]
    &\leq L_{\beta}C\sqrt{\frac{1}{n}\sum_{i=D}^l2^{(1-2\beta)i}{i-1\choose D-1}}\\
    &\lesssim n^{-\frac{1}{2}}2^{\frac{(1-2\beta)l}{2}}l^{\frac{D}{2}}.
\end{align*}

For the bias, we have:
\begin{align*}
    &\quad \sup_{f\in\CalH^\beta_{mix}(L_{\beta})}\sum_{|\bold{k}|> l}\sum_{\bold{s}\in\rho(\bold{k})} \hat{f}_{\bold{s}}\hat{p}_{\bold{s}}\\
    &\leq \sup_{f\in\CalH^\beta_{mix}(L_{\beta})}\big\{\sum_{|\bold{k}|> l} \|\sum_{\bold{s}\in\rho(\bold{k})} \hat{f}_{\bold{s}}\|_{l^2}^2\big\}^{\frac{1}{2}}\big\{\sum_{|\bold{k}|> l} \|\sum_{\bold{s}\in\rho(\bold{k})} \hat{p}_{\bold{s}}\|_{l^2}^2\big\}^{\frac{1}{2}}\\
    &\lesssim L_{\alpha}L_{\beta}\big\{\sum_{i\geq l+1}\sum_{|\bold{k}|=i} 2^{-2\beta i}\big\}^{\frac{1}{2}}\big\{\sum_{i\geq l+1}\sum_{|\bold{k}|=i} 2^{-2\alpha i}\big\}^{\frac{1}{2}}\\
    &\lesssim L_{\alpha}L_{\beta}  2^{-(\alpha+\beta)l}l^{D-1}
\end{align*}
where the second line is from H\"older inequality, the third line is from the following identities
\begin{align*}
    &\|p\|_{\CalH_{mix}^\alpha}=\big\{\sum_{\Bk\in\NatInt^D}2^{2|\Bk|\alpha} \|\sum_{\bold{s}\in\rho(\bold{k})} \hat{p}_{\bold{s}}\|_{l^2}^2\big\}^{\frac{1}{2}}\\
    &\|f\|_{\CalH_{mix}^\beta}=\big\{\sum_{\bold{k}\in\NatInt^D}2^{2|\Bk|\beta} \|\sum_{\bold{s}\in\rho(\bold{k})} \hat{f}_{\bold{s}}\|_{l^2}^2\big\}^{\frac{1}{2}}
\end{align*}

and the last line is from Lemma \ref{lem:bungartz-Lemma-3-7}. So we can minimize over all feasible $l$ to have:
\begin{align*}
    \E d_{F_d}(p,\tilde{p}_n)&\lesssim \inf_{l\in\NatInt}\sqrt{\frac{2^{(1-2\beta)l}l^{D}}{n}}+  2^{-(\alpha+\beta)l}l^{D}
\end{align*}
which give $2^l \asymp n^{\frac{1}{2\alpha+1}}\big[\log n\big]^{D\frac{(\alpha+\beta+1)}{2\alpha+1}}$. We then can plug this identity in the above equation to get the result.
\subsection{H\"older Discriminator Class: $F_d=\CalW^\beta_\infty$}\label{sec:upperBound_Holder}
In order to prove the statement, we first need to define a function norm equivalent to the H\"older norm so that we can apply the method in \ref{sec:UpperBound_Mix}.  The following theorem can help
\begin{theorem}[The Littlewood-Paley theorem.]\label{thm:Littlewood_Paley}
Let $1<p<\infty$. There exists positive number $C_1$ and $C_2$, which only depend on $p$ and $D$, such that for each function  $f\in L^p$,
\[C_1\|f\|_p\leq \bigg\|\bigg(\sum_{\Bk\in\NatInt^d}\big(\sum_{\Bs\in\rho(\Bk)}\hat{f}_\Bs\phi_\Bs\big)^2\bigg)^{\frac{1}{2}}\bigg\|_p\leq C_2\|f\|_p.\]
\end{theorem}
 H\"older spaces can be generalized to  Sobolev spaces defined as follows
\begin{defn}[Sobolev Space]
\label{dfn:SobolevSpace}
Let $\alpha\in\NatInt$ and $1\leq p\leq \infty$. For any function $f$, we define the following Sobolev $\CalW^\alpha_p$-norm:
\[\|f\|_{\CalW^\alpha_p}\coloneqq\sum_{|\Balpha|\leq \alpha}\|D^{\Balpha} f\|_p.\]
Then the Sobolev space $\CalW^\alpha_p$ is the set of all functions with bounded $\CalW^\alpha_p$-norm and the Sobolev ellipsoid $\CalW^\alpha_p(L)$ is the set of all functions with $\CalW^\alpha_p$-norm less than $L$.
\end{defn}
According to section 5.8.2 (b) in \cite{evans10} and the definition of H\"older space, we can easily check that H\"older norm \eqref{eq:Holder_norm} conincides with definition \ref{dfn:SobolevSpace}:
\[\max_{ |\Balpha|<\alpha}\|D^{\Balpha}f\|_\infty
    + \max_{ |\Balpha|=\alpha-1}\max_{\Bx\neq \By\in\Omega}\frac{|D^{\Balpha}f(\Bx)-D^{\Balpha}f(\By)|}{\|\Bx-\By\| }\asymp \sum_{|\Balpha|\leq\alpha}\|D^{\Balpha}f\|_\infty\]
    for any integer $\alpha$.
    
Now our goal is to define an equivalent norm of $\CalW^\alpha_p$  in the form similar to the mixed smooth Sobolev norm in definition \ref{dfn:mixSobolevNorm}:
\begin{lem}
\label{lem:Sobolev_equivalent_norm}
Let $1<p<\infty$. For any $\alpha\in\NatInt$,
\[\|f\|_{\CalW^\alpha_p}\asymp\bigg\|\bigg(\sum_{\Bk\in\NatInt^d}\big(\sum_{j=1}^D2^{\alpha k_j}\big)^2\big(\sum_{\Bs\in\rho(\Bk)}\hat{f}_\Bs\phi_\Bs\big)^2\bigg)^{\frac{1}{2}}\bigg\|_p.\]
\end{lem}
\begin{proof}
For notation simplicity, we let $F_{\Bk}[f](\cdot)\coloneqq \sum_{\Bs\in\rho(\Bk)}\hat{f}_\Bs\phi_\Bs(\cdot)$ and without loss of generality, we assume $\{\phi_\Bs\}$ are the Fourier features. Proof for other kind of wavelet system is similar. We first show 
\[\|f\|_{\CalW^\alpha_p}\lesssim\bigg\|\bigg(\sum_{\Bk\in\NatInt^d}\big(\sum_{j=1}^D2^{\alpha k_j}\big)^2\big(\sum_{\Bs\in\rho(\Bk)}\hat{f}_\Bs\phi_\Bs\big)^2\bigg)^{\frac{1}{2}}\bigg\|_p.\]
From the definition of $\|f\|_{\CalW^{\alpha}_p}$, theorem \ref{thm:Littlewood_Paley} and the fact that $D^{\Balpha}F_{\Bk}[f]=2^{\Balpha\cdot\Bk}F_{\Bk}[f]$, we can get:
\begin{align*}
    \|f\|_{\CalW^\alpha_p}&=\sum_{|\Balpha|\leq \alpha}\|D^{\Balpha} f\|_p\\
    &\asymp \sum_{|\Balpha|\leq \alpha}\bigg\|\bigg(\sum_{\Bk\in\NatInt^d}2^{2\Balpha\cdot\Bk}F^2_\Bk[f]\bigg)^{\frac{1}{2}}\bigg\|_p\\
    &\lesssim \bigg\|\bigg(\sum_{\Bk\in\NatInt^d}(\sum_{j=1}^D2^{\alpha k_j})^2F^2_\Bk[f]\bigg)^{\frac{1}{2}}\bigg\|_p.
\end{align*}
where the last line is because $\Balpha\cdot\Bk\leq |\Balpha|\max_{j=1,\ldots,D}k_j\leq \alpha \max_{j=1,\ldots,D}k_j $ and, hence,
\[2^{\Balpha\cdot\Bk}\leq \sum_{j=1}^D2^{\alpha k_j}.\]
Now we show
\[\|f\|_{\CalW^\alpha_p}\gtrsim\bigg\|\bigg(\sum_{\Bk\in\NatInt^d}\big(\sum_{j=1}^D2^{\alpha k_j}\big)^2\big(\sum_{\Bs\in\rho(\Bk)}\hat{f}_\Bs\phi_\Bs\big)^2\bigg)^{\frac{1}{2}}\bigg\|_p.\]
We need another equivalent norm of $\|\cdot\|_{\CalW^\alpha_p}$:
\[\|f\|_{\CalW^\alpha_p}^p\asymp \sum_{|\Balpha|\leq \alpha}\|D^{\Balpha} f\|_p^p.\]
Again, from theorem \ref{thm:Littlewood_Paley}, we have
\begin{align*}
    \sum_{|\Balpha|\leq \alpha}\|D^{\Balpha} f\|_p^p&\geq \sum_{j=1}^D\|\frac{\partial^\alpha}{\partial x_j^\alpha}f\|_p^p\\
    &\asymp  \sum_{j=1}^D\bigg\|\bigg(\sum_{\Bk\in\NatInt^d}2^{2\alpha k_j}F^2_\Bk[f]\bigg)^{\frac{1}{2}}\bigg\|_p^p\\
    &=\int_\Omega \sum_{j=1}^D\bigg(\sum_{\Bk\in\NatInt^d}2^{2\alpha k_j}F^2_\Bk[f](\Bx)\bigg)^{\frac{p}{2}}d \Bx\\
    &\gtrsim \int_\Omega \bigg(\sum_{\Bk\in\NatInt^d}\sum_{j=1}^D2^{2\alpha k_j}F^2_\Bk[f](\Bx)\bigg)^{\frac{p}{2}}d \Bx\\
    &\gtrsim \bigg\|\bigg(\sum_{\Bk\in\NatInt^d}(\sum_{j=1}^D2^{\alpha k_j})^2F^2_\Bk[f]\bigg)^{\frac{1}{2}}\bigg\|_p^p.
\end{align*}
\end{proof}
Now we can continue the proof of theorem \ref{thm:densityIPM} for the case $F_d=\CalW_\infty^\beta$. Because any function $f\in\CalW^\beta_\infty$, $f\in\CalW^\beta_2$, we have
\[d_{\CalW^\beta_\infty(L_\beta)}(p,q)\leq d_{\CalW^\beta_2(L_\beta)}(p,q)\]
for any density function $p,q$. Therefore, calculations similar to section \ref{sec:UpperBound_Mix} show
\begin{align*}
    \E[d_{\CalW^\beta_\infty(L_\beta)}(p,\tilde{p}_n)]&\leq\E[d_{\CalW^\beta_2(L_\beta)}(p,\tilde{p}_n)]\\
    &\leq \underbrace{\E \sup_{f\in\CalW^\beta_2(L_\beta)} \sum_{|\bold{k}|\leq l}\sum_{\bold{s}\in\rho(\bold{k})} \hat{f}_{\bold{s}}\big[\hat{p}_{\bold{s}}-\tilde{p}_{\bold{s}}\big]}_{\text{Variance}}+\underbrace{\sup_{f\in\CalW^\beta_2(L_\beta)}\sum_{|\bold{k}|> l}\sum_{\bold{s}\in\rho(\bold{k})} \hat{f}_{\bold{s}}\hat{p}_{\bold{s}}}_{\text{Bias}}.
\end{align*}
For the variance, we can use Lemma \ref{lem:Sobolev_equivalent_norm} to get
\begin{align*}
    &\quad \E \sup_{f\in\CalW^\beta_2(L_\beta)} \sum_{|\bold{k}|\leq l}\sum_{\bold{s}\in\rho(\bold{k})} \hat{f}_{\bold{s}}\big[\hat{p}_{\bold{s}}-\tilde{p}_{\bold{s}}\big]\\
    &\leq \E \sup_{f\in\CalW^\beta_2(L_\beta)}\bigg[\sum_{|\bold{k}|\leq l}(\sum_{j=1}^D2^{\beta k_j})^2\sum_{\bold{s}\in\rho(\bold{k})}\hat{f}_{\bold{s}}^2\bigg]^{\frac{1}{2}}\bigg[\sum_{|\bold{k}|\leq l}(\sum_{j=1}^D2^{\beta k_j})^{-2}\sum_{\bold{s}\in\rho(\bold{k})}[\hat{p}_{\bold{s}}-\tilde{p}_{\bold{s}}\big]^2\bigg]^{\frac{1}{2}}\\
    &\lesssim \sup_{f\in\CalW^\beta_2(L_\beta)}\|f\|_{\CalW^\beta_2(L_\beta)}\bigg[\sum_{|\bold{k}|\leq l}(\sum_{j=1}^D2^{\beta k_j})^{-2}\sum_{\bold{s}\in\rho(\bold{k})}\E[(\hat{p}_{\bold{s}}-\tilde{p}_{\bold{s}})^2\big]\bigg]^{\frac{1}{2}}\\
    &\leq \sup_{f\in\CalW^\beta_2(L_\beta)}\|f\|_{\CalW^\beta_2(L_\beta)}\bigg[\sum_{|\bold{k}|\leq l}D^{-2}2^{-2|\Bk|\frac{\beta}{D}}\sum_{\bold{s}\in\rho(\bold{k})}\E[(\hat{p}_{\bold{s}}-\tilde{p}_{\bold{s}})^2\big]\bigg]^{\frac{1}{2}}
\end{align*}
where the last line is from the inequality of arithmetic and geometric means:
\[\frac{1}{D}\sum_{j=1}^D2^{\beta k_j}\geq \sqrt[D]{2^{k_1\beta}2^{k_2\beta}\cdots 2^{k_D\beta}}=2
^{|\Bk|\frac{\beta}{D}}.\]

Similarly, for the bias, we have
\begin{align*}
    &\quad \sup_{f\in\CalW^\beta_2(L_\beta)}\sum_{|\bold{k}|> l}\sum_{\bold{s}\in\rho(\bold{k})} \hat{f}_{\bold{s}}\hat{p}_{\bold{s}}\\
    &\leq \sup_{f\in\CalW^\beta_2(L_\beta)}\big\{\sum_{|\bold{k}|> l} \|\sum_{\bold{s}\in\rho(\bold{k})} \hat{f}_{\bold{s}}\|_{l^2}^2\big\}^{\frac{1}{2}}\big\{\sum_{|\bold{k}|> l} \|\sum_{\bold{s}\in\rho(\bold{k})} \hat{p}_{\bold{s}}\|_{l^2}^2\big\}^{\frac{1}{2}}\\
    &\lesssim L_{\alpha}L_{\beta}\big\{\sum_{i\geq l+1}\sum_{|\bold{k}|=i} (\sum_{j=1}^D2^{\beta k_j})^{-2}\big\}^{\frac{1}{2}}\big\{\sum_{i\geq l+1}\sum_{|\bold{k}|=i} 2^{-2\alpha i}\big\}^{\frac{1}{2}}\\
    &\leq L_{\alpha}L_{\beta}\big\{\sum_{i\geq l+1}\sum_{|\bold{k}|=i} D^{-2}2^{-2\frac{\beta}{D}i}\big\}^{\frac{1}{2}}\big\{\sum_{i\geq l+1}\sum_{|\bold{k}|=i} 2^{-2\alpha i}\big\}^{\frac{1}{2}}.
\end{align*}
Notice that the variance and bias are in forms exactly the same as those in the analysis for the case $F_d=\CalH_{mix}^\beta$ except that we need to replace the $\beta$ in section \ref{sec:UpperBound_Mix} by $\beta/D$. However, the conclusion in section \ref{sec:UpperBound_Mix} holds for any $\beta>0$ and, therefore, it also holds for the case $\beta\to\beta/D>0$. We can finish the proof.

\section{Proof of Theorem \ref{thm:minmax_lower}}
Similar to section \ref{sec:proof_thm1}, we split the proof for lower bounds into two cases -- $F_d=\CalH_{mix}^\beta$ and $F_d=\CalW^\beta_\infty$. Firstly, we need to introduce two useful lemmas, which play important in this proof.

\begin{lem}[Fano's Lemma]
\label{lem:Fano}
Let $D_{KL}(P,Q)$ denotes the KL-divergence between probability measue $P$ and $Q$. Assume $H\geq 2$ and suppose the set $\Theta$ contains $\{\theta_i\}^H_{i=1}$ such that:
\begin{enumerate}
    \item There exists metric $d(\cdot,\cdot)$  on $\Theta$ with $d(\theta_i,\theta_j)\geq 2s>0$   for any $i,j \in[H]$ with $i\neq j$.
    \item Let $P_i$ be probability measures parametrized by $\theta_i$: $P_i=P(\cdot|\theta_i)$. $\frac{1}{H}\sum_{i=1}^HD_{KL}(P_i,P_0)\leq a\log H$ for some $0<a<1/8$ and some specific density $P_0$.
\end{enumerate}
Then for any estimator $\hat{\theta}$ for $\theta\in\{\theta_i\}_{i=1}^H$, we have:
$$\sup_{\theta\in\Theta}\mathbb{P}(d(\theta,\hat{\theta})\geq s)\geq \frac{\sqrt{H}}{1+\sqrt{H}}\bigg(1-2a-\sqrt{\frac{2a}{\log H}}\bigg)>0.$$
\end{lem}

\begin{lem}[Varshamov-Gilbert Bound]
\label{lem:VGbound}
 Among the subsets of $\{0,1\}^h$, there exists a set $\{w_i\}_{i=0}^H$ with $w_0=(0,\cdots,0)$ such that:
\begin{align*}
    &d(w_i,w_j)\geq \frac{h}{8}\quad \forall i,j\in [H],i\neq j\\
    & \log H \geq h\frac{log 2}{8}
\end{align*}
where $d(\cdot,\cdot)$ is some metric defined on $\{0,1\}^h$.
\end{lem}
 These two lemmas are commonly-seen tools in proving the asymptotic error lower bounds for estimation problems \citep{Tsybakov08,Nemirovski00}.
 
Our proof is inspired by the method in \cite{Tsybakov08}. Essentially, we  first construct a subspace $\Omega_\alpha$ on $\CalH^\alpha_{mix}(1)$ and a subspace  $\Lambda_\beta$ on $F_d$ such that it is hard to distinguish the estimators for functions in $\Omega_\alpha$ conditioned on $n$ samples, and thus the loss induced by the best discriminator in $\Lambda_\beta$ gives a lower 
bound on the estimation rate.
\subsection{Mixed Smooth Sobolev Discriminator Class:  $F_d=\CalH^\beta_{mix}$}\label{sec:LowerBound_Mix}

Without loss of generality, we assume $L_\alpha=L_\beta=1$. Let
$$\Omega_\alpha=\bigg\{g_\omega: \frac{1}{1+c_\omega}\big(1+c_\alpha\sum_{|\bold{k}|\leq l}\sum_{\bold{s}\in\rho(\bold{k})}\omega_\Bs\psi_{\Bs}\big),\omega_\Bs\in\{0,1\}\bigg\}$$
where
\begin{align*}
    &c_\alpha=\bigg(\sum_{i=D}^{l}2^{(2\alpha+1) i}{i-1\choose D-1}\bigg)^{-\frac{1}{2}}\\%\lesssim\frac{1}{2^{(2\alpha+1)l/2}l^{(D-1)/2}}\\
    &c_\omega=\int_{[0,1]^D}c_\alpha\sum_{|\bold{k}|\leq l}\sum_{\bold{s}\in\rho(\bold{k})}\omega_\Bs\psi_{\Bs}(\Bx)d\Bx\leq c_\alpha \sum_{|\Bk|\leq l}1\ll 1
\end{align*}
so that for any $g_\omega\in\Omega_\alpha$, we can have the upper bound:
\begin{align*}
    \|g_\omega\|^2_{\CalH^\alpha_{mix}}&=c_\alpha^2\sum_{|\Bk|\leq l}2^{2\alpha|\Bk|}\|\sum_{\bold{s}\in\rho(\bold{k})}\omega_\Bs\psi_{\Bs}\|_2^2\\
    &\lesssim c_\alpha^2\sum_{|\Bk|\leq l}2^{(2\alpha+1)|\Bk|}\\
    &=c_\alpha^2 \sum_{i=D}^{l}2^{(2\alpha+1) i}{i-1\choose D-1}\leq 1.
\end{align*}
As a result, $\Omega_\alpha\subset\CalH^\alpha_{mix}(1)$. According to Lemma \ref{lem:sparsegrid_sum}, we let $x=2^{2\alpha+1}$ and get
\begin{align*}
    c_\alpha^{-2}&= \sum_{i=D}^{l}2^{(2\alpha+1) i}{i-1\choose D-1}\\
    &=\sum_{i=0}^{l-D+1-1}2^{(2\alpha+1) (i+D)}{i+D-1\choose D-1}\\
    &=2^{(2\alpha+1)D}\bigg\{\sum_{i=0}^{D-1}{D-1\choose i}(\frac{x}{1-x})^{D-1-i}\frac{1}{1-x} -\frac{x^{l-D+1}}{1-x}\sum_{j=0}^{D-1}{l \choose j}(\frac{x}{1-x})^{D-1-j}\bigg\}\\
    &\asymp 2^{(2\alpha+1)l}l^{D-1}
\end{align*}
where the last inequality is from calculations similar to those for  estimating the variance in section \ref{sec:UpperBound_Mix}. 
and 

In a similar way, we can construct a subset $\Lambda_\beta\subset \CalH^\beta_{mix}(1)$:
$$\Lambda_\beta=\bigg\{f_\nu(\cdot): c_\beta\sum_{|\bold{k}|\leq l}\sum_{\bold{s}\in\rho(\bold{k})}\nu_\Bs\psi_{\Bs}(\cdot),\nu_\Bs\in\{-1,1\}\bigg\}$$
where
$$c_\beta=\frac{1}{2^{(2\beta+1)l/2}l^{(D-1)/2}}.$$

Then for any pair of functions $g_\omega,g_{\omega'}\in\Omega_\alpha$, their mixed smooth Sobolev IPM distance is bounded below as follows:
\begin{align*}
    d_{F_d}(g_\omega,g_{\omega'})&=\sup_{f\in\CalH^\beta_{mix}(1)}\bigg|\langle f,g_\omega-g_{\omega'}\rangle_{2}\bigg|\\
    &\gtrsim \sup_{f\in\Lambda_\beta}\bigg|\langle f,g_\omega-g_{\omega'}\rangle_{2}\bigg|\\
    &=\sup_{\nu}c_\alpha c_\beta \sum_{|\bold{k}|\leq l}\sum_{\bold{s}\in\rho(\bold{k})}\nu_\Bs (\omega_\Bs-\omega'_\Bs)\\
    &=c_\alpha c_\beta \sum_{|\bold{k}|\leq l}\sum_{\bold{s}\in\rho(\bold{k})}\mathbbm{1}_{\{\omega_\Bs\neq \omega'_\Bs\}}\coloneqq c_\alpha c_\beta  d(\omega,\omega')
\end{align*}
 where $d(\omega,\omega')$ denotes the $l^1$ distance between $\omega$ and $\omega'$. If we only consider $h=\sum_{|\bold{k}|\leq l}\sum_{\bold{s}\in\rho(\bold{k})} 1\asymp 2^{l}l^{D-1}$ elements in $\Omega_\alpha$, according to Lemma \ref{lem:VGbound}, there exist a subset $\{\omega_i\}_{i=0}^H$ such that
\begin{align*}
    &d(w_i,w_j)\gtrsim 2^{l}l^{D-1}\quad \forall i,j\in [H],i\neq j\\
    & \log H \gtrsim 2^{l}l^{D-1}.
\end{align*}

%Now consider lemma \ref{lem:VGbound}, we let $h=\sum_{|\bold{k}|\leq l}\sum_{\bold{s}\in\rho(\bold{k})} 1\asymp 2^{l}l^{D-1}.$ 

Given any estimators $\hat{g}_{\omega_j}$   for any $g_{\omega_j}\in\Omega_\alpha$ conditioned on $n$ i.i.d. samples, we have
\begin{align*}
    D_{KL}(\hat{g}_{\omega_j},\hat{g}_{\omega_0})&=nD_{KL}({g}_{\omega_j},{g}_{\omega_0})\\
    &=n\int_{\Omega}-\log\bigg(1+\frac{g_{\omega_0}-g_{\omega_j}}{g_{\omega_j}}\bigg)g_{\omega_j}d\Bx\\
    &\leq n \int_\Omega\frac{[g_{\omega_0}-g_{\omega_j}]^2}{g_{\omega_j}}d\Bx\\
    &\lesssim n c_\alpha^2\|\omega_\Bs- \omega_0\|_{2}^2\\
    &= n c_\alpha^2\sum_{\bold{s}\in\rho(\bold{k})}\mathbbm{1}_{\{\omega_\Bs=1\}}\\
    &\lesssim nc_\alpha^2 h=n2^{-2\alpha l}
\end{align*}
where the third line is from the identity $\log (1+x)\geq x-x^2$ for any $x>-1/2$ and the the following estimation
\begin{align*}
    \frac{g_{\omega_0}-g_{\omega_j}}{g_{\omega_j}}&=\frac{1}{g_{\omega_j}}-1\\
    &=\frac{1+c_{\omega_j}}{1+c_\alpha \sum_{|\Bk|\leq l}\sum_{\Bs\in\rho(\Bk)}\omega_{j,\Bs}\psi_\Bs}-1\\
    &> \bigg({1+c_\alpha\sup_{\Bx\in\Omega,\omega\in\{0,1\}^h}\sum_{|\Bk|\leq l}\sum_{\Bs\in\rho(\Bk)}\omega_{\Bs}\psi_\Bs}\bigg)^{-1}-1\\
    &\geq \bigg(1+2^{-(\alpha +\frac{1}{2})l}\|\sum_{|\Bk|\leq l}\sum_{\Bs\in\rho(\Bk)}\psi_\Bs\|_\infty\bigg)^{-1}-1\\
    &\geq \big(1+2^{-\alpha l+\frac{l}{2}}\big)^{-1}-1\\
    &= -\frac{1}{2}
\end{align*}
where the fifth line is from our assumption $\alpha>1/2$ and Lemma 2.4 in \cite{Dung16}, which states that $\|\sum_{|\Bk|\leq l}\sum_{\Bs\in\rho(\Bk)}\psi_\Bs\|_\infty\asymp 2^{l}$.

In order to use lemma \ref{lem:Fano},  we also need:
\begin{align*}
    \frac{1}{H}\sum_{j=1}^HD_{KL}(\hat{g}_{\omega_j},\hat{g}_{\omega_0})\lesssim h \lesssim \log H.
\end{align*}
The $l$ that satisfies $2^l\asymp n^{\frac{1}{(2\alpha+1)}}[\log n]^{\frac{1-D}{(2\alpha+1)}}$ does the job. Moreover, according to lemma \ref{lem:VGbound}, there exist a subset of $\Omega_H\subset\Omega_\alpha$ such that, for any pair $g_\omega$ and $g_{\omega'}\in \Omega_H$:
$$d_{F_d}(g_\omega,g_{\omega'})\geq c_\alpha c_\beta d(\omega,\omega')\gtrsim c_\alpha c_\beta h\asymp n^{-\frac{\alpha+\beta}{2\alpha+1}}[\log n]^{(D-1)\frac{\alpha+\beta}{2\alpha+1}}.$$
As we have checked all the required conditions in lemma \ref{lem:Fano}, we can use it  to have the final result:
\begin{align*}
    &\quad \inf_{p_n}\sup_{p\in\CalH^\alpha_{mix}(1)}d_{F_d}(p_n,p)\\
    &\geq \inf_{\hat{g}}\sup_{g\in\CalH^\alpha_{mix}(1)}\E\sup_{f\in\CalH^\beta_{mix}(1)}\langle f,\hat{g}-g\rangle_2\\
    &\gtrsim \inf_{\hat{\omega}}\sup_{\omega\in\{\omega_j\}_{j=0}^H}\E d(g_{\hat{\omega}},g_\omega)\\
    &\gtrsim n^{-\frac{\alpha+\beta}{2\alpha+1}}[\log n]^{(D-1)\frac{\alpha+\beta}{2\alpha+1}}\inf_{\hat{g}}\sup_{g\in\CalH^\alpha_{mix}}\mathbb{P}\bigg(d(g_{\hat{\omega}},g_\omega)\geq n^{-\frac{\alpha+\beta+1}{2(\alpha+1)}}[\log n]^{(D-1)\frac{\alpha+\beta+1}{2(\alpha+1)}}\bigg)\\
    &\asymp n^{-\frac{\alpha+\beta}{2\alpha+1}}[\log n]^{(D-1)\frac{\alpha+\beta}{2\alpha+1}}+n^{-\frac{1}{2}}.
\end{align*}
\subsection{H\"older Discriminator Class: $F_d=\CalW^\beta_\infty$}\label{sec:LowerBound_Holder.}
Our proof still relies on lemma \ref{lem:VGbound} and \ref{lem:Fano}. However, in this can, we need to first introduce an expansion for functions in $\CalW^\beta_\infty$ by the  cardinal \textit{B-spline} basis, which is defined as follows:
\begin{defn}[B-Spline]\label{dfn:B_spline}
Let $N(x)=\mathbbm{1}_{\{x\in [0,1]\}}$. Then
the cardinal B-spline of order $\alpha$ is defined by taking $(\alpha+1)$-times convolution of $N$:
\[N_\alpha=\underbrace{N*N*\cdots*N}_{\alpha+1\ \  \text{time}}.\]
For any $\Bk,\Bs\in\NatInt^D$, define
\[M^{D,\alpha}_{\Bk,\Bs}(\Bx)=\prod_{j=1}^DN_\alpha(2^{k_j}x_j-s_j).\]
For any $k\in\NatInt$ and $\Bs\in \NatInt^D$, define
\[M^{D,\alpha}_{k,\Bs}(\Bx)=\prod_{j=1}^DN_\alpha(2^{k}x_j-s_j).\]
\end{defn}
The following two theorems provide equivalent definitions of H\"older spaces and mixed smooth Sobolev spaces, respectively.  
\begin{theorem}[\cite{Ronald88}]
\label{thm:B_spline_Holder}
Let $S_\alpha(k)=\{-\alpha,-\alpha+1,\cdots,2^k\}^D$. Then  any function $f\in\CalW^\alpha_\infty$ can be decomposed as
\[f(\Bx)=\sum_{k\in\NatInt}\sum_{\Bs\in S_\alpha(k)}f_{k,\Bs} M^{D,\alpha}_{k,\Bs}(\Bx)\]
with the following norm equivalence relation:
\[\|f\|_{\CalW^\alpha_\infty}\asymp \sup_{k\in\NatInt}2^{\alpha k}\sup_{\Bs\in S_\alpha(k)}|f_{k,\Bs}|.\]
\end{theorem}

\begin{theorem}[\cite{Dung16}]
\label{thm:B_spline_mixSobolev}
Let $S_\alpha(\Bk)=\times_{j=1}^D\{-\alpha,-\alpha+1,\cdots,2^{k_j}\}$. Then  any function $f\in\CalH^\alpha_{mix}$ can be decomposed as
\[f(\Bx)=\sum_{\Bk\in\NatInt^D}\sum_{\Bs\in S_\alpha(\Bk)}f_{\Bk,\Bs} M^{D,\alpha}_{\Bk,\Bs}(\Bx)\]
with the following norm equivalence relation:
\[\|f\|_{\CalH^\alpha_{mix}}\asymp \bigg(\sum_{\Bk\in\NatInt^D}\sum_{\Bs\in S_\alpha(\Bk)}2^{(2\alpha-1)|\Bk|}f_{\Bk,\Bs}^2\bigg)^{\frac{1}{2}}.\]
\end{theorem}
Now we can continue our proof. Similar to section \ref{sec:LowerBound_Mix}, we need to build $\Omega_\alpha$ and $\Lambda_\beta$ such that, conditioned on $n$ samples, it is hard to distinguish the estimators for functions in $\Omega_\alpha$ under the loss induced by the discriminator $\Lambda_\beta$. Unlike the proof in \ref{sec:LowerBound_Mix}, in this subsection, both $\Omega_\alpha$ and $\Lambda_\beta$ consist of B-splines.

We notice that the support of $M^{D,\alpha}_{k,\Bs}$ is $\times_{j=1}^D[2^{-k}s_j,2^{-k}s_j+2^{-k}(\alpha+1)]$. Therefore, given any $k\in\NatInt$, there must be at lest $\lfloor2 ^{kD}/(\alpha+1)^D\rfloor$ functions $\{M^{D,\alpha}_{k,\Bs}\}$ such that the supports of this functions are mutually disjoint and form a partition of $[0,1]^D$. So we define the following subset of $S_\alpha(k)$:
\begin{equation}\label{eq:B_spline_partition}
    \tilde{S}_\alpha(k)=\{\forall \Bs,\Bs'\in \tilde{S}_\alpha(k), {\rm supt}[M^{D,\alpha}_{k,\Bs}]\bigcap {\rm supt}[M^{D,\alpha}_{k,\Bs'}]=\emptyset, \bigcup_{\Bs\in\tilde{S}_\alpha(k)} {\rm supt}[M^{D,\alpha}_{k,\Bs}]\supseteq [0,1]^D  \}
\end{equation}
and  $|\tilde{S}_\alpha(k)|=\lfloor 2^{kD}/(\alpha+1)^D\rfloor$. Then the required $\Omega_\alpha$ and $\Lambda_\beta$ are constructed as follows:

\begin{align*}
    &\Omega_\alpha=\bigg\{g_\omega: \frac{1}{1+c_\omega}\big(1+2^{-\alpha kD}\sum_{\Bs\in \tilde{S}_\alpha(k)}\omega_{\Bs}M^{d,\alpha}_{k,\Bs}\big), \ \omega_\Bs\in\{0,1\}\bigg\}\\
    &\Lambda_\beta=\bigg\{f_\nu: 2^{-k\beta}\sum_{\Bs\in\tilde{S}_{\alpha}(k)}\nu_\Bs M^{d,\alpha}_{k,\Bs},\ \nu_\Bs\in\{0,1\}\bigg\}
\end{align*}
where 
\begin{align*}
    c_\omega&=\int_{[0,1]^D}2^{-\alpha kD}\sum_{\Bs\in \tilde{S}_\alpha(k)}\omega_{\Bs}M^{d,\alpha}_{k,\Bs}(\Bx) d\Bx\\
    &=(\alpha+1)^D2^{-(\alpha+1) kD}\sum_{\Bs\in \tilde{S}_\alpha(k)}\omega_{\Bs}\\
    &\leq (\alpha+1)^D2^{-(\alpha+1) kD}|\tilde{S}_\alpha(k)|\lesssim 2^{-\alpha kD}.
\end{align*}
Obviously, $\Lambda_\beta$ is a subset of $\CalW^\beta_\infty$ according to theorem \ref{thm:B_spline_Holder}. For any function $g_\omega\in\Omega_\alpha$, we have
\begin{align*}
    \|g\|_{\CalH^\alpha_{mix}}&=\frac{1}{1+c_\omega}\bigg(1+2^{-2\alpha kD}\sum_{\Bs\in\tilde{S}_\alpha(k)}2^{(2\alpha-1) kD}\omega_\Bs^2\bigg)^{\frac{1}{2}}\\
    &\leq \frac{1}{1+c_\omega}\bigg(1+2^{-kD}\sum_{\Bs\in\tilde{S}_\alpha(k)}1\bigg)^{\frac{1}{2}}\leq 1.
\end{align*}
Therefore, $\Omega_\alpha$ is a subset of $\CalH^\alpha_{mix}.$

Then for any pair of functions $g_\omega,g_{\omega'}\in\Omega_\alpha$, their H\"older IPM distance is bounded below as follows:
\begin{align*}
    d_{F_d}(g_\omega,g_{\omega'})&=\sup_{f\in\CalW^\beta_\infty}(1)\bigg|\langle f,g_\omega-g_{\omega'}\rangle_{2}\bigg|\\
    &\geq \sup_{f\in\Lambda_\beta}\bigg|\langle f,g_\omega-g_{\omega'}\rangle_{2}\bigg|\\
    &\asymp \sup_{\nu}2^{-k\beta-\alpha kD} \sum_{\Bs\in\tilde{S}_\alpha(k)}\nu_\Bs (\omega_\Bs-\omega'_\Bs)\min_{\Bs\in\tilde{S}_\alpha(k)}\int\big|M^{D,\alpha}_{k,\Bs}(\Bx)\big|^2d\Bx\\
    &=2^{-(\beta+\alpha D+D)k} \sum_{\Bs\in\tilde{S}_\alpha(k)}\mathbbm{1}_{\{\omega_\Bs\neq \omega'_\Bs\}}\coloneqq 2^{-(\beta+\alpha D+D)k} d(\omega,\omega')
\end{align*}
 If we consider $h=|\tilde{S}_\alpha(k)|=\lfloor 2^{kD}/(\alpha+1)^D\rfloor$ elements in $\Omega_\alpha$, according to Lemma \ref{lem:VGbound}, there exist a subset $\{\omega_i\}_{i=0}^H$ such that
\begin{align*}
    &d(w_i,w_j)\geq \lfloor 2^{kD}/(\alpha+1)^D\rfloor/8\quad \forall i,j\in [H],i\neq j\\
    & \log H \geq \lfloor 2^{kD}/(\alpha+1)^D\rfloor\frac{\log 2}{8}.
\end{align*}
Given any estimators $\hat{g}_{\omega_j}$   for any $g_{\omega_j}\in\Omega_\alpha$ conditioned on $n$ i.i.d. samples, we have
\begin{align*}
    D_{KL}(\hat{g}_{\omega_j},\hat{g}_{\omega_0})&=nD_{KL}({g}_{\omega_j},{g}_{\omega_0})\\
    &=n\int_{\Omega}-\log\bigg(1+\frac{g_{\omega_0}-g_{\omega_j}}{g_{\omega_j}}\bigg)g_{\omega_j}d\Bx\\
    &\leq n \int_\Omega\frac{[g_{\omega_0}-g_{\omega_j}]^2}{g_{\omega_j}}d\Bx\\
    &\leq n 2^{-2\alpha kD}\|\omega_\Bs- \omega_0\|_{2}^2\max_{\Bs\in\tilde{S}_\alpha(k)}\int\big|M^{D,\alpha}_{k,\Bs}(\Bx)\big|^2d\Bx\\
    &= n  2^{-2\alpha kD-kD}\sum_{\bold{s}\in \tilde{S}_\alpha(k)}\mathbbm{1}_{\{\omega_\Bs=1\}}\\
    &\leq n2^{-2\alpha kD-kD} h\leq n2^{-2\alpha kD}/(\alpha+1)^D
\end{align*}
where the third line is from the identity $\log (1+x)\geq x-x^2$ for any $x>-1/2$ and the fact that
\[\frac{g_{\omega_0}-g_{\omega_j}}{g_{\omega_j}}=\frac{1}{g_{\omega_j}}-1\geq \frac{1+c_{\omega_j}}{1+2^{-\alpha kD}\max_{\Bs\in\tilde{S}_\alpha(k)}\|M^{d,\alpha}_{k,\Bs}\|_\infty}-1>-\frac{1}{2}.\]
In order to use lemma \ref{lem:Fano},  we also need:
\begin{align*}
    \frac{1}{H}\sum_{j=1}^HD_{KL}(\hat{g}_{\omega_j},\hat{g}_{\omega_0})\lesssim h \lesssim \log H.
\end{align*}
We can choose $2^k=n^{\frac{1}{2\alpha D+D}}$ to acchieve this requirement.  Moreover, according to lemma \ref{lem:VGbound}, there exist a subset of $\Omega_H\subset\Omega_\alpha$ such that, for any pair $g_\omega$ and $g_{\omega'}\in \Omega_H$:
$$d_{F_d}(g_\omega,g_{\omega'})\geq 2^{-(\beta+\alpha D+D)k} d(\omega,\omega')\geq  2^{-(\beta+\alpha D+D)k} h\geq 2^{-(\beta+\alpha D)k}.$$
As we have checked all the required conditions in lemma \ref{lem:Fano}, we can use it  to have the final result:
\begin{align*}
    &\quad \inf_{p_n}\sup_{p\in\CalH^\alpha_{mix}(1)}d_{F_d}(p_n,p)\\
    &\geq \inf_{\hat{g}}\sup_{g\in\CalH^\alpha_{mix}(1)}\E\sup_{f\in\CalW^\beta_\infty(1)}\langle f,\hat{g}-g\rangle_2\\
    &\gtrsim \inf_{\hat{\omega}}\sup_{\omega\in\{\omega_j\}_{j=0}^H}\E d(g_{\hat{\omega}},g_\omega)\\
    &\gtrsim n^{-\frac{\alpha D+\beta}{2\alpha D+D}}+n^{-\frac{1}{2}}= n^{-\frac{\alpha +\beta/D}{2\alpha +1}}+n^{-\frac{1}{2}}.
\end{align*}
\section{Proof of Theorem \ref{thm:empEst}}
Main idea of the proof is to select a true density $p^*$ from mixed smooth Space $\CalH^\alpha_{mix}$ and  a function $f^*$ from the discriminator class $F_d$ such that $\E_{X\sim p^*}[f^*(X)]-\E_{Y\sim \hat{p}}[f^*(Y)]$ is bounded below by the target rate.

For $p^*$, we choose the density of uniform distribution $p^*(\Bx)=1$ for any $\Bx\in\Omega$. Obviously, $p^*\in\CalH^\alpha_{mix}$ for any $\alpha>0$ because it is a constant. For candidate function $f^*\in F_d$, we need to split our proof into two cases, namely $F_d=\CalW^\beta_\infty$ and $F_d=\Phi(H,W,S,B)$.
\subsection{H\"older Discriminator Class: $F_d=\CalW^\beta_\infty$}
\label{sec:empirical_lowBound_Holder}
For any $n$, we select an integer  $k^*\in\NatInt$ such that $2^{(k^*-1)D}\leq 2n<2^{k^*D}$. Let $\{M^{D,\beta}_{k,\Bs}\}$ be the B-spline functions defined in \ref{dfn:B_spline} and let $\tilde{S}_{\beta}(k^*)$ be the index set defined in \eqref{eq:B_spline_partition} such that any $\Bs,\Bs'\in\tilde{S}_\beta(k^*)$, the supports of $M^{D,\beta}_{k^*,\Bs}$ and $M^{D,\beta}_{k^*,\Bs'}$ are mutually disjoint and
\[\bigcup_{\Bs\in\tilde{S}_{\beta}(k^*)} {\rm supt}[M^{D,\beta}_{k^*,\Bs}]=[0,1]^D.\]

Because the set $\mathcal{M}=\{M^{D,\beta}_{k^*,\Bs}:\Bs\in\tilde{S}_{\beta}(k^*) \}$ has at least $2n$ functions and all the functions in $\mathcal{M}$ are mutually disjoint, there must be at half of the functions in $\mathcal{M}$ that satisfy:
\[M^{D,\beta}_{k^*,\Bs}(X_i)=0,\quad i=1,\cdots,n\]
for any samples $X_i\in\BX$. Let $\tilde{S}_{\beta,\BX}(k^*)$ denote the set of indices of all these functions:
\begin{equation}
    \label{eq:empirical_0Spline}
    \tilde{S}_{\beta,\BX}(k^*)=\{\Bs\in\tilde{S}_\beta(k^*):M^{D,\beta}_{k^*,\Bs}(\BX)=0 \}
\end{equation}
so $|\tilde{S}_{\beta,\BX}(k^*)|\geq n$ and let
\begin{equation}
    \label{eq:empirical_star_f}
    f^*(\Bx)=2^{-\beta k^*}\sum_{\Bs\in\tilde{S}_{\beta,\BX}(k^*)}M^{D,\beta}_{k^*,\Bs}(\Bx).
\end{equation}
According to theorem \ref{thm:B_spline_Holder}, $f^*\in\CalW^\beta_\infty$ because
\[\|f^*\|_{\CalW^\beta_\infty}\asymp2^{\beta k^*}\sup_{\Bs\in\tilde{S}_{\beta,\BX}(k^*)}2^{-\beta k^*}=1.\]

Direct calculations then show
\begin{align*}
    \E_{X\sim p^*}[f^*(X)]-\E_{Y\sim \hat{p}}[f^*(Y)]&=\int_\Omega f^*(\Bx)d\Bx-\frac{1}{n}\sum_{i=1}^nf^*(X_i)\\
    &=2^{-\beta k^*}\sum_{\Bs\in\tilde{S}_{\beta,\BX}(k^*)}\int M^{D,\beta}_{k^*,\Bs}(\Bx)d\Bx\\
    &\gtrsim 2^{-\beta k^*} \asymp n^{-\frac{\beta}{D}}.
\end{align*}

\subsection{ReLU Net Discriminator Class: $F_d=\Phi(H,W,S,B)$}
\label{sec:empirical_lowBound_ReLU}
The idea is that we first construct the target function
\begin{equation}
    f^*(\Bx)=2^{- k^*}\sum_{\Bs\in\tilde{S}_{1,\BX}(k^*)}M^{D,1}_{k^*,\Bs}(\Bx),
\end{equation}
which is a special case of \eqref{eq:empirical_star_f} with $\beta=1$ and explore the structure requirement of ReLU net $\Phi(H,W,S,B)$ for approximating $f^*$. The following lemma, which is a special case of  Lemma 1 in \cite{suzuki2018adaptivity}, can help 
\begin{lem}[\cite{suzuki2018adaptivity}]\label{lem:suzuki_Bspline}
There exists a constant $c$  depending only on dimension $D$ such that, for all $\varepsilon>0$, there exists a ReLU net $\hat{M}^{D,1}_{0,0}\in\Phi(H,W,S,B)$ with
    \[H=3+2\lceil log_2\big(\frac{3^D}{\varepsilon c}\big)+5\rceil\lceil\log_2 D\rceil,\quad W=20D,\quad S=HW^2,\quad B=4\]
that satisfies 
\[\|\hat{M}^{D,1}_{0,0}-{M}^{D,1}_{0,0}\|_\infty\leq \varepsilon\]
and $\hat{M}^{D,1}_{0,0}(\Bx)=0$ for all $\Bx\not\in (0,2)^D$.
\end{lem}
Based on Lemma \ref{lem:suzuki_Bspline}, our goal become constructing the following ReLU net
\[\hat{f^*}=2^{-k^*}\sum_{\Bs\in\tilde{S}_{1,\BX}(k^*)}\hat{M}^{D,1}_{k^*,\Bs}(\Bx)\]
where $\hat{M}^{D,1}_{k^*,\Bs}(\Bx)=\hat{M}^{D,1}_{0,0}(2^{k^*}\Bx+\Bs)$ such that
\[\|f^*-\hat{f^*}\|_\infty\lesssim n^{-\frac{1}{D}}.\]
Notice that, for any $\Bs$, $\hat{M}^{D,1}_{k^*,\Bs}$ is a translation of $\hat{M}^{D,1}_{k^*,0}$ with  $\hat{M}^{D,1}_{k^*,0}= \hat{M}^{D,1}_{0,0}(2^{k^*}\Bx)$, we can immediately derive via lemma \ref{lem:suzuki_Bspline} that for any $\Bs$, if we let
\[H_\Bs=3+2\lceil log_2\big(\frac{3^D}{ c}\big)+5\rceil\lceil\log_2 D\rceil,\quad W_\Bs=20D,\quad S_\Bs=H_\Bs W_\Bs^2,\quad B_\Bs=4\cdot2^{k^*}\leq 4(2n)^{\frac{1}{D}}\]
then we can get  $\hat{M}^{D,1}_{k^*,\Bs}$ from $\Phi(H_\Bs,W_\Bs,S_\Bs,B_\Bs)$ for any $\Bs\in\tilde{S}_{1,\BX}(k^*)$ with
\[\| \hat{M}^{D,1}_{k^*,\Bs}-{M}^{D,1}_{k^*,\Bs}\|_\infty\leq 1.\]
To construct $\hat{f^*}$, we needs $|\tilde{S}_{1,\BX}(k^*)|$ such ReLU net to build their linear combination. Therefore, if
\[H=3+2\lceil log_2\big(\frac{3^D}{ c}\big)+5\rceil\lceil\log_2 D\rceil,\quad W=20D\vee |\tilde{S}_{1,\BX}(k^*)|\leq 20D\vee 2n,\quad S=H W^2,\quad B=4\cdot2^{k^*}\leq 4(2n)^{\frac{1}{D}}\]
then we can get a $\hat{f}^*\in\Phi(H,W,S,B)$ such that
\begin{align*}
    \|\hat{f^*}-f^*\|_\infty&=\|2^{-k^*}\sum_{\Bs\in\tilde{S}_{1,\BX}(k^*)}\hat{M}^{D,1}_{k^*,\Bs}-t{M}^{D,1}_{k^*,\Bs}\|_\infty\\
    &=2^{-k^*}\sup_{\Bs\in\tilde{S}_{1,\BX}(k^*)}\|\hat{M}^{D,1}_{k^*,\Bs}-t{M}^{D,1}_{k^*,\Bs}\|_\infty\\
    &\lesssim n^{-\frac{1}{D}}
\end{align*}
where the second line is because supports of $\hat{M}^{D,1}_{k^*,\Bs}$ and $\hat{M}^{D,1}_{k^*,\Bs'}$ are disjoint for any $\Bs,\Bs'\in\tilde{S}_{1,\BX}(k^*)$.

At last we can compute the distance between $p^*$ and its empirical density $\hat{p}$ with loss induced by $\hat{f^*}$:
\begin{align*}
    \E_{X\sim p^*}[\hat{f^*}(X)]-\E_{X\sim \hat{p}}[\hat{f^*}(X)]&= \E_{X\sim p^*}[\hat{f^*}(X)]\\
    &=\int_{\omega} f^*(\Bx)+ \hat{f^*}(\Bx)-f^*(\Bx) d\Bx\\
    &\gtrsim n^{-\frac{1}{D}}- \|\hat{f^*}-f^*\|_\infty\\
    &\gtrsim n^{-\frac{1}{D}}.
\end{align*}

\section{Proof of Theorem \ref{thm:GAN-improved}}
Main idea of the proof follows theorem 1 and theorem 2 in \cite{chen2020statistical}. The difference is that we need to replace the empirical density in \cite{chen2020statistical} by AHCE and, so, estimations of statistical quantities are different. We also consider the case the AHCE is constructed via Fourier series only for the wavelet case can be proved in the same manner.

The first step is to decompose the H\"older metric as follows:
\begin{align*}
    d_{\CalW^\beta_\infty(L_\beta)}(\hat{g}_n\#\mu,p)&\leq  d_{\CalW^\beta_\infty(L_\beta)}(\hat{g}_n\#\mu,\tilde{p}_n)+d_{\CalW^\beta_\infty(L_\beta)}(\tilde{p}_n,p)\\
    &\leq \underbrace{d_{\CalW^\beta_\infty(L_\beta)}(\hat{g}_n\#\mu,\tilde{p}_n)-d_{\Phi_d}(\hat{g}_n\#\mu,\tilde{p}_n)}_{A}+\underbrace{d_{\Phi_d}(\hat{g}_n\#\mu,\tilde{p}_n)}_{B}+\underbrace{d_{\CalW^\beta_\infty(L_\beta)}(\tilde{p}_n,p)}_C.
\end{align*}
For term $A$, we have
\begin{align*}
    d_{\CalW^\beta_\infty(L_\beta)}(\hat{g}_n\#\mu,\tilde{p}_n)-d_{\Phi_d}(\hat{g}_n\#\mu,\tilde{p}_n)&\leq \sup_{f\in\CalW^\beta_\infty}\inf_{f'\in\Phi_d}\int_{\Omega}f(\hat{g}_n\#\mu-\tilde{p}_n)-f'(\hat{g}_n\#\mu-\tilde{p}_n)d\Bx\\
    &=\sup_{f\in\CalW^\beta_\infty}\inf_{f'\in\Phi_d}\int_{\Omega}(f-f')(\hat{g}_n\#\mu-\tilde{p}_n)d\Bx\\
    &\leq 2\sup_{f\in\CalW^\beta_\infty}\inf_{f'\in\Phi_d}\|f-f'\|_\infty \coloneqq  2\varepsilon_1.
\end{align*}
For term $B$, we have
\begin{align*}
    d_{\Phi_d}(\hat{g}_n\#\mu,\tilde{p}_n)&=\inf_{g\in\Phi_g}d_{\Phi_d}(g\#\mu,\tilde{p}_n)\\
    &=\inf_{g\in \Phi_g}\bigg\{d_{\Phi_d}(g\#\mu,\tilde{p}_n)-d_{\CalW^\beta_\infty}(g\#\mu,\tilde{p}_n)+d_{\CalW^\beta_\infty}(g\#\mu,\tilde{p}_n)\bigg\}\\
    &\leq 2\sup_{f\in\CalW^\beta_\infty}\inf_{f'\in\Phi_d}\|f-f'\|_\infty+\inf_{g\in \Phi_g}d_{\CalW^\beta_\infty}(g\#\mu,\tilde{p}_n)\\
    &\coloneqq 2\varepsilon_1 +\varepsilon_2.
\end{align*}
For term $C$, we have proved in theorem \ref{thm:densityIPM} that $\E d_{\CalW^\beta_\infty(L_\beta)}(\tilde{p}_n,p)\lesssim n^{-\frac{\alpha+\beta/D}{2\alpha+1}}\big[\log n\big]^{D\frac{\alpha+\beta/D+1}{2\alpha+1}}+n^{-\frac{1}{2}}$.

The rest of the proof is to estimate $\varepsilon_1$ and $\varepsilon_2$. We call $\varepsilon=\sup_{f\in\CalW^\beta_\infty}\inf_{f'\in\Phi_d}\|f-f'\|_\infty$  the discriminator error because it only depends on the complexity of ReLU net $\Phi_d$; we call $\inf_{g\in \Phi_g} d_{\CalW^\beta_\infty}(g\#\mu,\tilde{p}_n)$ the generator error because it only depends on the complexity of ReLU net $\Phi_g$.

\subsection{Discriminator Error}\label{sec:improved_GAN_discriminator_err}
The following theorem from \cite{chen2019efficient} immediately implies the target result:
\begin{theorem}[\cite{chen2019efficient}]\label{thm:chen_efficient}
Given any $\delta\in(0,1)$, there exists a ReLU network
architecture such that, for any $f\in\CalW^\beta_\infty$ with $\beta\geq 1$, if the parameters of the network is properly chosen, then the network yields a function $\hat{f}$ for the approximation of $f$ with $\|f-\hat{f}\|_\infty\leq \delta$. Such a network has (i) no more than $c(\log\frac{1}{\delta}+1)$ layers; (ii) at most $c' \delta^{-\frac{D}{\beta}}(\log\frac{1}{\delta}+1)$ activation functions and weights where $c,c'$ only depend on $D$, $\beta$ and $\|f\|_{\CalW^\beta_\infty}$.
\end{theorem}
Theorem \ref{thm:chen_efficient} essentially states that if we choose $H=c\log(\frac{1}{\delta}+1)$ and $W+S = c' \delta^{-\frac{D}{\beta}}(\log\frac{1}{\delta}+1)$ then the discriminator ReLU net $\Phi_d(H,W,S)$ (no requirement of $B$) can yields a function $\hat{f}$ close to $f$.

We let $\delta= n^{-\frac{\alpha+\beta/D}{2\alpha+1}}\big[\log n\big]^{D\frac{\alpha+\beta/D+1}{2\alpha+1}}+n^{-\frac{1}{2}}$ so  if
\[H\asymp (\log n^{\frac{\alpha+\beta/D}{2\alpha+1}}+1)\asymp \log n+1,\quad W+S\asymp \delta^{-\frac{D}{\beta}}H\asymp (n^{\frac{\alpha D+\beta}{2\alpha\beta+\beta}}+n^{\frac{D}{2\beta}}) H\]
the $\varepsilon_1$ is bounded by
\[\sup_{f\in\CalW^\beta_\infty}\inf_{f'\in\Phi_d(H,W,S)}\|f-f'\|_\infty\leq  n^{-\frac{\alpha+\beta/D}{2\alpha+1}}\big[\log n\big]^{D\frac{\alpha+\beta/D+1}{2\alpha+1}}+n^{-\frac{1}{2}} .\]
\subsection{Generator Error}\label{sec:improved_GAN_generator_err}
The proof for generator error involves the following lemma regarding optimal transport between two spaces:
\begin{lem}[\cite{caffarelli1992regularity}]\label{lem:optimal_transport}
Let $\CalX$ and $\mathcal{Z}$ be two convex compact subset of $\Real^D$ equipped with probability measure $\mu$ and $p$ respectively: $(\CalX,\mu)$ and $(\mathcal{Z},p)$. Suppose
\begin{enumerate}[label=(\roman*)]
    \item density $p$ is lower bounded $\inf_{\Bx\in\Omega}p(\Bx)>0$ and  $p\in\CalW^\alpha_\infty$ with $\alpha>0$;
    \item density $\mu$ is compactly supported on a convex domain and infinitely differentiable.
\end{enumerate}
Then there exists a transformation $T:\CalX\to\mathcal{Z}$ such that $T\#\mu =p$ with $T\in\CalW^{\alpha+1}_\infty$.
\end{lem}
Our goal is to construct a generator ReLU net $\Phi_g$ so that 
\[\E \inf_{g\in \Phi_g}d_{\CalW^\beta_\infty}(g\#\mu,\tilde{p}_n)\lesssim n^{-\frac{\alpha+\beta/D}{2\alpha+1}}\big[\log n\big]^{D\frac{\alpha+\beta/D+1}{2\alpha+1}}+n^{-\frac{1}{2}} .\]
From lemma \ref{lem:optimal_transport}, we can see that we need to prove the existence of the mapping $T$ such that $T\#\mu=\tilde{p}_n$ and then construct a generator $\Phi_g$ to approximate $T$. To this end, we first need to prove that our AHCE $\tilde{p}_n$ is H\"older continuous. The following lemma is one of the components for estimating the H\"older condition of $\tilde{p}_n$: 
\begin{lem}\label{lem:Hoeffding_AHCE}
 Let  $p$ be any density in $\CalH_{mix}^\alpha$. Let $\zeta(\Bk)$ be any non-negative function of  $\Bk\in\NatInt^D$ which satisfies
 \[E^2\coloneqq 2\sum_{|\Bk|\leq l}\zeta^2(\Bk)2^{|\Bk|}\sum_{\Bs\in\rho(\Bk)}\hat{p}_\Bs^2<\infty.\]
Let $\tilde{p}_n$ be the AHCE of $p$. Then for any $1<p<\infty$ and $S$ large, we have
\[\pr\bigg(\big\|\big[\sum_{|\Bk|\leq l}\zeta^2(\Bk)\big(\sum_{\Bs\in\rho(\Bk)}\tilde{p}_\Bs\phi_\Bs(\cdot)\big)^2\big]^{\frac{1}{2}}\big\|_p>S\bigg)\leq \sum_{|\Bk|\leq l}\exp\bigg\{-C_{D,\alpha}\frac{2^{(1+2\alpha)l}\sqrt{(S^2-E^2)/2}}{2^{|\Bk|}l^{2D}{l\choose D}\zeta(\Bk)}+|\Bk|\bigg\}\]
for some constant $C_{D,\alpha}$ only depends on $D$ and $\alpha$.
\end{lem}
\begin{proof}
For notation simplicity, let $F_{\Bk}[\tilde{p}_n]=\sum_{\Bs\in\rho(\Bk)}\tilde{p}_{\Bs}\phi_{\Bs}$. Then
\begin{align*}
    \big\|\big(\sum_{|\Bk|\leq l}\zeta^2(\Bk)F^2_\Bk[p]\big)^{\frac{1}{2}}\big\|_p^2&\leq \sum_{|\Bk|\leq l}\zeta^2(\Bk)\sup_{\Bx\in\Omega}F^2_{\Bk}(\Bx)\\
    &\leq \sum_{|\Bk|\leq l}\zeta^2(\Bk)\bigg(\sum_{\Bs\in\rho(\Bk)}\big|\tilde{p}_\Bs\big|\bigg)^2\\
    &\leq \sum_{|\Bk|\leq l}\zeta^2(\Bk)\bigg(\sum_{\Bs\in\rho(\Bk)}\big|\hat{p}_\Bs\big|+\big|\tilde{p}_\Bs-\hat{p}_\Bs\big|\bigg)^2\\
    &\leq 2\sum_{|\Bk|\leq l}\zeta^2(\Bk)\bigg(\sum_{\Bs\in\rho(\Bk)}\big|\hat{p}_\Bs\big|\bigg)^2+2\sum_{|\Bk|\leq l}\zeta^2(\Bk)\bigg(\sum_{\Bs\in\rho(\Bk)}\big|\tilde{p}_\Bs-\hat{p}_\Bs\big|\bigg)^2\\
    &\leq  \underbrace{2\sum_{|\Bk|\leq l}\zeta^2(\Bk)2^{|\Bk|}\sum_{\Bs\in\rho(\Bk)}\hat{p}_\Bs^2}_{E^2}+2\bigg(\sum_{|\Bk|\leq l}\sum_{\Bs\in\rho(\Bk)\zeta(\Bk)}\zeta(\Bk)\big|\tilde{p}_\Bs-\hat{p}_\Bs\big|\bigg)^2
\end{align*}
where the second line is from the fact that $\|\phi_\Bs\|_\infty\leq 1$ for any $\Bs$ and the last line is from $l^p$ inequality: $\sum_{i}^n|x_i|\leq \sqrt{n} (\sum_{i=1}^nx_i^2)^{1/2}$ and $(\sum_{i=1}^nx_i^2)^{1/2}\leq  \sum_{i}^n|x_i|$.
So, we can have the following upper bound:
\begin{align}
    \pr\bigg(\big\|\big[\sum_{|\Bk|\leq l}\zeta^2(\Bk)\big(\sum_{\Bs\in\rho(\Bk)}\tilde{p}_\Bs\phi_\Bs(\cdot)\big)^2\big]^{\frac{1}{2}}\big\|_p^2>S^2\bigg)&\leq \pr\bigg(\bigg(\sum_{|\Bk|\leq l}\sum_{\Bs\in\rho(\Bk)}\zeta(\Bk)\big|\tilde{p}_\Bs-\hat{p}_\Bs\big|\bigg)^2>(S^2-E^2)/2\bigg)\nonumber\\
    &\leq \sum_{|\Bk|\leq l}\sum_{\Bs\in\rho(\Bk)}\pr\bigg(\big|\tilde{p}_\Bs-\hat{p}_\Bs\big|>\frac{\sqrt{(S^2-E^2)/2}}{\zeta(\Bk)2^{|\Bk|}{l\choose D}}\bigg)\label{eq:thm4_generr_probsum}
\end{align}
where the last line is from :
\begin{align*}
    &\sum_{\Bs\in\rho(\Bk)}1=2^{|\Bk|}\\
    &\sum_{|\Bk|\leq l}1=\sum_{i=D}^l{i-1\choose D-1}=\sum_{i=D-1}^{l-1}{i\choose D-1}={l \choose D}.
\end{align*}
Because for any $|\Bk|\leq l$  and $\Bs\in\rho(\Bk)$, $-1\leq \phi_\Bs(X_i)\leq 1$ for $i=1,\cdots, D$, $\{X_i\}_{i=1}^n$ are i.i.d. distributed and $n\asymp 2^{(1=2\alpha)l}/l^{2D}$, we can use Hoeffding's inequality \citep{geer2000empirical} to get
\[\pr\bigg(\big|\tilde{p}_\Bs-\hat{p}_\Bs\big|> \varepsilon\bigg)=\pr\bigg(\big|\E_{X\sim p}[\Phi_\Bs(X)]-\frac{1}{n}\sum\phi_\Bs(X_i)\big|\geq \varepsilon\bigg)\lesssim\exp\bigg\{-C_{D,\alpha}\frac{2^{(1+2\alpha)l}}{l^{2D}}\varepsilon\bigg\}\]
for any $\varepsilon>0$. Let $\varepsilon= \frac{\sqrt{(S^2-E^2)/2}}{\zeta(\Bk)2^{|\Bk|}{l\choose D}}$ for each $\Bs$, we can use the peeling technique to get:
\begin{align*}
        \pr\bigg(\big\|\big[\sum_{|\Bk|\leq l}\zeta^2(\Bk)\big(\sum_{\Bs\in\rho(\Bk)}\tilde{p}_\Bs\phi_\Bs(\cdot)\big)^2\big]^{\frac{1}{2}}\big\|_p^2>S^2\bigg)&\leq \sum_{|\Bk|\leq l}\sum_{\Bs\in\rho(\Bk)}\pr\bigg(\big|\tilde{p}_\Bs-\hat{p}_\Bs\big|>\frac{\sqrt{(S^2-E^2)/2}}{\zeta(\Bk)2^{|\Bk|}{l\choose D}}\bigg)\\
        &\leq \sum_{|\Bk|\leq l}2^{|\Bk|}\exp\bigg\{-C_{D,\alpha}\frac{2^{(1+2\alpha)l}}{l^{2D}}\frac{\sqrt{(S^2-E^2)/2}}{\zeta(\Bk)2^{|\Bk|}{l\choose D}}\bigg\}\\
        &\leq \sum_{|\Bk|\leq l}\exp\bigg\{-C_{D,\alpha}\frac{2^{(1+2\alpha)l}\sqrt{(S^2-E^2)/2}}{2^{|\Bk|}l^{2D}{l\choose D}\zeta(\Bk)}+|\Bk|\bigg\}.
\end{align*}
\end{proof}

In lemma \ref{lem:Hoeffding_AHCE}, we let 
\[\zeta(\Bk)=\sum_{j=1}^D2^{(\alpha-\frac{1}{2})k_j}\]
then for any $p\in\CalH^\alpha_{mix}(L_\alpha)$ with $\alpha>1/2$,
\begin{align*}
    E^2&=2\sum_{|\Bk|\leq l}\zeta^2(\Bk)2^{|\Bk|}\sum_{\Bs\in\rho(\Bk)}\hat{p}_\Bs^2\\
    &=2\sum_{|\Bk|\leq l}\big(\sum_{j=1}^D2^{(\alpha-\frac{1}{2})k_j}\big)^22^{|\Bk|}\sum_{\Bs\in\rho(\Bk)}\hat{p}_\Bs^2\\
    &\leq 2D^2 \sum_{|\Bk|\leq l} 2^{\alpha|\Bk|}\sum_{\Bs\in\rho(\Bk)}\hat{p}_\Bs^2\\
    &\leq 2D^2L_\alpha^2.
\end{align*}
Notice that, according to theorem \ref{lem:Sobolev_equivalent_norm}, for any $1<p<\infty$,
\[\|\tilde{p}_n\|_{\CalW^{\alpha-\frac{1}{2}}_p}\asymp \big\|\big[\sum_{|\Bk|\leq l}\big(\sum_{j=1}^D2^{(\alpha-\frac{1}{2})k_j}\big)^2\big(\sum_{\Bs\in\rho(\Bk)}\tilde{p}_\Bs\phi_\Bs(\cdot)\big)^2\big]^{\frac{1}{2}}\big\|_p=\big\|\big[\sum_{|\Bk|\leq l}\zeta^2(\Bk)\big(\sum_{\Bs\in\rho(\Bk)}\tilde{p}_\Bs\phi_\Bs(\cdot)\big)^2\big]^{\frac{1}{2}}\big\|_p \]
and, therefore
\begin{align*}
    \pr\bigg(\|\tilde{p}_n\|_{\CalW^{\alpha-\frac{1}{2}}_p}^2>S^2\bigg) &\leq \sum_{|\Bk|\leq l}\exp\bigg\{-C_{D,\alpha}\frac{2^{(1+2\alpha)l}\sqrt{(S^2-2D^2L_\alpha^2)/2}}{2^{|\Bk|}l^{2D}{l\choose D}\sum_{j=1}^D2^{(\alpha-\frac{1}{2})k_j}}+|\Bk|\bigg\}\\
    &\leq \sum_{|\Bk|\leq l}\exp\bigg\{-C_{D,\alpha}\frac{2^{(1+2\alpha)l}\sqrt{(S^2-2D^2L_\alpha^2)/2}}{2^{l}l^{2D}{l\choose D}D2^{(\alpha-\frac{1}{2})l}}+l\bigg\}\\
    &\leq {l\choose D}\exp\bigg\{-C'_{D,\alpha}\frac{2^{(\alpha+\frac{1}{2})l}\sqrt{(S^2-2D^2L_\alpha^2)/2}}{l^{2D}{l\choose D}}+l\bigg\}\\
    &=o(1)\quad \text{as}\ n\ (\text{or}\ l)\to\infty
\end{align*}
for any $S$ large, any $\alpha>\frac{1}{2}$ and any $1<p<\infty$. Because the above inequality holds independent of $p$, we also have
\[\pr\bigg(\|\tilde{p}_n\|_{\CalW^{\alpha-\frac{1}{2}}_\infty}^2>S^2\bigg)=o(1).\]
This shows that $\tilde{p}_n$ is H\"older-$(\alpha-\frac{1}{2})$ continuous independent of $n$. We then can use lemma \ref{lem:optimal_transport} to argue that there exist a Transformation $T_n$ such that $T_n\#\mu =\tilde{p}_n$ with $T_n\in\CalW^{\alpha+\frac{1}{2}}_\infty$ for any $1\leq n\leq \infty$.

The final step is to construct a generator $\Phi_g$  close to $T_n$ so that
\[\E\inf_{g\in \Phi_g} d_{\CalW^\beta_\infty}(g\#\mu,\tilde{p}_n)= \E \inf_{g\in \Phi_g} d_{\CalW^\beta_\infty}(g\#\mu,T_n\#\mu)\lesssim n^{-\frac{\alpha+\beta/D}{2\alpha+1}}\big[\log n\big]^{D\frac{\alpha+\beta/D+1}{2\alpha+1}}+n^{-\frac{1}{2}}. \]

For any $g\in\Phi_g$, let $g=(g_1,\cdots,g_D)$ and $T_n=(T_{n,1},\cdots,T_{n,D})$. Then for any $f\in\CalH^\alpha_{mix}(L_\alpha)$, 
\begin{align*}
    \E \inf_{g\in \Phi_g} d_{\CalW^\beta_\infty(L_\beta)}(g\#\mu,\tilde{p}_n)&=\E\inf_{g\in \Phi_g} d_{\CalW^\beta_\infty(L_\beta)}(g\#\mu,T_n\#\mu)\\
    &=\E^n\inf_{g\in\Phi_g}\sup_{f\in\CalW^\beta_\infty(L_\beta)}\E_{X\sim \mu}[f\circ g(X)]-\E_{Y\sim\mu}[f\circ T_n(Y)] \\
    &\leq L_\beta \E^n \inf_{g\in\Phi_g} \E_{X\sim \mu} \|g(X)-T_n(X)\|_\infty\\
    & \leq L_\beta \inf_{g\in\Phi_g}\max_{j=1,\cdot,D}\|g_j-T_{n,j}\|_\infty. 
\end{align*}
We then apply theorem \ref{thm:chen_efficient} and choose $\delta = n^{-\frac{\alpha+\beta/D}{2\alpha+1}}\big[\log n\big]^{D\frac{\alpha+\beta/D+1}{2\alpha+1}}+n^{-\frac{1}{2}}$ and Let
\begin{align*}
    H\gtrsim \bigg((\frac{\alpha+\beta/D}{2\alpha+1}+\frac{1}{2})\log n+1\bigg), W+S\gtrsim (n^{\frac{\alpha D+\beta}{2\alpha^2+2\alpha+1/2}}+n^{\frac{D}{2\alpha+1}}) H.
\end{align*}
Then we can construct $\Phi_g=(\Phi_1(H,W,S,B),\cdots, \Phi_D(H,W,S,B))$ that satisfies, for any $p\in\CalH^\alpha_{mix}$ with $\alpha>1/2$ and any $j=1,\cdots,D$,
\[\inf_{g\in\Phi_g}\max_{j=1,\cdots,D}\|g_j-T_{n,j}\|_\infty\leq n^{-\frac{\alpha+\beta/D}{2\alpha+1}}\big[\log n\big]^{D\frac{\alpha+\beta/D+1}{2\alpha+1}}+n^{-\frac{1}{2}}.\]
Therefore, we can finish the proof for theorem \ref{thm:GAN-improved}:
\[ \E \inf_{g\in \Phi_g} d_{\CalW^\beta_\infty(L_\beta)}(g\#\mu,\tilde{p}_n)\lesssim n^{-\frac{\alpha+\beta/D}{2\alpha+1}}\big[\log n\big]^{D\frac{\alpha+\beta/D+1}{2\alpha+1}}+n^{-\frac{1}{2}} .\]

\section{Proof of Theorem \ref{thm:smoothing}}
For notation simplicity, let $d_{F_d}(g,p)$ denote $d_{F_d}(g\# \mu,p)$ if $g$ is a mapping. Similar to the proof of theorem \ref{thm:GAN-improved}, we decompose the target error into different parts:
\begin{align*}
    d_{\CalW^\beta_\infty}(\hat{g}_n,p)\leq  d_{\CalW^\beta_\infty}(\hat{g}_n,\tilde{p}_n)+ d_{\CalW^\beta_\infty}(\tilde{p}_n,p)\lesssim d_{\CalW^\beta_\infty}(\hat{g}_n,\tilde{p}_n)+n^{-\frac{\alpha+\beta/D}{2\alpha+1}}\big[\log n\big]^{D\frac{\alpha+\beta/D+1}{2\alpha+1}}+n^{-\frac{1}{2}}
\end{align*}
 Then
\begin{align*}
    d_{\CalW^\beta_\infty}(\tilde{p}_n,\hat{g}_n)&=d_{{\Phi}_{d}}(\tilde{p}_n,\hat{g}_n)+\underbrace{d_{\CalW^\beta_\infty}(\tilde{p}_n,\hat{g}_n)-d_{{\Phi}_{d}}(\tilde{p}_n,\hat{g}_n)}_{A}\\
    &=\sup_{{f}\in{\Phi}_{d}}\E_{X\sim \mu}[{f}\circ \hat{g}_n(X)]-\frac{1}{n}\sum_{i=1}^n\tilde{f}_n(X_i)+A\\
    &\leq \sup_{{f}\in{\Phi}_{d}}\E_{X\sim \mu}[{f}\circ \hat{g}_n(X)]-\frac{1}{n}\sum_{i=1}^n{f}(X_i)+\underbrace{\sup_{f\in\Phi_d}\bigg(\frac{1}{n}\sum_{i=1}^n{f}(X_i)-\frac{1}{n}\sum_{i=1}^n\tilde{f}_n(X_i)\bigg)}_{B}+A\\
    &=\inf_{g\in\Phi_g}\sup_{{f}\in{\Phi}_{d}}\E_{X\sim \mu}[{f}\circ {g}(X)]-\frac{1}{n}\sum_{i=1}^n{f}(X_i)+A+B\\
    &\leq \inf_{g\in\Phi_g}\sup_{{f}\in{\Phi}_{d}}\E_{X\sim \mu}[{f}\circ {g}(X)]-\frac{1}{n}\sum_{i=1}^n\tilde{f}_n(X_i)+\underbrace{\sup_{f\in\Phi_d}\frac{1}{n}\sum_{i=1}^n\tilde{f}_n(X_i)-\frac{1}{n}\sum_{i=1}^n{f}(X_i)}_B+A+B\\
    &\leq \underbrace{\inf_{g\in\Phi_g}d_{\Phi_{d}}(g,\tilde{p}_n)}_C+A+2B
\end{align*}
For term $A$, we can use the same argument in section \ref{sec:improved_GAN_discriminator_err} and assumption (i) of $\Phi_d$ to get
\begin{align*}
   A\lesssim n^{-\frac{\alpha+\beta/D}{2\alpha+1}}\big[\log n\big]^{D\frac{\alpha+\beta/D+1}{2\alpha+1}}+n^{-\frac{1}{2}}.
\end{align*}
Let $\|f\|_n$ denote the empirical $L^2$ norm: $\|f\|_n^2=\frac{1}{n}\sum_{i=1}^nf(X_i)^2$. Then for term $B$, we have
\[\bigg|\frac{1}{n}\sum_{i=1}^n\tilde{f}_n(X_i)-f(X_i)\bigg|\leq \sqrt{\frac{1}{n}\sum_{i=1}^n\big|\tilde{f}_n(X_i)-f(X_i)\big|^2}=\|f-\tilde{f}_n\|_n.\]

According to assumption (iii), 
\[J=\inf_{\delta>0}\{K\int_{\delta/4}^1\sqrt{\mathcal{E}_\infty(K u/2,\Phi_d)}du+\sqrt{n}\delta K\}<\infty.\]
So we can apply theorem 2.1 in \cite{van2014uniform} to get
\begin{align*}
    \E \sup_{f\in\Phi_d}\bigg\{\|f-\tilde{f}_n\|_n-\|f-\tilde{f}_n\|_2 \bigg\}\leq \frac{2 JR}{\sqrt{n}}+\frac{4J^2}{n} \lesssim n^{-\frac{1}{2}}
\end{align*}
where $R=\sup_{f\in\Phi_d}\|f\|_{2}$. Then, we can use assumption (ii) to get
\[\E \sup_{f\in\Phi_d}\bigg|\frac{1}{n}\sum_{i=1}^n\tilde{f}_n(X_i)-f(X_i)\bigg|\leq\E \sup_{f\in\Phi_d}\|f-\tilde{f}_n\|_n\lesssim n^{-\frac{\alpha+\beta/D}{2\alpha+1}}\big[\log n\big]^{D\frac{\alpha+\beta/D+1}{2\alpha+1}}+n^{-\frac{1}{2}}. \]
For term $C$, we have:
\begin{align*}
    \inf_{g\in\Phi_g}d_{\Phi_d}(g,\tilde{p}_n)&=\inf_{g\in \Phi_g}\bigg\{d_{\Phi_d}(g,\tilde{p}_n)-d_{\CalW^\beta_\infty}(g,\tilde{p}_n)+d_{\CalW^\beta_\infty}(g,\tilde{p}_n)\bigg\}\\
    &\leq 2\sup_{f\in\CalW^\beta_\infty}\inf_{f'\in\Phi_d}\|f-f'\|_\infty+\inf_{g\in \Phi_g}d_{\CalW^\beta_\infty}(g,\tilde{p}_n)\\
    &=2A+\underbrace{\inf_{g\in \Phi_g}d_{\CalW^\beta_\infty}(g,\tilde{p}_n)}_D.
\end{align*}
Notice that term $A$ has been estimated and term $D$ is exactly the term estimated in section \ref{sec:improved_GAN_generator_err}. Therefore,
\[\inf_{g\in\Phi_g}d_{\Phi_d}(g,\tilde{p}_n) \lesssim n^{-\frac{\alpha+\beta/D}{2\alpha+1}}\big[\log n\big]^{D\frac{\alpha+\beta/D+1}{2\alpha+1}}+n^{-\frac{1}{2}}.\]

\section{Goodness of Fit Theorem}\label{apdix:fit_thm}

\subsection{Proof of Proposition 2}
This is a special case in Theorem 1 with $\beta=0$.
\subsection{Proof of Proposition 1}
According to the MISE upper bound in Proposition 2, we have $\lim_{n\rightarrow\infty}\mathbb{E}||\tilde{p}_{n}-p ||^2=0$. 

\subsection{Proof of Theorem \ref{thm:wchi}}
Under the null hypothesis $H_0 : p=p_{0}$,   $\tilde{H}_{n}$ are degenerate for all n, which means
\begin{align*}
\mathbb{E}_{X\sim p}\tilde{H}_{n}(X, y)=
\mathbb{E}_{X\sim p}H_{n}(X, y)-\mathbb{E}_{X,Y\sim p}H_{n}(X, Y)
-\mathbb{E}_{X\sim p}H_{n}(X, y)
+\mathbb{E}_{X,Y\sim p}H_{n}(X, Y)=0
\end{align*}
We define the integral operator $\mathcal{L}_{\tilde{H}_n}: L_{2}(p)\rightarrow\mathcal{F}$ satisfying 
$$\mathcal{L}_{\tilde{H}_n}f(x):=\int \tilde{H}_{n}(x, y)f(y)p(y)dy$$
Since $\mathbb{E}\tilde{H}_{n}(X,Y)^2<\infty$, therefore the integral operator $\mathcal{L}_{\tilde{H}_n}$ is Hilbert-Schmidt and compact. Now we can rewrite the kernel $\tilde{H}_{n}(x, y)$ in terms of its eigenfunctions $\psi_{l,n}(x)$ with respect the the probability measure under the null hypothesis, 
$$\tilde{H}_{n}(x,y)=\sum\limits_{l=1}^{\infty}\lambda_{l,n}\psi_{l,n}(x)\phi_{l,n}(y), $$ where 
\begin{align*}
\int \tilde{H}_{n}(x,y)\psi_{i,n}(x)p(x)dx=\lambda_{i,n}\psi_{i}(y)\\
\int \psi_{i,n}(x)\psi_{j,n}(x)p(x)dx=\delta_{i,j}
\end{align*}
Since we have $\mathbb{E}\tilde{H}_{n}^2(X,Y)<\infty$, and the operator is Hilbert-Schmidt, according to Theorem VI.22 in \cite{reed1980methods}, we have $\sum\lambda_{i,n}<\infty$. 
According to the degeneracy of the U-statistic, when $\lambda_{i}\neq 0,$ we have 
$$\lambda_{i,n}\mathbb{E}_{X}\psi_{i,n}(x)=\int\mathbb{E}\tilde{H}_{n}(x,y)\psi_{i,n}(y)p(y)dy=0.$$ Thereby, $\mathbb{E}_{X}\psi_{i,n}(X)=0$.
Now we are able to find the asymptotic distribution of our statistic $T_n$.
\begin{align*}
\frac{1}{n}\sum\limits_{i=1}^{n}\sum\limits_{i\neq j}^{n}{H}_{n}(X_i, X_j)&=\frac{1}{n}\sum\limits_{i=1}^{n}\sum\limits_{i\neq j}^{n}\sum\limits_{l=1}^{\infty}\lambda_{l,n}\psi_{l,n}(X_i)\psi_{l,n}(X_j)\\
&=\frac{1}{n}\sum\limits_{l=1}^{\infty}\lambda_{l,n}[(\sum\limits_{i})\psi_{l,n}(X_i))^2-\sum\limits_{i}\psi_{l,n}^{2}(X_i)]\\
&\rightarrow_{d}\sum\limits_{l=1}^{\infty}\lambda_{l,n}(b_{l}^2-1)
\end{align*}
where $b_{l}\sim\mathcal{N}(0,1)$ are i.i.d.. In addition, since $\lim_{n\rightarrow\infty}{H}_{n}(\cdot,\cdot)={H}(\cdot,\cdot),$ then we have $\lim_{n\rightarrow\infty}\sum\limits_{l=1}^{\infty}\lambda_{l,n}<\infty.$ Therefore, we can finished the proof for Theorem \ref{thm:wchi}.

\subsection{Proof of Theorem \ref{thm:asymptotic_normalty}}
In order to prove Theorem \ref{thm:asymptotic_normalty}, we need precise definition of wavelet system for mixed smooth Sobolev spaces. Moreover, we also need to impose some conditions on the wavelets adopted. The detailed introduction of wavelets and the required conditions are put in section \ref{sec:wavelet}. We first prove the following useful lemmas:
\begin{lem}\label{lem:2nd_moment}
Let $X_1,X_2$ be i.i.d. distributed random variables with $X_1\sim p$. Then
 \[\mathbb{E}_{X_1,X_2}\tilde{H}_{n}^2(X_1,X_2)]\gtrsim [\log n]^{\kappa}\]
 for some positive $\kappa$ independent of $n$. 
\end{lem}
\begin{proof}
According to the definition of $H_n$, we can immediately derive:
\begin{equation}\label{eq:H_n_square}
    \E H_{n}^2(X_1,X_2) =\sum_{|\Bk_1|\leq l}\sum_{\Bs_1\in\rho(\Bk_1)}\sum_{|\Bk_2|\leq l}\sum_{\Bs_2\in\rho(\Bk_2)}\big[\int \phi_{\Bk_1,\Bs_1}\phi_{\Bk_2,\Bs_2}p\big]^2
\end{equation}
where the level $l$ satisfies $2^l\asymp n^{\frac{1}{2\alpha}}[\log n]^{\frac{D(\alpha+1/2)}{2\alpha}}$ and the wavelet $\phi_{\Bk,\Bs}$ is defined as tensor product of one-dimensional wavelets
\begin{equation*}
    \phi_{\Bk,\Bs}=\prod_{j=1}^D\phi_{k_j,s_j},\quad \phi_{k,s}=\begin{cases}
    \phi\quad &\text{if}\ k=s=0\\
    2^{k/2}\psi(2^kx-s)\quad &\text{otherwise}
    \end{cases}
\end{equation*}
with $\phi$ the father wavelet and $\psi$ the mother wavelet. Without loss of generality, we can assume both $\phi$ and $\psi$ satisfy $\int_0^1|\phi|=\int_0^1|\psi| =1$. From the definition of $\phi_{\Bk,\Bs}$ in section \ref{sec:wavelet}, we can see that for any $\Bk_1\leq \Bk_2$ entry-wise, the supports of $\phi_{\Bk_1,\Bs_1}$ and $\phi_{\Bk_2,\Bs_2}$ are either disjoint or nested. If the former case holds, then $\int \phi_{\Bk_1,\Bs_1}\phi_{\Bk_2,\Bs_2}p=0$; if the later case holds, 
there are $\CalO(2^{|\Bk_2|-|\Bk_1|})$ number of basis among $\{\phi_{\Bk_2,\Bs_2}:\Bs_2\in\rho(\Bk_2)\}$ that satisfy
\begin{align*}
    \big[\int \phi_{\Bk_1,\Bs_1}\phi_{\Bk_2,\Bs_2}p\big]^2&\asymp  \phi_{\Bk_1,\Bs_1}^2(\Bs_2\cdot 2^{-\Bk_2-1})\big[\int \phi_{\Bk_2,\Bs_2}p\big]^2\asymp2^{|\Bk_1|}\check{p}_{\Bk_2,\Bs_2}^2
\end{align*}
where $\Bs\cdot 2^{-\Bk-1}\coloneqq (s_12^{-k_1-1},\cdots,s_D2^{-k_D-1})$ for any $\Bk,\Bs$. This is because both $\phi$ and $\psi$ are compactly supported functions and the support of any $\phi_{\Bk,\Bs}$ is a hyper-cube with side length  $\CalO(2^{-k_j})$ associated to the $j^{\rm th}$ dimension. 

To see why the above relation holds, we can use the Haar wavelet as an illustration (of course  we will not use Haar wavelets for constructing $H_n$, but they provide a clear illustration). In the Haar wavelet system, $\phi=1$ and $\psi=\mathbbm{1}_{\{x\in[0,1/2)\}}-\mathbbm{1}_{\{x\in[1/2,1)\}}$. 
\begin{figure}[ht!]
    \centering
    \includegraphics[width=0.3\textwidth]{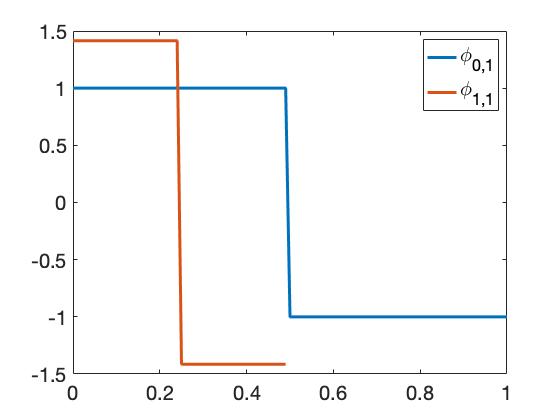}
    \caption{Plots of $\phi_{0,1}$ (blue curve) and $\phi_{1,1}$ (red curve)\label{fig:Haar}}
    \label{fig:my_label}
\end{figure}

From Figure \ref{fig:Haar}, we can see that $\int_0^1\phi_{0,1}\phi_{1,1} p=\phi_{0,1}(1/4)\int_0^1\phi_{1,1}p$ because $\phi_{0,1}$ is a constant on the support of $\phi_{1,1}$ and the constant is determined by the value of $\phi_{0,0}$ at $1\cdot2^{-1-1}=1/4$.

For any $\Bk,\Bs$, because we have assumed that the true density $p$ is lower bounded, we have
\[\check{p}_{\Bk,\Bs}^2=\big[\int\phi_{\Bk,\Bs}p\big]^2\gtrsim [\int\phi_{\Bk,\Bs}]^2 \asymp 2^{-|\Bk_2|}. \]
Because the indices in \eqref{eq:H_n_square}are symmetric, estimation of $\int\phi_{\Bk1,\Bs_1}\phi_{\Bk_2,\Bs_2}p$ leads to the following estimate
\begin{align}
    \sum_{|\Bk_1|\leq l}\sum_{\Bs_1\in\rho(\Bk_1)}\sum_{|\Bk_2|\leq l}\sum_{\Bs_2\in\rho(\Bk_2)}\big[\int \phi_{\Bk_1,\Bs_1}\phi_{\Bk_2,\Bs_2}p\big]^2&\asymp 2^D \sum_{|\Bk_1|\leq l}\sum_{\Bs_1\in\rho(\Bk_1)}\sum_{|\Bk_2|\leq l,\Bk_2\geq\Bk_1}\sum_{\Bs_2\in\rho(\Bk_2)}\big[\int \phi_{\Bk_1,\Bs_1}\phi_{\Bk_2,\Bs_2}p\big]^2\nonumber\\
    &\gtrsim \sum_{|\Bk_1|\leq l}\sum_{\Bs_1\in\rho(\Bk_1)}\sum_{|\Bk_2|\leq l,\Bk_2\geq\Bk_1}\sum_{\Bs_2\in\rho(\Bk_2)} \phi_{\Bk_1,\Bs_1}^2(\Bs_2\cdot 2^{-\Bk_2-1})\check{p}_{\Bk_2,\Bs_2}^2. \label{eq:H_n_symmetric}
\end{align}
We can also notice that for any fixed index $(\Bk2,\Bs2)$, and any fixed $\Bk1\leq \Bk_2$, there must be constantly many $\Bs_1\in\rho(\Bk_1)$ such that $\phi_{\Bk_1,\Bs_1}^2(\Bs_2\cdot 2^{-\Bk_2-1})\neq 0$ due to the nested and disjoint property of wavelets; similarly, for any fixed index $(\Bk_1,\Bs_1)$, and any fixed $\Bk_2\geq \Bk_1$, the number of $\Bs_2\in\rho(\Bk_2)$ that satisfies $\phi_{\Bk_1,\Bs_1}^2(\Bs_2\cdot 2^{-\Bk_2-1})\neq 0$ is in the same order as $2^{|\Bk_2|-|\Bk_1|}$.

As a result, \eqref{eq:H_n_symmetric} can be further estimated as
\begin{align*}
    \sum_{|\Bk_1|\leq l}\sum_{\Bs_1\in\rho(\Bk_1)}\sum_{|\Bk_2|\leq l,\Bk_2\geq\Bk_1}\sum_{\Bs_2\in\rho(\Bk_2)} \phi_{\Bk_1,\Bs_1}^2(\Bs_2\cdot 2^{-\Bk_2-1})\check{p}_{\Bk_2,\Bs_2}^2&\gtrsim  \sum_{|\Bk_1|\leq l}\sum_{|\Bk_2|\leq l,\Bk_2\geq\Bk_1} 2^{|\Bk_1|}2^{|\Bk_2|-|\Bk_1|} 2^{-|\Bk_2|}\asymp l^\kappa
\end{align*}
where $\kappa$ is some constant only depends on dimension $D$. Together with the fact that $2^l\asymp n^{\frac{1}{2\alpha}}[\log n]^{\frac{D(\alpha+1/2)}{2\alpha}}$, we have $\mathbb{E}_{X_1,X_2}H_{n}^2(X_1,X_2)\gtrsim [\log n]^{\kappa}$. Note that 
\begin{align}\label{eq:2m_central}
  \mathbb{E}\tilde{H}_{n}^2(X_1,X_2)=\mathbb{E}H^2_{n}(X_1,X_2)-2\mathbb{E}_{X_1}\big|\mathbb{E}_{X_2}[H_n(X_1,X_2)|X_1]\big|^2+[\mathbb{E}H_n(X_1,X_2)]^2.  
\end{align}
Therefore, we only need the check the order of the second term on the right-hand side of equation \eqref{eq:2m_central}. In particular, 
\begin{align*}
 \mathbb{E}_{X_1}\big|\mathbb{E}_{X_2}[H_n(X_1,X_2)|X_1]\big|^2=\mathbb{E}[\big|\tilde{p}(X_1)\big|^2]
\end{align*}
where $\tilde{p}$ is the AHCE. Because we have proved in section \ref{sec:improved_GAN_generator_err} that the H\"older norm $\|\tilde{p}\|_{\CalW^{\alpha-\frac{1}{2}}_\infty}$ is almost sure bounded independent of $n$ and we have assume that $\alpha>1$. Therefore, $\mathbb{E}[\big|\tilde{p}(X_1)\big|^2]$ is also bounded independent of $n$.
We can finish our proof.
\end{proof}

\begin{lem}\label{lem:3rd_moment}
Let $X_1,X_2$ and $X_3$ be i.i.d. distributed random variables with $X_1\sim p$.
  \[{\mathbb{E}_{X_1,X_2,X_3}\tilde{H}_{n}^2(X_1,X_2)\tilde{H}_{n}^2(X_1,X_3)}\lesssim n^{\frac{1}{2\alpha}}[\log n]^\kappa\]
  for some positive $\kappa$ independent of $n$.
\end{lem}
\begin{proof}
Similar to the proof of lemma \ref{lem:2nd_moment}, using the symmetry of indices we can rewrite ${\mathbb{E} {H}_{n}^2(X_1,X_2){H}_{n}^2(X_1,X_3)}$ as
\begin{align}
    &\quad {\mathbb{E}{H}_{n}^2(X_1,X_2){H}_{n}^2(X_1,X_3)}\nonumber\\
    &= \sum_{|\Bk_1|\leq l}\sum_{\Bs_1\in \rho(\Bk_1)}\cdots  \sum_{|\Bk_4|\leq l}\sum_{\Bs_4\in \rho(\Bk_4)}V_{(\Bk_1,\Bs_1,\Bk_2,\Bs_2)} K_{(\Bk_1,\Bs_1,\Bk_2,\Bs_2),(\Bk_3,\Bs_3,\Bk_4,\Bs_4)}V_{(\Bk_3,\Bs_3,\Bk_4,\Bs_4)}\nonumber\\
     &\lesssim \sum_{|\Bk_1|\leq l}\sum_{\Bs_1\in \rho(\Bk_1)}\sum_{\Bk_2\geq \Bk_1}\sum_{\Bs_2\in \rho(\Bk_2)}\sum_{\Bk_3\geq \Bk_2}\sum_{\Bs_3\in \rho(\Bk_3)}  \sum_{\Bk_4\geq \Bk_3}\sum_{\Bs_4\in \rho(\Bk_4)}V_{(\Bk_1,\Bs_1,\Bk_2,\Bs_2)} K_{(\Bk_1,\Bs_1,\Bk_2,\Bs_2),(\Bk_3,\Bs_3,\Bk_4,\Bs_4)}V_{(\Bk_3,\Bs_3,\Bk_4,\Bs_4)}\nonumber
\end{align}
where
\begin{align*}
    &V_{(\Bk_i,\Bs_i,\Bk_j,\Bs_j)}=\int \phi_{\Bk_i,\Bs_i} \phi_{\Bk_j,\Bs_j} p\\
    &K_{(\Bk_1,\Bs_1,\Bk_2,\Bs_2),(\Bk_3,\Bs_3,\Bk_4,\Bs_4)}=\int \phi_{\Bk_1,\Bs_1} \phi_{\Bk_2,\Bs_2}\phi_{\Bk_3,\Bs_3}\phi_{\Bk_4,\Bs_4} p.
\end{align*}
Similar to the proof of lemma \ref{lem:2nd_moment}, in the case $\Bk_4\geq \Bk_3\geq\Bk_2\geq \Bk_1$, we can use the disjoint and nested property of wavelets to derive that
\begin{align*}
    &\bigg|\int \phi_{\Bk_1,\Bs_1} \phi_{\Bk_2,\Bs_2} p\bigg|\lesssim \phi^*_{\Bk_1,\Bs_1}(\Bs_4\cdot 2^{-\Bk_4-1})\big|\check{p}_{\Bk_2,\Bs_2}\big|,\\
    &\bigg|\int \phi_{\Bk_3,\Bs_3} \phi_{\Bk_4,\Bs_4} p\bigg|\lesssim\phi^*_{\Bk_3,\Bs_3}(\Bs_4\cdot 2^{-\Bk_4-1})\big|\check{p}_{\Bk_4,\Bs_4}\big|,\\
    &\bigg|\int \phi_{\Bk_1,\Bs_1} \phi_{\Bk_2,\Bs_2}\phi_{\Bk_3,\Bs_3}\phi_{\Bk_4,\Bs_4} p\bigg|\lesssim\big|\check{p}_{\Bk_4,\Bs_4}\big|\prod_{j=1}^3\phi^*_{\Bk_j,\Bs_j}(\Bs_4\cdot 2^{-\Bk_4-1})
\end{align*}
where 
\[\phi^*_{\Bk,\Bs}(\Bx)=\mathbbm{1}_{\{\phi_{\Bk,\Bs}(\Bx)\neq 0\}}\max_{\bold{t}}|\phi_{\Bk,\Bs}(\bold{t})|\lesssim \mathbbm{1}_{\{\phi_{\Bk,\Bs}(\Bx)\neq 0\}} 2^{|\Bk|/2}.\]
 Because we have assume that $p\in\CalH^\alpha_{mix}$ and $p<\infty$ (which is implied by $p\in\CalH^\alpha_{mix}$ ), so
\[\sum_{\Bs\in\rho(\Bk)}\check{p}_{\Bk,\Bs}^2\lesssim 2^{2(1/2-\alpha)|\Bk|},\quad \check{p}_{\Bk,\Bs}^2\leq 2^{-|\Bk|}\]
for any $\Bk$ and $\Bs$. 

For any fixed index $(\Bk_4,\Bs_4)$, there are only constantly many $ \phi_{\Bk_1,\Bs_1}$, $ \phi_{\Bk_2,\Bs_2}$ and $ \phi_{\Bk_3,\Bs_3}$ that are active at $\Bs_4\cdot 2^{-\Bk_4-1}$; for any fixed index $(\Bk_3,\Bs_3)$, the number of indices $(\Bk_4,\Bs_4)$ that satisfy $\phi_{\Bk3,\Bs3}\neq 0$ is  $\CalO(2^{|\Bk4|-|\Bk3|})$. 
Therefore
\begin{align}
    &\quad \sum_{|\Bk_1|\leq l}\sum_{\Bs_1\in \rho(\Bk_1)}\sum_{\Bk_2\geq \Bk_1}\sum_{\Bs_2\in \rho(\Bk_2)}\sum_{\Bk_3\geq \Bk_2}\sum_{\Bs_3\in \rho(\Bk_3)}  \sum_{\Bk_4\geq \Bk_3}\sum_{\Bs_4\in \rho(\Bk_4)}V_{(\Bk_1,\Bs_1,\Bk_2,\Bs_2)} K_{(\Bk_1,\Bs_1,\Bk_2,\Bs_2),(\Bk_3,\Bs_3,\Bk_4,\Bs_4)}V_{(\Bk_3,\Bs_3,\Bk_4,\Bs_4)}\nonumber\\
    &\lesssim \sum_{|\Bk_1|\leq l}\sum_{|\Bk_2|\leq l}\sum_{|\Bk_3|\leq l}\sum_{|\Bk_4|\leq l} 2^{|\Bk_1|} 2^{|\Bk_3|} 2^{|\Bk4|-|\Bk3|} 2^{2(1/2-\alpha)|\Bk_4|}\nonumber\\
    &\lesssim \sum_{|\Bk_1|\leq l}2^{|\Bk_1|}l^{\kappa'}\nonumber\\
    &\lesssim 2^{l}l^\kappa 
\end{align}
where $\kappa'$ and $\kappa$ are some constants only depends on dimension $D$. Together with the fact that $2^l\asymp n^{\frac{1}{2\alpha}}[\log n]^{\frac{D(\alpha+1/2)}{2\alpha}}$, and  $\mathbb{E}_{X_1,X_2,X_3}\tilde{H}_{n}^2(X_1,X_2)\tilde{H}_{n}^2(X_1,X_3)\lesssim {\mathbb{E}_{X_1,X_2,X_3}{H}_{n}^2(X_1,X_2){H}_{n}^2(X_1,X_3)}$, we can derive the final result.
\end{proof}
\begin{lem}
  \label{lem:4th_moment}
  Let $X_1,X_2$ be i.i.d. distributed random variables with $X_1\sim p$. Then
 \[\mathbb{E}_{X_1,X_2}\tilde{H}_{n}^4(X_1,X_2)]\lesssim n^{\frac{1}{\alpha}} [\log n]^{\kappa}\]
 for some positive $\kappa$ independent of $n$. 
\end{lem}
\begin{proof}
Similar to the proof of lemma \ref{lem:3rd_moment}, $\mathbb{E}{H}_{n}^4(X_1,X_2)]$ can be expanded as
\begin{align*}
   \mathbb{E}\tilde{H}_{n}^4(X_1,X_2)]\lesssim \mathbb{E}{H}_{n}^4(X_1,X_2)] &=\sum_{|\Bk_1|\leq l}\sum_{\Bs_1\in\rho(\Bk_1)}\cdots \sum_{|\Bk_4|\leq l}\sum_{\Bs_4\in\rho(\Bk_4)} \big[\int \prod_{j=1}^4\phi_{\Bk_j,\Bs_j}p\big]^2\\
    & \lesssim \sum_{|\Bk_1|\leq l}\sum_{\Bs_1\in\rho(\Bk_1)}\sum_{\Bk_2\geq \Bk_1}\sum_{\Bs_2\in\rho(\Bk_2)}\cdots \sum_{\Bk_4\geq \Bk_3}\sum_{\Bs_4\in\rho(\Bk_4)} \big[\int \prod_{j=1}^4\phi_{\Bk_j,\Bs_j}p\big]^2\\
    &\lesssim \sum_{|\Bk_1|\leq l}\sum_{\Bs_1\in\rho(\Bk_1)}\sum_{\Bk_2\geq \Bk_1}\sum_{\Bs_2\in\rho(\Bk_2)}\cdots \sum_{\Bk_4\geq \Bk_3}\sum_{\Bs_4\in\rho(\Bk_4)} \check{p}^2_{\Bk_4,\Bs_4}\bigg[\prod_{j=1}^3 \phi^*_{\Bk_j,\Bs_j}(\Bs_4\cdot 2^{-\Bk_4-1})\bigg]^2\\
    &\lesssim \sum_{|\Bk_1|\leq l}\sum_{\Bk_2\geq \Bk_1}\sum_{\Bk_3\geq \Bk_2} \sum_{\Bk_4\geq \Bk_3} 2^{|\Bk_1|+|\Bk_2|+|\Bk_3|} 2^{|\Bk_4|-|\Bk_3|} 2^{2(1/2-\alpha)|\Bk_4|}\\
    &\lesssim \sum_{|\Bk_1|\leq l}\sum_{|\Bk_2|\leq l}2^{|\Bk_1|+|\Bk_2|} l^{\kappa'} \lesssim 2^{2l}l^\kappa
\end{align*}
where the second line is from symmetry of indices; the third line is from the nested and disjoint property of $\phi_{\Bk,\Bs}$; the fourth line is from argument exactly the same as the proof of lemma \ref{lem:3rd_moment}; the last line is because $2^{|\Bk_3|} 2^{|\Bk_4|-|\Bk_3|} 2^{2(1/2-\alpha)|\Bk_4|}\lesssim 1$ so the sum over $\Bk_3$ and $\Bk_4$ must be bounded by $l^{\kappa'}$  for some $\kappa'$ independent of $l$. Together with the fact that $2^l\asymp n^{\frac{1}{2\alpha}}[\log n]^{\frac{D(\alpha+1/2)}{2\alpha}}$, we can derive the final result.
\end{proof}
\begin{lem}\label{lem:central_moment}
Let $X_1,X_2$ and $X_3$ be i.i.d. distributed random variables with $X_1\sim p$.
  \[{\mathbb{E}_{X_1,X_3}\big|\E_{X_2}\tilde{H}_{n}(X_1,X_2)\tilde{H}_{n}(X_2,X_3)}\big|^2\lesssim 1.\]
\end{lem}
\begin{proof}
Similar to the proof of lemma \ref{lem:3rd_moment}, 
$$ {\mathbb{E}_{X_1,X_3}\big|\E_{X_2}\tilde{H}_{n}(X_1,X_2)\tilde{H}_{n}(X_2,X_3)}\big|^2\lesssim \mathbb{E}_{X_1,X_3}\big|\E_{X_2}H_{n}(X_1,X_2)H_{n}(X_2,X_3)\big|^2.$$
According to the definition of the kernel function,  $\mathbb{E}_{X_1,X_3}\big|\E_{X_2}H_{n}(X_1,X_2)H_{n}(X_2,X_3)\big|^2$ can expanded as
\begin{align*}
    &\quad {\mathbb{E}_{X_1,X_3}\big|\E_{X_2}{H}_{n}(X_1,X_2){H}_{n}(X_2,X_3)}\big|^2\\
    &=\sum_{|\Bk_1|\leq l}\sum_{\Bs_1\in\rho(\Bk_1)}\cdots \sum_{|\Bk_4|\leq l}\sum_{\Bs_4\in\rho(\Bk_4)}V_{(\Bk_1,\Bs_1,\Bk_2,\Bs_2)}V_{(\Bk_3,\Bs_3,\Bk_4,\Bs_4)}V_{(\Bk_1,\Bs_1,\Bk_3,\Bs_3)}V_{(\Bk_2,\Bs_2,\Bk_4,\Bs_4)}\\
    &\lesssim \sum_{|\Bk_1|\leq l}\sum_{\Bs_1\in\rho(\Bk_1)}\sum_{\Bk_2\geq \Bk_1}\sum_{\Bs_2\in\rho(\Bk_2)}\cdots \sum_{\Bk_4\geq \Bk_3}\sum_{\Bs_4\in\rho(\Bk_4)}V_{(\Bk_1,\Bs_1,\Bk_2,\Bs_2)}V_{(\Bk_3,\Bs_3,\Bk_4,\Bs_4)}V_{(\Bk_1,\Bs_1,\Bk_3,\Bs_3)}V_{(\Bk_2,\Bs_2,\Bk_4,\Bs_4)}
\end{align*}
where 
\[V_{(\Bk_i,\Bs_i,\Bk_j,\Bs_j)}=\int \phi_{\Bk_i,\Bs_i} \phi_{\Bk_j,\Bs_j} p.\]
Similar to the proof of lemma \ref{lem:3rd_moment}, argument using the nested and disjoint property of $\phi_{\Bk,\Bs}$ shows that when $\Bk_4\geq \Bk_3\geq \Bk_2\geq \Bk_1$, we have
\begin{align*}
    &\quad V_{(\Bk_1,\Bs_1,\Bk_2,\Bs_2)}V_{(\Bk_3,\Bs_3,\Bk_4,\Bs_4)}V_{(\Bk_1,\Bs_1,\Bk_3,\Bs_3)}V_{(\Bk_2,\Bs_2,\Bk_4,\Bs_4)}\\
    &\lesssim |\phi^*_{\Bk_1,\Bs_1}(\Bs_4\cdot 2^{-\Bk4-1})|^2\check{p}_{\Bk_2,\Bs_2}\phi^*_{\Bk_3,\Bs_3}(\Bs_4\cdot 2^{-\Bk4-1})\check{p}^2_{\Bk_4,\Bs_4} \check{p}_{\Bk_3,\Bs_3}\phi^*_{\Bk_2,\Bs_2}(\Bs_4\cdot 2^{-\Bk4-1}).
\end{align*}
Because $p$ is upper bounded,
\[\big|\check{p}_{\Bk,\Bs}\big|\lesssim 2^{-|\Bk|/2}\]
for any $\Bk$ and $\Bs$.  Therefore,
\[|\phi^*_{\Bk_1,\Bs_1}(\Bs_4\cdot 2^{-\Bk_4-1})|^2\check{p}_{\Bk_2,\Bs_2}\phi^*_{\Bk_3,\Bs_3}(\Bs_4\cdot 2^{-\Bk_4-1})\check{p}^2_{\Bk_4,\Bs_4} \check{p}_{\Bk_3,\Bs_3}\phi^*_{\Bk_2,\Bs_2}(\Bs_4\cdot 2^{-\Bk_4-1})\lesssim 2^{|\Bk_1|}\check{p}^2_{\Bk_4,\Bs_4}\]
if it is non-zero.

The above estimation and argument similar to the proof in lemma \ref{lem:3rd_moment} shows
\begin{align*}
    &\quad \sum_{|\Bk_1|\leq l}\sum_{\Bs_1\in\rho(\Bk_1)}\sum_{\Bk_2\geq \Bk_1}\sum_{\Bs_2\in\rho(\Bk_2)}\cdots \sum_{\Bk_4\geq \Bk_3}\sum_{\Bs_4\in\rho(\Bk_4)}V_{(\Bk_1,\Bs_1,\Bk_2,\Bs_2)}V_{(\Bk_3,\Bs_3,\Bk_4,\Bs_4)}V_{(\Bk_1,\Bs_1,\Bk_3,\Bs_3)}V_{(\Bk_2,\Bs_2,\Bk_4,\Bs_4)}\\
    &\lesssim  \sum_{|\Bk_1|\leq l}2^{|\Bk_1|}\sum_{\Bk_2\geq \Bk_1}\sum_{\Bk_3\geq \Bk_2} \sum_{\Bk_4\geq \Bk_3} 2^{|\Bk_4|-|\Bk_3|}2^{2(1/2-\alpha)|\Bk_4|}\\
    &= \sum_{|\Bk_1|\leq l}2^{|\Bk_1|}\sum_{\Bk_2\geq \Bk_1}\sum_{\Bk_3\geq \Bk_2}2^{-|\Bk_3|}\sum_{\Bk_4\geq \Bk_3} 2^{(2-2\alpha)|\Bk_4|}.
\end{align*}
Because we have assumed that $\alpha>1$, so the summation is finite for any $l$:
\begin{align*}
    \sum_{|\Bk_1|\leq l}2^{|\Bk_1|}\sum_{\Bk_2\geq \Bk_1}\sum_{\Bk_3\geq \Bk_2}2^{-|\Bk_3|}\sum_{\Bk_4\geq \Bk_3} 2^{(2-2\alpha)|\Bk_4|}\lesssim \sum_{|\Bk_1|\leq l}2^{|\Bk_1|}  2 ^{(1-2\alpha)|\Bk_1|}l^\kappa=\sum_{|\Bk_1|\leq l} 2 ^{(2-2\alpha)|\Bk_1|}l^\kappa\lesssim 1
\end{align*}
where $\kappa$ is some constant independent of $l$. We can finish the proof. 
\end{proof}

Now we are ready to prove theorem \ref{thm:asymptotic_normalty}. Aware that $T_{n}$ is a U-statistic, the general techniques for U-statistics can be applied to establish its asymptotic normality. Specially, according to the Theorem 1 in \cite{hall1984central}), it suffices to verify the following conditions:
\begin{equation}
    \label{eq:test_norm}
\begin{aligned}
    \frac{\mathbb{E}\tilde{H}_{n}^4(X_1,X_2)}{n^2[\mathbb{E}\tilde{H}_{n}^2(X_1,X_2)]^2}\rightarrow 0\\
     \frac{\mathbb{E}\tilde{H}_{n}^2(X_1,X_2)\tilde{H}_{n}^2(X_1,X_3)}{n[\mathbb{E}\tilde{H}_{n}^2(X_1,X_2)]^2}\rightarrow 0\\
      \frac{\mathbb{E}G_{n}^2(X_1,X_2)}{[\mathbb{E}\tilde{H}_{n}^2(X_1,X_2)]^2}\rightarrow 0
\end{aligned}
\end{equation}
as $n\rightarrow\infty$, where 
$$G_{n}(x,y)=\mathbb{E}\tilde{H}_{n}(x,X_3)\tilde{H}_{n}(y, X_3),
\ \ \ \ \forall x, y\in \mathbb{R}^{D}$$
From lemma \ref{lem:2nd_moment}, \ref{lem:3rd_moment}, \ref{lem:4th_moment} and \ref{lem:central_moment}, we can immediately derive that any condition in  \eqref{eq:test_norm} holds. Therefore, according to Theorem 1 in \cite{hall1984central}, we have 
$$\frac{nT_{n}}{\sqrt{2\mathbb{E}[\tilde{H}_{n}(X_1,X_2)]^2}}\rightarrow_{d}\mathcal{N}(0,1).$$

Utilizing the degeneracy of our U-statistic $T_n$ under $H_0$, we have 
\begin{align}\label{eq:var_statisc}
\begin{split}
    var(T_n)&=\frac{2}{n(n-1)}\mathbb{E}[\tilde{H}_{n}(X_1,X_2)]^2\\
    &=\frac{2}{n(n-1)}\{\mathbb{E}[H_{n}(X_1,X_2)]^2-2\mathbb{E}[H_{n}(X_1,X_2)H_{n}(X_1,X_3)]+[\mathbb{E}H_{n}(X_1,X_2)]^2\}
\end{split}
\end{align}
 In order to estimate $var(T_{n})$, we use three U-statistics to estimate each of the three terms in \ref{eq:var_statisc} :

\begin{align}
\begin{split}
  \hat{\sigma}_{n}^2=&\frac{1}{n(n-1)}\sum\limits_{1\leq i\neq j\leq n}H_{n}^{2}(X_{i}, X_{j})\\
 &-\frac{2(n-3)!}{n!}\sum\limits_{\substack{1\leq i, j_1, j_2\leq n \\ |\{i, j_1, j_2\}|=3}}H_{n}(X_{i}, X_{j_{1}})H_{n}(X_{i}, X_{j_{2}})\\
&+\frac{(n-4)!}{n!}\sum\limits_{\substack{1\leq i_1, i_2, j_1, j_2\leq n \\ |\{i_1, i_2, j_1, j_2\}|=4}}H_{n}(X_{i_{1}}, X_{j_{1}})H_{n}(X_{i_{2}}, X_{j_{2}}).   
\end{split}
\end{align}

Since we have shown that 
$$\frac{nT_{n}}{\sqrt{2\mathbb{E}[\tilde{H}_{n}(X_1,X_2)]^2}}\rightarrow_{d}\mathcal{N}(0,1).$$
By Slutsky Theorem, in order to prove $\frac{nT_n}{\sqrt{2}\hat{\sigma}_n}\rightarrow_{d}\mathcal{N}(0,1),$ it is suffices to show 
\begin{align*}
   \frac{\hat{\sigma}^2_n}{\mathbb{E}[\tilde{H}_{n}(X_1,X_2)]^2}\rightarrow_{p} 1
\end{align*}
which is equivalent to show 
\begin{align}
    \mathbb{E}[\hat{\sigma}_{n}^2]=\mathbb{E}[\tilde{H}_{n}(X_1,X_2)]^2\\
   \lim_{n\rightarrow\infty}\frac{var(\hat{\sigma}^2_{n})}{[\mathbb{E}\tilde{H}_{n}^2(X_1,X_2)]^2} =0.
\end{align}
Since 
\begin{align*}
    var(\hat{\sigma}_{n}^2)\lesssim &\underbrace{n^{-4}var\Big(\sum\limits_{1\leq i\neq j\leq n}H_{n}^2(X_i,X_j)\Big)}_{\textcircled{1}}+\underbrace{n^{-6}var\Bigg(\sum\limits_{\substack{1\leq i, j_1, j_2\leq n \\ |\{i, j_1, j_2\}|=3}}H_{n}(X_{i},X_{j_{1}})H_{n}(X_{i},X_{j_{2}})\Bigg)}_{\textcircled{2}}\\
    &\underbrace{n^{-8}var\Bigg(\sum\limits_{\substack{1\leq i_1, i_2, j_1, j_2\leq n \\ |\{i_1, i_2, j_1, j_2\}|=4}}H_{n}(X_{i_{1}}, X_{j_{1}})H_{n}(X_{i_{2}}, X_{j_{2}})\Bigg)}_{\textcircled{3}}\\
  % \lesssim \frac{C_n^2}{n^4} \mathbb{E} H^{4}_{n}(X_1,X_2)+
 \end{align*}
Through simple calculation, it is easy to see that $\textcircled{1}\lesssim  n^{-2}\mathbb{E} H^{4}_{n}(X_1,X_2).$ 

Moreover, 

\begin{align*}
\textcircled{2}&\lesssim n^{-6}\Bigg( C_{n}^{4}\mathbb{E}H_{n}(X_1, X_2)H_{n}(X_1, X_3)H_{n}(X_2, X_4)H_{n}(X_3, X_4)
\\
&\quad+C_{n}^{5}\mathbb{E}H_{n}(X_1, X_2)H_{n}(X_1, X_3)H_{n}(X_{3},X_{4})H_{n}(X_{4},X_{5})\\
&\quad+C_{n}^{6}\mathbb{E}H_{n}(X_1,X_2)H_{n}(X_1,X_3)H_{n}(X_{4},X_{5})H_{n}(X_4,X_6)-[C_{n}^3]^2[\mathbb{E}H_{n}(X_1, X_2)H_{n}(X_1,X_3)]^2\Bigg)\\
&\lesssim n^{-6}\Bigg(C_{n}^{4}\sqrt{\mathbb{E}H_{n}^2(X_1,X_2)H_{n}^2(X_1,X_3)\mathbb{E}H_{n}^2(X_2,X_4)H_{n}^2(X_3,X_4)})\\
&\quad+C_{n}^{5}\sqrt{\mathbb{E}H_{n}^2(X_1,X_2)H_{n}^2(X_1, X_3)\mathbb{E}H_{n}^2(X_3,X_4)H_{n}^2(X_4,X_5)}\Bigg)\\
&\lesssim n^{-2}\mathbb{E}H_{n}^2(X_1,X_2)H_{n}^2(X_1,X_3)+n^{-1}\mathbb{E}H_{n}^2(X_1,X_2)H_{n}^2(X_1,X_3)
\end{align*}

similarly, we have 

\begin{align*}
\textcircled{3}&\lesssim n^{-8}\Bigg( C_{n}^{4}\mathbb{E}H_{n}^2(X_1, X_2)H_{n}^2(X_3, X_4)+C_{n}^{5}\mathbb{E}H_{n}(X_1,X_2)H_n(X_3,X_4)H_{n}(X_1,X_2)H_n(X_3,X_5)\\
&\quad+C_{n}^6\mathbb{E}H_n(X_1,X_2)H_n(X_3,X_4)H_n(X_1,X_2)H_n(X_5,X_6)\\
&\quad+C_n^7\mathbb{E}H_{n}(X_1,X_2)H_n(X_3,X_4)H_n(X_1,X_5)H_n(X_6,X_7)\\
&\quad+C_n^8[\mathbb{E}H_n(X_1,X_2)]^4-[C_n^4]^2[\mathbb{E}H_n(X_1,X_2)]^4\Bigg)\\
&\lesssim n^{-8}\Bigg([C_n^4+C_n^5+C_n^6+C_n^7][\mathbb{E}H_n^2(X_1,X_2)]^2\Bigg)
\end{align*}

Therefore, we can bound $var(\hat{\sigma}^2_{n})$ by the following term
\begin{align*}
 var(\hat{\sigma}^2_{n}) \lesssim & n^{-2}\mathbb{E} H^{4}_{n}(X_1,X_2)+[n^{-1}+n^{-2}]\mathbb{E}H^{2}_{n}(X_1,X_2)H^{2}_{n}(X_1,X_3)\\
 & +[n^{-1}+n^{-2}+n^{-3}+n^{-4}][\mathbb{E}H_{n}^2(X_1,X_2)]^2   
\end{align*}

Because we have proved that \ref{eq:test_norm}, therefore, it is straightforward to see 
$$var(\hat{\sigma}^2_{n})=o((\mathbb{E}[\tilde{H}_{n}(X_1,X_2)]^2)^2).$$
Moreover, simple calculation yields $ \mathbb{E}[\hat{\sigma}_{n}^2]=\mathbb{E}[\tilde{H}_{n}(X_1,X_2)]^2$. Now we can complete the proof.

\section{Wavelet System in Mixed Smooth Besov Spaces}\label{sec:wavelet}
Detailed introduction of mixed smooth Besov spaces can be found at \cite{Dung16} and references therein. A mixed smooth Besov space for functions on $U\subseteq\Real^D$, denoted as $\mathcal{MB}^r_{p,q}(U)$, is characterized as the tensor product of one-dimensional Besov spaces $\CalB^r_{p,q}(U)$ as
\[\CalMB^r_{p,q}(U)=\bigotimes_{j=1}^D\CalB^r_{p,q}(U).\]
 
When $p=q=2$, $\CalMB_{2,2}^r$ becomes the mixed  smooth Sobolev space $\CalH^2_{mix}$. In the following subsections, we first introduction the wavelet system for constructing one-dimensional Besov spaces. Then, based on the one-dimensional wavelet system, we can introduce the so-called \emph{hyperbolic cross wavelet}, which is tensor product of one-dimensional wavelets,  and how to use hyperbolic cross wavelets to define mixed smooth Besov norms.
\subsection{One-Dimensional Wavelets}
In one dimension, Besov spaces model functions via a \emph{multi-resolution approximation} (MRA) or dyadic approximation in some literatures. Specifically, an MRA is defined as follows:
\begin{defn}[MRA]
\label{dfn:MRA}
An MRA of $L^2(\Real)$ is an incrasing sequence of spaces $\{V_k\}_{k\in\NatInt}$ with the following properties:
\begin{enumerate}
    \item $\bigcap_{k\in\NatInt}V_k=\{0\}$ and the closure $\overline{\bigcup_{k\in\NatInt}V_k}=L^2(\Real^D)$;
    \item any $f\in V_0$, $f(x-s)\in V_0$ for any $s\in\NatInt$ and any $f\in V_k$, $f(2^{k'}\cdot)\in V_{k+k'}$;
    \item There exist function $\phi$ (father wavelet) such that $\{\phi(x-s):s\in\mathbb{Z}\}$ is an orthogonal basis of $V_0$.
\end{enumerate}
\end{defn}
Under an MRA, it has been proved that there exists an multi-scale function (mother wavelet) $\psi$ such that the set
$\big\{2^{k/2}\psi(2^k x-s): s\in\mathbb{Z}\big\}$ is an orthogonal basis for $V_k$. In this paper, we only consider the case $U=[0,1]$ so we only need system of the form $V_0=\{\phi\}$ and $V_k=\{\psi(2^kx-s): s\in\rho(k)\}$ where supports of $\phi$ and $\psi$ are all subsets of $[0,1]$. A typical example of wavelet system satisfying definition \ref{dfn:MRA} for functions defined on $[0,1]$ is the \emph{Haar} wavelet system whose father and mother wavelet are defined as follows:
\[\phi(x)=1,\quad \psi(x)=\mathbbm{1}_{x\in [0,1/2)}-\mathbbm{1}_{x\in [1/2,1)}.\]
For notation convenience, we can define MRA of $L^2([0,1])$ as $V_k=\{\phi_{k,s}:s\in\rho(k)\}$ for $k\in\NatInt$ with $\phi_{0,0}=\phi$, $\phi_{0,1}=\psi$ and $\phi_{k,s}(x)=2^{k/2}\psi(2^kx-s)$. With MRA, we are ready to define Besov spaces:
\begin{defn}[Besov Space]
Suppose $p,q\geq 1$. Given an MRA $\{V_k\}_{k\in\NatInt}$ of $L^2([0,1])$, the Besov space $\CalB^r_{p,q}$ is defined as the set of functions $f:[0,1]\to \Real$ such that
\[\|f\|_{\CalB^r_{p,q}}\coloneqq\bigg(\sum_{k\in\NatInt}2^{qk(r+1/2-1/q)}\big(\sum_{s\in\rho(k)}\check{f}_{k,s}^p\big)^{q/p}\bigg)^{1/q}<\infty\]
where $\check{f}_{k,s}=\int_{0}^1 f(x)\phi_{k,s}(x)dx$ and $\|f\|_{\CalB^r_{p,q}}$ is called the Besov norm of $f$.  
\end{defn}
\subsection{Hyperbolic Cross Wavelets}
In a straightforward way, we define the tensor product of one-dimensional wavelets as
\[\phi_{\Bk,\Bs}=\prod_{j=1}^D\phi_{k_j,s_j}\]
for any $\Bk\in\NatInt^D$, $\Bs\in\rho(\Bk)$ and $\phi_{k_j,s_j}$ is in the one-dimensional MRA $V_{k_j}$ corresponding to the $j^{\rm th}$ dimension. Obviously, the tensor product of $L^2$ spaces on $[0,1]$ is  isometrically identified with the space $L^2([0,1]^D)$ and the support of $\phi_{\Bk,\Bs}$ is:
\[\text{supt}[\phi_{\Bk,\Bs}]=\bigtimes_{j=1}^d I_{k_j,s_j}\]
with
\begin{equation*}
    I_{k_j,s_j}=\begin{cases}
    [(s_j-1)2^{-k_j-1},(s_j+1)2^{-k_j-1} ]\quad &\text{if}\ k_j,s_j\neq 0\\
    [0,1]\quad & \text{if}\ k_j=s_j =0
    \end{cases}.
\end{equation*}
Therefore, the following sets
\[\{V_{\Bk}\}_{\Bk\in\NatInt^D}=\big\{\phi_{\Bk,\Bs}=\prod_{j=1}^D\phi_{k_j,s_j}: \phi_{k_j,s_j}\in V_{k_j},j=1,\cdots,D\big\}_{\Bk\in\NatInt^D}\] 
is a MRA of $L^2([0,1]^D)$ with index set $\NatInt^D$ rather than $\NatInt$ (changing the index set does not change any property of an MRA). According to property 1 of MRA, we then get basis $\mathcal{F}=\big\{\phi_{\Bk,\Bs}: \Bk\in\NatInt^D,\Bs\in\rho(\Bk)\}$ for $L^2([0,1]^D)$. In order to distinguish the concept of mixed smooth Besov spaces from isotropic Besov space, we call $\{V_{\Bk}\}_{\Bk\in\NatInt^D}$ hyperbolic cross MRA and the set $\mathcal{F}$ hyperbolic cross wavelets. With hyperbolic cross MRA, a mixed smooth Besov space is then defined as:

\begin{defn}[Mixed Smooth Besov Space]
Suppose $p,q\geq 1$. Given a hyperbolic cross MRA $\{V_\Bk\}_{\Bk\in\NatInt}$ of $L^2([0,1]^D)$, the mixed smooth Besov space $\CalMB^r_{p,q}$ is defined as the set of functions $f:[0,1]^D\to \Real$ such that
\[\|f\|_{\CalMB^r_{p,q}}\coloneqq\bigg(\sum_{\Bk\in\NatInt^D}2^{q|\Bk|(r+1/2-1/q)}\big(\sum_{\Bs\in\rho(\Bk)}\check{f}_{\Bk,\Bs}^p\big)^{q/p}\bigg)^{1/q}<\infty\]
where $\check{f}_{\Bk,\Bs}=\int_{[0,1]^d} f(\Bx)\phi_{\Bk,\Bs}(\Bx)d\Bx$ and $\|f\|_{\CalMB^r_{p,q}}$ is called the mixed Besov norm of $f$.  
\end{defn}

\subsection{Wavelet Regularity Conditions in Constructing Density Estimators}
We have introduced the concept of mixed smooth Besov space. On the other hand, given a function $f\in\CalMB^r_{p,q}$, we also want to approximate $f$ by finitely many wavelets so that the approximation and $f$ is close under some metric distance. This approximation requires some regularity conditions on the wavelets adopted. Therefore, we fist  introduce the following \emph{regularity} conditions on wavelets.

\begin{defn}
A one-dimensional mother wavelet $\psi^{(m)}$ is said to be $m$-regular if it satisfies the following conditions:
\begin{enumerate}
    \item  $\psi^{(m)}$ restricted to interval $[\frac{k}{2},\frac{k+1}{2}]$, $k\in\mathbb{Z}$, is a polynomial of degree at most $m-1$;
    \item  $\psi^{(m)}\in C^{m-2}$ if $m\geq 2$;
    \item $\frac{\partial^{m-2}}{\partial x^{m-2}} \psi^{(m)}$ is uniformly Lipschitz continuous if $m\geq 2$.
    \item  $\psi^{(m)}$ satisfies the $(m-1)$-order moment condition:
    \[\int x^l \psi^{(m)}(x)dx=0,\quad l=0,1,\cdots,m-1.\]
\end{enumerate}
\end{defn}

It is easy to check that the Haar wavelet system is $1$-regular. A hyperbolic cross wavelet system is said to be $m$-regular if it is a tensor product of one-dimensional wavelet system whose mother wavelet is $m$-regular. With the concept of $m$-regularity, we can introduce the approximation, or called reconstruction process, of a function $f\in\CalMB^r_{p,q}$.

Let $\{\phi_{\Bk,\Bs}\}$ be an $m$-regular hyperbolic cross wavelet system, define the following level-$l$ projection:
\[P^m_l [f]=\sum_{|\Bk|\leq l}\sum_{\Bs\in\rho(\Bk)} \phi_{\Bk,\Bs}\int f(\Bx)\phi_{\Bk,\Bs}(\Bx)d\Bx \]
Then for any $f\in\CalMB^r_{p,p}$ with $0<p<\infty$, $(\frac{1}{p}-1)^+<r<m-1+\min\{1,\frac{1}{p}\}$, and for any $l\in\NatInt$, we have
\begin{equation}\label{eq:wavelet_approximation}
    \|f-P^m_l[f]\|_p\lesssim\begin{cases}
2^{-rl},\quad &\text{if}\ p\leq 2\\
l^{(D-1)(\frac{1}{2}-\frac{1}{p})}2^{-rl},\quad &\text{if}\ p>2
\end{cases}.
\end{equation}

Notice that the mixed smooth Sobolev space $\CalH^\alpha_{mix}$ in the main paper is equivalent to the mixed smooth Besov space $\CalMB^\alpha_{2,2}$. As a result, our analysis is based on the fact that approximation \eqref{eq:wavelet_approximation} with $p=2$ holds for the approximation of the true density $p$ as well. We need to impose the following condition on the wavelet system for constructing Hyperbolic cross estimator 
\begin{condition}
Suppose the true density $p$ is in $\CalH^\alpha_{mix}$. let $\tilde{p}$ and $H_n(\Bx,\By)$ be the hyperbolic cross estimator and kernel used in goodness-of-fit test. Then both $\tilde{p}$ and $H_n$ must be constructed from $m$-regular hyperbolic cross wavelet system with $m>\alpha+\frac{1}{2}$.
\end{condition}
While constructing $m$-regularity is not trivial, it suffices for our purpose to note that $m$-regular hyperbolic cross wavelets do exist, such as the most commonly used spline-wavelet isomorphisms. Please refer to references listed on \cite{Dung16} for more details.
\end{document}